\documentclass[10pt,twocolumn,letterpaper]{article}
\usepackage{graphicx}
\usepackage{amsmath}
\usepackage{amssymb}
\usepackage{booktabs}
\usepackage{multirow}
\usepackage{url}
\usepackage{bm}
\usepackage{float}
\usepackage{amsthm}
\usepackage{caption}
\usepackage{subcaption}
\usepackage[margin=1in]{geometry}
\usepackage{adjustbox}
\usepackage[pagebackref,breaklinks,colorlinks]{hyperref}

\usepackage[capitalize]{cleveref}
\crefname{section}{Sec.}{Secs.}
\Crefname{section}{Section}{Sections}
\Crefname{table}{Table}{Tables}
\crefname{table}{Tab.}{Tabs.}

\newtheorem{proposition}{Proposition}

\newcommand{\vvec}{\mathbf{v}}

\begin{document}

\title{Positional-Encoding Image Prior}

\author{
Nimrod Shabtay\thanks{Equal contribution}, 
Eli Schwartz$^{\small*}$,
Raja Giryes
\\
\tt Tel-Aviv University
}
\date{}
\maketitle

%%%%%%%%% ABSTRACT
\begin{abstract}
In Deep Image Prior (DIP), a Convolutional Neural Network (CNN) is fitted to map a latent space to a degraded (e.g. noisy) image but in the process learns to reconstruct the clean image.
This phenomenon is attributed to CNN's internal image prior. We revisit the DIP framework, examining it from the perspective of a neural implicit representation.
Motivated by this perspective, we replace the random latent with Fourier-Features (Positional Encoding). We empirically demonstrate that the convolution layers in DIP can be replaced with simple pixel-level MLPs thanks to the Fourier features properties. We also prove that they are equivalent in the case of linear networks. We name our scheme ``Positional Encoding Image Prior" (PIP) and exhibit that it performs very similar to DIP on various image-reconstruction tasks with much fewer parameters. Moreover, we show that PIP can be easily extended to videos, where image-prior and INR methods struggle and suffer from instability.
Code and additional examples for all tasks, including videos, are available on the project page \href{https://nimrodshabtay.github.io/PIP/}{\texttt{nimrodshabtay.github.io/PIP}}.
\end{abstract}

%%%%%%%%% BODY TEXT
\section{Introduction}
\label{sec:intro}
\begin{figure}[t]    
    \centering
    \includegraphics[width=0.9\linewidth]{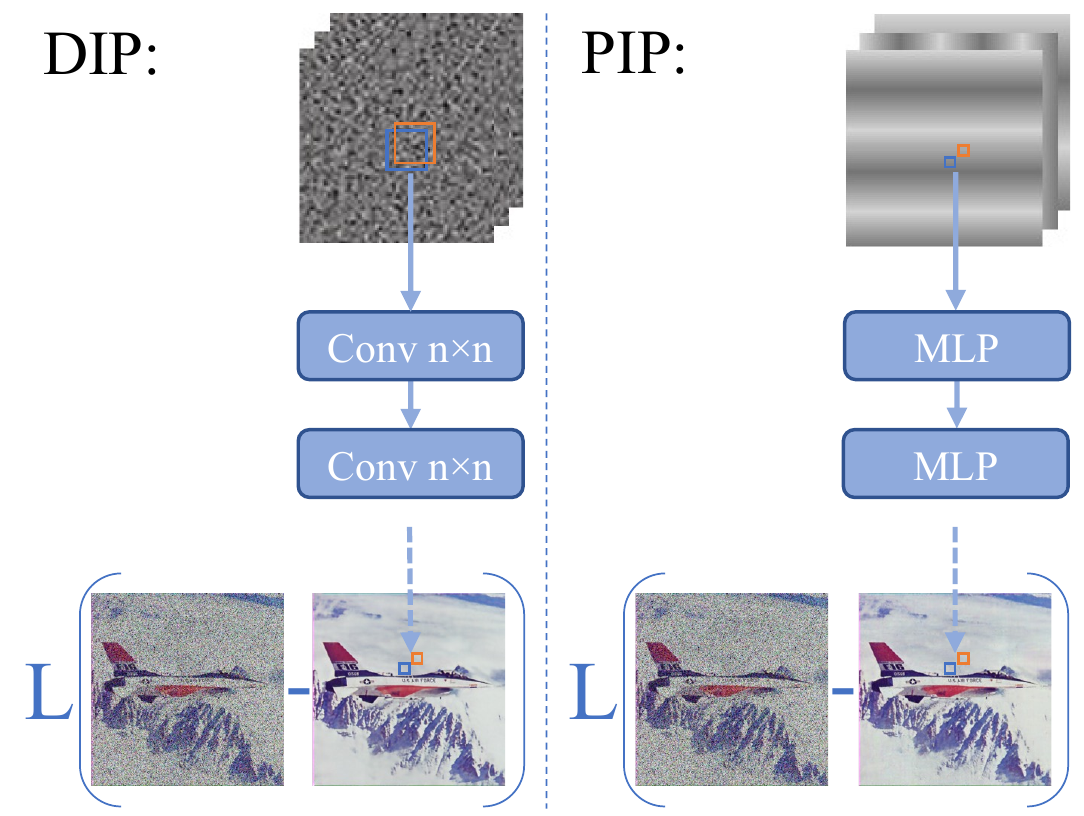}    
    \caption{We offer a novel view of DIP as an implicit model that maps noise to RGB values (left). Although it maps noise to a degraded image, DIP produces a clean image. We suggest that this image prior, or regularization, stems from the fact that due to the convolutional structure of the DIP architecture neighboring pixels in the output image (blue and orange in the picture) are a function of almost the same noise box in the input but a bit shifted. With this implicit model perspective, we suggest that one may achieve a similar `image prior' effect by replacing the input noise with Fourier-Features. We also prove equivalence in the linear network case.
    As a result, we may use a simple pixel-level MLP that has much fewer parameters than the DIP CNN and still produces denoised images of the same quality.
    Remarkably, for video, this leads to significant improvement.
    }
    \label{fig:teaser}
\end{figure}

Deep Image Prior (DIP) \cite{ulyanov2018deep} has shown that if a CNN is trained to map random noise or a learned latent code to a degraded image, it will either converge to the restored image (e.g. for super-resolution) or produce the restored image in the middle of the optimization process (e.g. for denoising). This has been attributed to an internal image prior that CNNs have.
Even though DIP is used in a zero-shot setup with no supervision and with only a degraded image available, it achieves impressive results for many image restoration tasks, including denoising, super-resolution, inpainting, and more. DIP was also extended to other tasks such as segmentation and dehazing using DoubleDIP \cite{DoubleDIP}, video processing \cite{Zhang_2019_ICCV} and 3D mesh reconstruction \cite{Hanocka2020p2m}.

Despite the remarkable success of DIP, it is still unclear why fitting random noise to a deteriorated image can restore the image. 
One explanation of how DIP works, but not why it works, is that the CNN learns to fit first the low frequencies and only later the higher frequencies, hence early stopping has a low-pass filter effect \cite{ShiIJCV22}.

In this work, we propose that DIP should be considered as a neural implicit model that is trained to represent the target image. These models gained high popularity in the task of 3D scene representation, due to their use in Neural Radiance Fields (NeRF) \cite{mildenhall2020nerf} and its many follow-up works.
They were also used for 2D image super-resolution \cite{chen2021learning}, but not in the zero-shot form, where they learned to map representations from a pre-trained CNN to RGB values.

Neural implicit functions are neural models trained to represent a signal, by training a multilayer perceptron (MLP) to map the coordinates of the target signal (an image in our case) to the value in these coordinates (RGB color in the case of images). This results in a neural network that implicitly represents the image, i.e., it can generate it by simply calculating the output of all its coordinates and even interpolating the image by using a finer scale of the coordinates. 

Training an implicit model with the raw coordinates as input results in an oversmoothed output due to the spectral bias of 
neural networks towards low frequencies \cite{Rahaman2019Spectral}. To overcome that, it has been shown that by mapping the input coordinates to Fourier-Features (FF), the implicit model can learn higher frequencies. This leads to a remarkable improvement in the ability of neural implicit representations.

Taking the neural implicit representation perspective, we may think of the inputs of DIP as the equivalent of the coordinate representation in implicit models. With this view in mind, DIP maps random codes to RGB values. 
However, despite the initial impression, this is not totally a random code. Providing independent random codes for each coordinate and then training an implicit neural model using MLP for a given image simply does not work. 
Yet, due to the convolutional structure of DIP, there is a ``structure'' in its randomness. The encoding for a specific location in the image is a shifted version of the random codes of its neighboring coordinates. Moreover, for a linear network, we prove that MLP with FF is equivalent to CNN with random inputs.
Thus, DIP can be considered as an implicit function that maps shifted random codes of the coordinates to the RGB values of the target image. With this view in mind, a natural question is why not simply use other types of positional encodings for the input coordinates? Or more specifically, Fourier features, which have been proven to be successful for neural implicit representation. 

We refer to this setup of DIP with positional encoding inputs as ``Positional Encoding Image Prior" (PIP).
In most cases, PIP-CNN outperforms DIP-CNN in image-restoration tasks such as denoising and super-resolution.
Moreover, we show that using Fourier features the convolutional layers in the DIP architecture can be replaced by MLP operating independently on each coordinate (implemented by $1\times 1$ conv.). The MLP-based architecture is inspired by the MLP used for implicit models, but it is not identical as we still have skip-connection and down/up-sampling like in the original DIP. We denote this network structure as PIP-MLP, and show that it achieves similar performance compared to conventional DIP-CNN but with a lower parameter count and FLOPs. Note that using DIP-MLP fatally fails.

A remarkable advantage of PIP is that it can easily be adapted to other modalities, such as video (3D), by modifying its encoded dimensions. In contrast, extending DIP to video by employing random input encoding with 3D convolutions (3D-DIP) struggles to give reliable results. PIP, on the other hand, achieves high-quality reconstruction and significantly outperforms image-prior and INR methods for video tasks using a simple extension of 2D PIP to 3D and using Fourier features that encode both spatial and temporal dimensions.

To summarize, the contributions of our paper are threefold: (i) a novel interpretation of DIP; (ii) an efficient network for self-training; and (iii) a significant improvement in video denoising and super-resolution from a single video.

\section{Related Work}
\label{sec:related}

\begin{figure}[t]
    \centering
    \resizebox{\columnwidth}{!}{
    \begin{tabular}{cccc}
    GT & Noisy & DIP & PIP (ours) \\
      \includegraphics[width=0.25\textwidth]{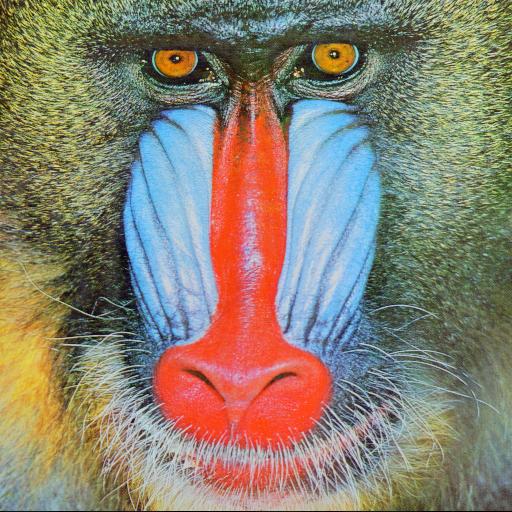}   &  
      \includegraphics[width=0.25\textwidth]{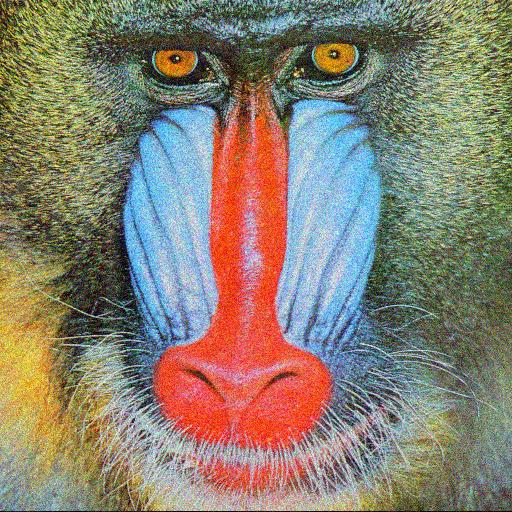}   &  
      \includegraphics[width=0.25\textwidth]{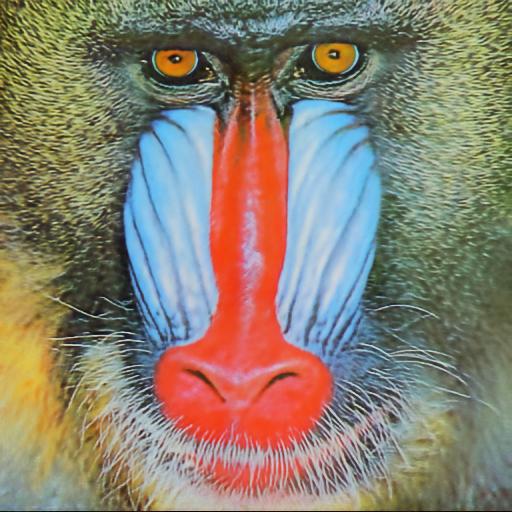} & 
      \includegraphics[width=0.25\textwidth]{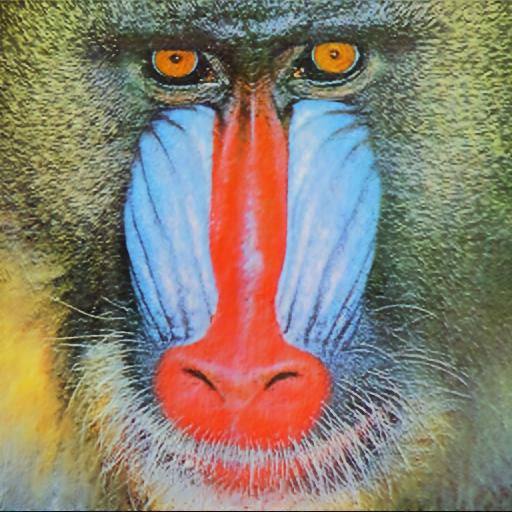}\\
      \includegraphics[width=0.25\textwidth , trim={120px 400px 120px  20px},clip]{figures/Image/Denoising/Baboon/Baboon_gt.jpg}   &  % lbru
      \includegraphics[width=0.25\textwidth, trim={120px 400px 120px  20px},clip]{figures/Image/Denoising/Baboon/Baboon_noisy.jpg}   &  
      \includegraphics[width=0.25\textwidth, trim={120px 400px 120px  20px},clip]{figures/Image/Denoising/Baboon/Baboon_dip.jpg} & 
      \includegraphics[width=0.25\textwidth, trim={120px 400px 120px  20px},clip]{figures/Image/Denoising/Baboon/Baboon_fixed_ff.jpg}\\
      \includegraphics[width=0.25\textwidth]{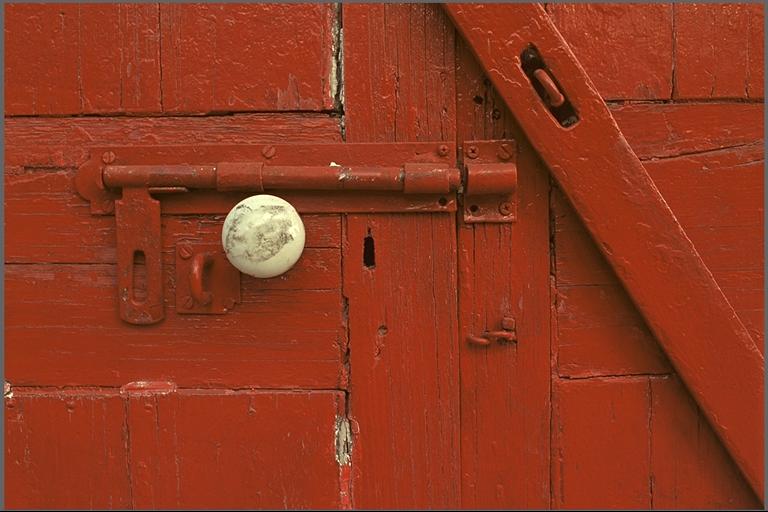}   &  
      \includegraphics[width=0.25\textwidth]{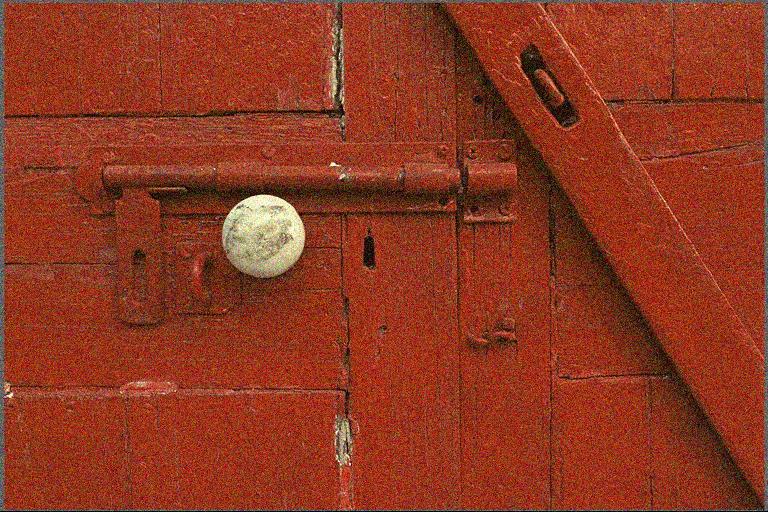}   &  
      \includegraphics[width=0.25\textwidth]{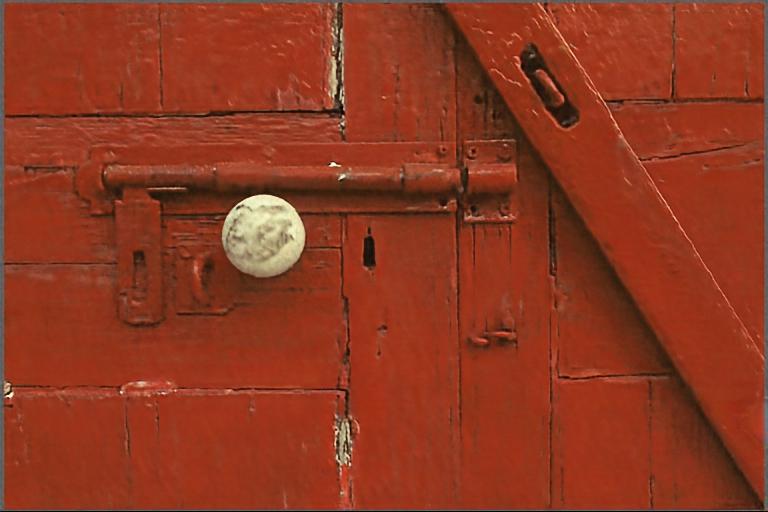} & 
      \includegraphics[width=0.25\textwidth]{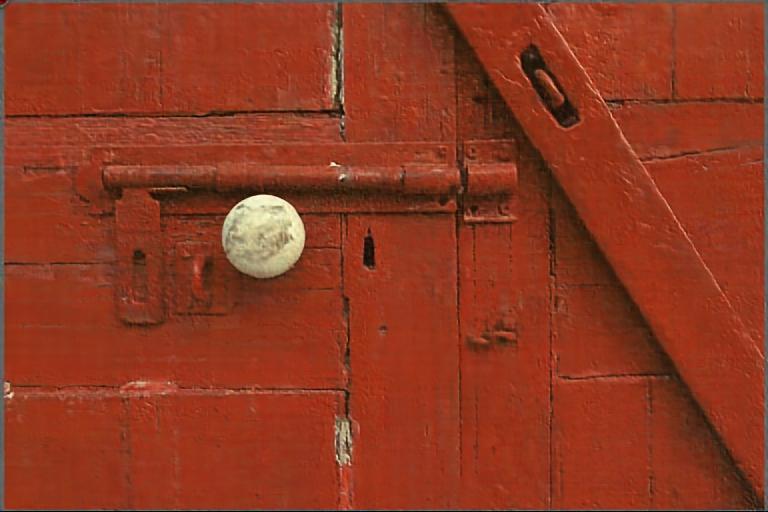}\\
      \includegraphics[width=0.25\textwidth, trim={200px 200px 200px 200px},clip]{figures/Image/Denoising/Kodim02/kodim_02_gt.jpg}   &  
      \includegraphics[width=0.25\textwidth, trim={200px 200px 200px 200px},clip]{figures/Image/Denoising/Kodim02/kodim_02_noisy.jpg}   &  
      \includegraphics[width=0.25\textwidth, trim={200px 200px 200px 200px},clip]{figures/Image/Denoising/Kodim02/kodim_02_dip.jpg} & 
      \includegraphics[width=0.25\textwidth, trim={200px 200px 200px 200px},clip]{figures/Image/Denoising/Kodim02/kodim_02_fixed_ff.jpg}\\
    \end{tabular}
    }
    \vspace{-8pt}
    \caption{\textbf{Image denoising examples} for Gaussian noise ($\sigma=25$). From left to right: clean image (GT), noisy image, DIP (CNN) and PIP (MLP) results. The results for DIP and PIP are very similar, suggesting they follow a similar image prior. }
    \label{fig:denoise}
    \vspace{-5pt}
\end{figure}

\textbf{Positional encoding} was first presented in the transformer architecture \cite{NIPS2017_3f5ee243}, where the importance of positional encoding in learning without an explicit structure was demonstrated.
The visual transformer (ViT) \cite{dosovitskiy2020image} extended the transformer architecture along with the concept of positional encoding to the 2D visual domain, where an image is broken into patches and each patch contains the pixel values along with their corresponding spatial location in the image.
%Peiris et al. 
Fourier-features were used as the positional encoding of transformers for achieving accurate segmentation masks in medical imaging applications \cite{peiris2022robust}. 

Another area of research utilizing positional encoding is implicit neural functions, which are neural networks optimized to map input coordinates to target values.
It was demonstrated in \cite{tancik2020fourfeat} that representing the input coordinates as Fourier-features with a tuneable bandwidth enables a simple MLP to generate complex target domains such as images and 3D shapes while preserving their high-level details.
In Neural Radiance Fields (NeRF) \cite{mildenhall2022nerf}, implicit functions with Fourier features are used for synthesizing novel views of a 3D scene from sparse 2D images.
SIREN \cite{sitzmann2019siren} showed that by changing the activation function of a simple MLP to a periodic one, the network can represent the spatial and temporal derivatives so one can use a regular grid of input coordinates to successfully recover a wide range of target domains (images, videos, 3D surfaces, etc.).
Local implicit image functions (LIIF) \cite{chen2021learning} represent an image as a continuous function using an implicit model. This allows it to perform super-resolution at an arbitrary scale.
Xu et. al \cite{xu2020positionalencoding} show how positional encoding can be interpreted as a spatial bias in GANs.
SAPE \cite{hertz2021sape} and BACON \cite{lindell2021bacon} demonstrate how a combination of a coordinate network along with a limitation of its frequency spectrum can achieve a multi-scale representation of the target domain.
Due to the success of implicit models, coded by Fourier-features, other areas of research incorporate Fourier-features. For example, Li et al. 
\cite{li2021functional} apply learned Fourier-features to control the inputs to gain better sample efficiency in many reinforcement learning problems.
\cite{kim2022zero} proposed to use implicit neural representation with a weight regularization for each layer for blind image restoration.
Bustos-Brinez et al. 
\cite{bustos-brinez2022addmkdeanomaly} use random Fourier-features to perform accurate data-set density estimation for anomaly detection.
This work treats DIP as an implicit model showing that as such it can be used with FF and an MLP.

\begin{table}[t]
    \caption{\textbf{Denoising and Super-Resolution image-prior evaluation (PSNR $\uparrow$).}
    Compared to DIP-CNN, DIP-MLP fails.
    On the other hand, PIP is robust to the model used and does not require a CNN.
    Using fixed frequencies is slightly better for denoising, while learned frequencies (optimizing input frequencies together with model params) are slightly better for SR.}
    \label{tab:denoise_SR}
    \centering
    \resizebox{\linewidth}{!}{
    \begin{tabular}{lll|ccc|cc}
    \toprule
    & & & \multicolumn{3}{c|}{Denoising} & \multicolumn{2}{c}{SR} \\
     & Arch. & Freq. & $\sigma=25$ & $\sigma=10$ & Poisson & $\times 4$ & $\times 8$ \\
     \midrule
    DIP & CNN & - & 28.56 & 29.506 & 29.467 & 27.73 & \textbf{24.58} \\
    DIP & MLP & - & 20.37 & 19.204 & 19.332 & 21.66 & 16.89\\
    \midrule
    PIP & CNN & Fixed & \textbf{28.80} & \textbf{30.011} & \textbf{29.957} & 27.867 & 24.03\\
    PIP & CNN & Learned & 28.28 & 29.795 & 29.813 & \textbf{28.09} & 24.26\\
    PIP & MLP & Fixed & 28.26 & 29.784 & 29.734 & 27.53 & 23.96 \\
    PIP & MLP & Learned & 28.24 & 29.591 & 29.464 & 27.87 & 24.35 \\
    \bottomrule
    \end{tabular}
    }    
\end{table}

\begin{table}[t]
\caption{\textbf{Denoising and Super-Resolution implicit models evaluation (PSNR $\uparrow$).}
    Compared to other implicit models, PIP performs better or similarly compared to all new proposed methods.}
    \label{tab:denoise_sr_inr}
    \centering
    \resizebox{\linewidth}{!}{
    \begin{tabular}{l|ccc|cc}
    \toprule
    & \multicolumn{3}{c|}{Denoising} &  SR \\
    Method & $\sigma=25$ & $\sigma=10$ & Poisson & x4 & x8\\
     \midrule
    SIREN \cite{sitzmann2019siren}                    & 27.94 & \textbf{29.78} & 29.6 & 27.66 & 24.18 \\
    Gaussian Activation \cite{ramasinghe2022beyond}   & 25.11 & 25.20 & 25.21 & 21.6 & 17.43 \\
    Titan  \cite{lejeune2022titan}                    & 25.35 & 25.83 & 25.54 & 24.8 & 21.33 \\
    PIP                                               &\textbf{28.26}& \textbf{29.78} & \textbf{29.73} & \textbf{27.87} & \textbf{24.35}\\
    \bottomrule
    \end{tabular}
    }   
\end{table}

\textbf{Deep Image Prior} (DIP) \cite{ulyanov2018deep} showed that a CNN network (typically a U-net shaped architecture) can recover a clean image from a degraded image in the optimization process of mapping random noise to a degraded image. The power of DIP is shown for several key image restoration tasks such as denoising, super-resolution, and inpainting.
DIP was extended to various applications. Double-DIP \cite{DoubleDIP} introduced a system composed of several DIP networks, where each learns one component of the image such that their sum is the original image. It has been used for image dehazing and foreground/background segmentation.
SelfDeblur \cite{ren2019selfDeblur} performed blind image deblurring by simultaneously optimizing a DIP model and the corresponding blur kernel.
In the medical imaging domain, several works tried to utilize the power of image priors to reconstruct PET images \cite{Yokota_2019_ICCV, Hashimoto_2022, 8581448, 9576711}.
Other works tried to improve the performance of image priors using modifications and additions to the original setup. Mataev et al. 
\cite{Mataev2019DeepREDDI} and Fermanian et al. 
\cite{fermanian:hal-03310533} combined DIP with the plug-and-play framework to improve DIP's performance in several inverse problems. Jin et al. \cite{Jin2022Dual} extended DIP to blended distortion priors.
Zukerman et al. 
\cite{zukerman2021bp} showed how the back-projection fidelity term can improve DIP performance in restoration tasks like deblurring.
In other areas, 
Kurniawan et al. 
\cite{kurniawan2022noise} proposed a method for demosaicing based on DIP. Chen et al. 
\cite{chen2020dip} suggested DIP-based neural architecture search.

Many other works \cite{ShiIJCV22,DBLP:journals/corr/abs-1912-08905,wang2023early} analyzed DIP and attempted to overcome its limitations, e.g., its spectral bias and the need for early-stopping. 
These works, however, did so by maintaining the input as a random latent code, focusing on the limitations of the architecture itself. We challenge the use of random input for DIP and aim at improving its lack of robustness when changing the used neural architecture or the target domain, e.g., video.

Chen et al. 
\cite{Zhang_2019_ICCV} utilized DIP for large hole inpainting to remove objects from videos. Yet, naively scaling DIP to videos by replacing 2D with 3D convolutions (3D-DIP) fails to maintain temporal consistency. Thus, they introduced additional regularization and inputs to improve temporal consistency issues.
Yoo et al. 
\cite{yoo2021time} introduced time-dependent DIP in MRI images, which encodes the temporal variations of images with a DIP network that generates MRI images.
Lu et al. 
\cite{lei2020blind} proposed using DIP based method for blind video restoration by applying DIP network on pairs for degraded and reconstruct image sequences and force consistency from the DIP optimization.
\cite{lu2021} utilize DIP for video editing and manipulation.
The above works show clearly that DIP cannot be naively extended from images to video and it requires additional regularization and loss functions.
On the other hand, PIP can easily be extended to videos.
It learns temporal connections and preserves consistency between frames in a controllable and intuitive manner.

A recent work \cite{lejeune2022titan} tried to draw a connection between image priors and implicit neural representations. While they mentioned the connection, it is not complete with a partial comparison.  PIP comprehensively draws a connection to implicit models both analytically and through a comprehensive study.

\begin{figure}[t]
    \centering
    \resizebox{\columnwidth}{!}{
    \begin{tabular}{cccc}
    HR & LR & DIP & PIP (ours) \\
      \includegraphics[width=0.25\textwidth]{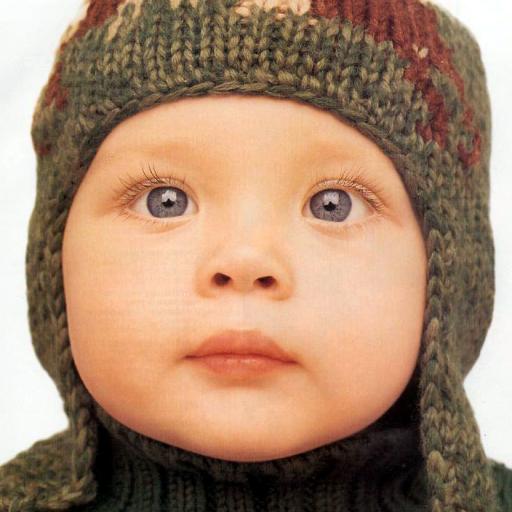}   &  
      \includegraphics[width=0.25\textwidth]{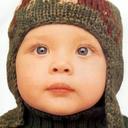}   &  
      \includegraphics[width=0.25\textwidth]{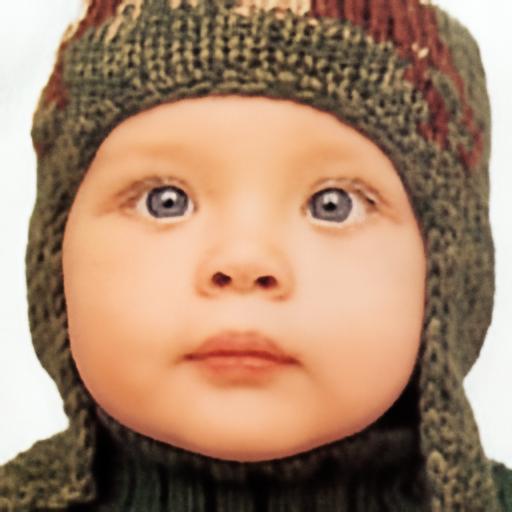} & 
      \includegraphics[width=0.25\textwidth]{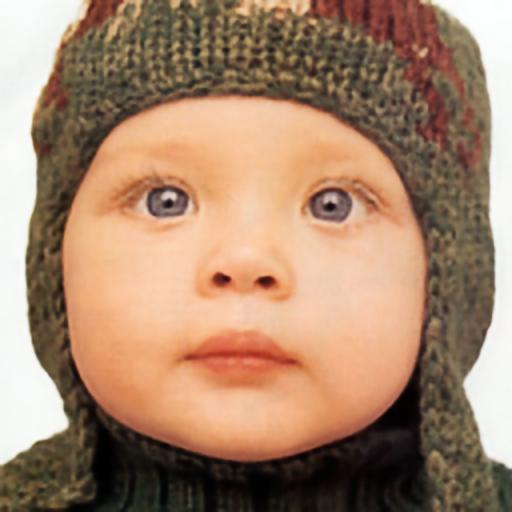}
      \\
      \includegraphics[width=0.25\textwidth, trim={300px 300px 0px 150px},clip]{figures/Image/SR/X4/Set5/Baby/Set5_baby_hr.jpg}   &  
      \includegraphics[width=0.25\textwidth, trim={75px 75px 0px 37.5px},clip]{figures/Image/SR/X4/Set5/Baby/Set5_baby_lr.jpg}   &  
      \includegraphics[width=0.25\textwidth, trim={300px 300px 0px 150px},clip]{figures/Image/SR/X4/Set5/Baby/Set5_baby_dip.jpg} & 
      \includegraphics[width=0.25\textwidth, trim={300px 300px 0px 150px},clip]{figures/Image/SR/X4/Set5/Baby/Set5_baby_infer_ff.jpg}
      \\
      \includegraphics[width=0.25\textwidth]{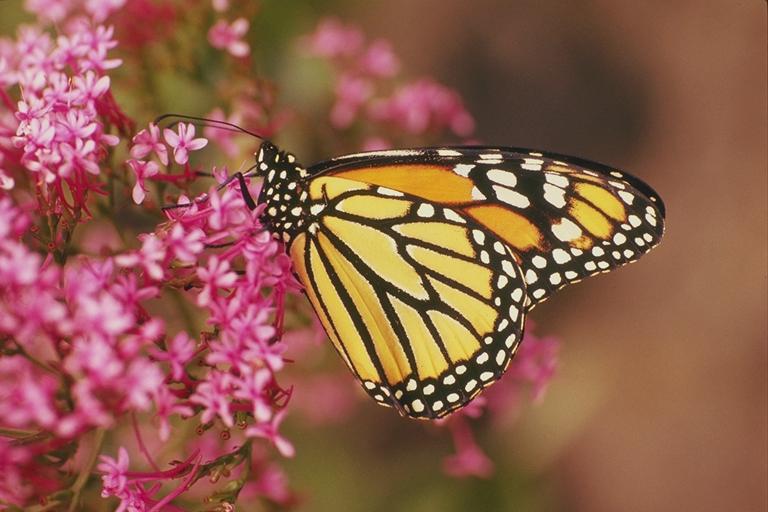}   &  
      \includegraphics[width=0.25\textwidth]{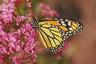}   &  
      \includegraphics[width=0.25\textwidth]{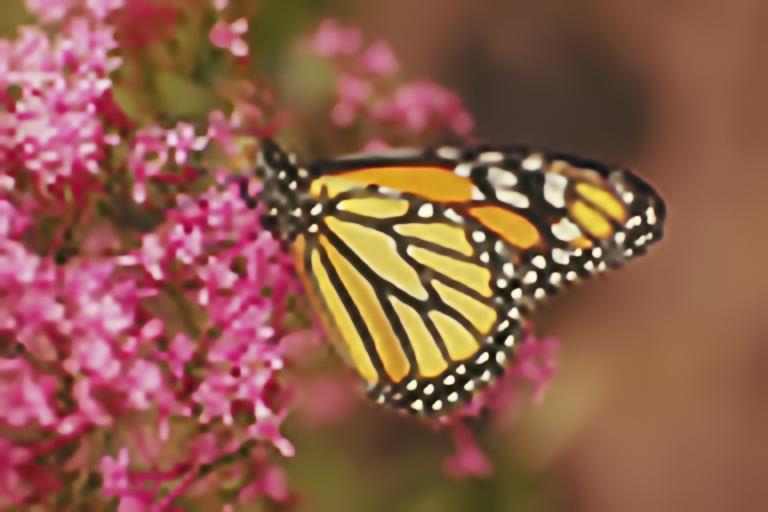} & 
      \includegraphics[width=0.25\textwidth]{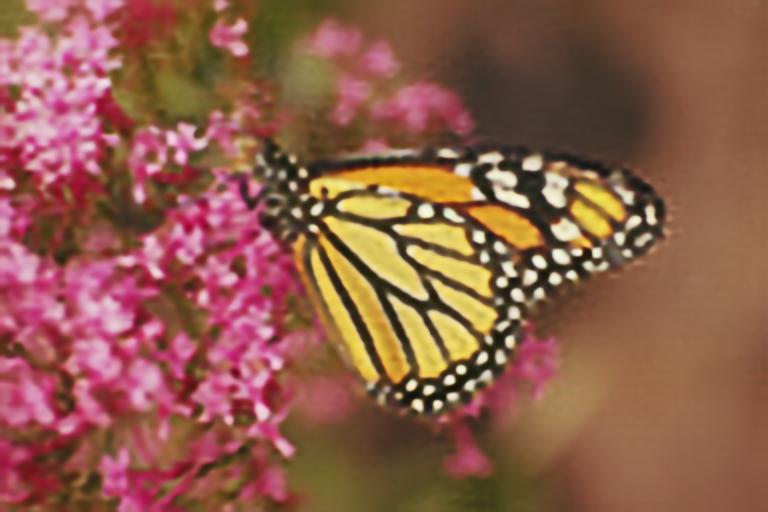}
      \\
      \includegraphics[width=0.25\textwidth, trim={450 220 200px 250px},clip]{figures/Image/SR/X8/Butterfly_gt.jpg}   &  
      \includegraphics[width=0.25\textwidth, trim={56.25px 27.5px 25px 31.25px},clip]{figures/Image/SR/X8/Butterfly_lr.jpg}   &  
      \includegraphics[width=0.25\textwidth, trim={450 220 200px 250px},clip]{figures/Image/SR/X8/Butterfly_noise.jpg} & 
      \includegraphics[width=0.25\textwidth, trim={450 220 200px 250px},clip]{figures/Image/SR/X8/Butterfly_learned_ff.jpg}
      \\
    \end{tabular}
    }
    \vspace{-8pt}
    \caption{\textbf{SR examples} - First example is for $\times 4$ SR; Second is for $\times 8$ SR. The results for DIP (CNN) and PIP (MLP) are very similar, suggesting they have a similar image prior.}
    \label{fig:sr_x4_x8}
    \vspace{-5pt}
\end{figure}

\begin{table}[t]
    \caption{\textbf{The effect of the U-net structure} is evaluated on both the denoising and super-resolution tasks (PNSR $\uparrow$). While DIP is ineffective without convolutions and the U-Net structure, PIP retains its performance for denoising. We observe a drop in performance for PIP on SR without the U-Net (-1.3dB), but not as significant as for DIP.}
    \label{tab:unet_effect}
    \centering
    \resizebox{\linewidth}{!}{
    \begin{tabular}{cc}
    \begin{tabular}{lcc}
    \toprule
        Denoising &  DIP & PIP\\
        \midrule
         U-Net MLP  & 20.37 & 28.26 \\
         Simple MLP & 15.26 & 28.42 \\
         \bottomrule
    \end{tabular}
    &
      \begin{tabular}{lcc}
    \toprule
        SR $\times 4$ &  DIP & PIP \\
        \midrule
         U-Net MLP  & 21.66 & 27.87 \\
         Simple MLP & 14.03 & 26.57 \\
         \bottomrule
    \end{tabular}
    \end{tabular}
    }
\end{table}

\begin{table}[t]
    \caption{\textbf{Activation effect on reconstruction results}. We evaluate (PSNR $\uparrow$) recent proposed activations for implicit models such as Sine and Gaussian and compare them to the originally proposed Leaky-ReLU over several image reconstruction tasks. Leaky-ReLU outperformed all activations.}
    \label{tab:activations}
    \centering
    \resizebox{\linewidth}{!}{
    \begin{tabular}{lll|ccc|cc}
    \toprule
    & & & \multicolumn{3}{c|}{Denoising} & \multicolumn{2}{c}{SR} \\
    & & & $\sigma=25$ & $\sigma=10$ & Poisson & $\times 4$ & $\times 8$ \\
     \midrule
    Gaussian & & & 28.06 & 30.68 & 28.35 & 25.95 & 19.762 \\
    Sine & & & 27.21 & 27.85 & 28.12 & 25.85 & 20.80 \\
    Leaky-ReLU & & & 28.26 & 29.784 & 29.734 & 27.87 & 24.35 \\
    \bottomrule
    \end{tabular}
    }    
\end{table}

\begin{figure}[t]
    \centering
    \vspace{-10pt}
    \includegraphics[width=\columnwidth]{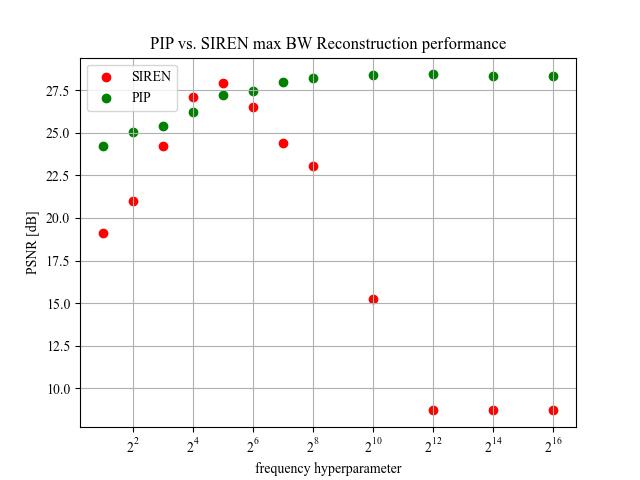}
    \vspace{-5pt}
    \caption{\textbf{SIREN vs. PIP frequency hyper-parameter effect.} We apply different Band-Width limits to SIREN (different $\omega_0$ values) and PIP (different $F_{max}$ values), and evaluate the reconstruction results over the "standard dataset". PIP produce over-smooth images when the max frequency is low but converge to a steady plateau as the frequency range gets higher as oppose to SIREN where we can see an optimal point around $\omega_0=2^5=32$ but the performance drops as the frequency range move away from the optimal point.}
    \vspace{-5pt}
    \label{fig:bw_pip_siren}
\end{figure}

\section{Positional-Encoding Image \\Prior (PIP)}
\label{sec:method}

DIP \cite{ulyanov2018deep} demonstrates that if an untrained CNN model is fitted, in a zero-shot manner, to a corrupted image, it will eventually recover the clean image without adding any explicit regularization.
Formally, we have $x=f_\theta(z)$, where $f_\theta$ is a neural network that maps random input codes $z \in \mathbb{R}^{C \times W \times H} \sim U(0,\frac{1}{10})$ to image space $x \in \mathbb{R} ^{3 \times W \times H}$. Given a corrupted image $x_0$, we want to recover its clean version. Such an image reconstruction task can be formulated as an energy minimization problem
\begin{equation}\label{eq:direct}
    x^* = \operatornamewithlimits{argmin}_x E(x;x_0) + R(x),
\end{equation}
where $E(x;x_0)$ is the data term and $R(x)$ is a regularization term.
The main observation in DIP is that, due to the CNN's image prior, the regularization term can be dropped and the data term can be optimized by adapting the model parameters with gradient descent
\begin{equation}\label{eq:reparametrization}
   \theta^* = \operatornamewithlimits{argmin}_\theta E(f_\theta(z);x_0),\qquad x^* =f_{\theta^*}(z)\,.
\end{equation}

As illustrated in Fig.~\ref{fig:teaser}, one may consider DIP as an implicit model and its random input as positional encoding of shifted random patches. 
Motivated by the above relations, in PIP, we suggest replacing the random input codes $z$ with Fourier-features as the positional encoding of the image coordinates.
% and use element-wise optimization, i.e., MLP for each input coordinate, which is equivalent to performing $1\times 1$ convolution on the whole input.
These features are advantageous of being smooth, continuous, and with controllable bandwidth. 
For a coordinate vector $\vvec$ and a frequency set $\{f_0,\ldots,f_{m-1}\}$, the Fourier-features are
\begin{equation}\label{eq:ff}
\gamma(\vvec)=\left[\ldots, \cos(f_i \vvec),\sin(f_i \vvec), \ldots \right].
\end{equation}
We use log-linear spaced frequencies $f_i=\sigma^{i/m}$, with the base $\sigma$ set by a predefined maximum frequency, such that $f_{m-1}=f_{max}$. (see \ref{fig:pip_encoding} for a visual diagram of the input encoding.
Inspired by \cite{li2021learnable}, we also tested a learnable-frequencies variant, where the frequencies are initialized as described above and then optimized together with the model parameters using gradient descent.
This variant allows learning the most appropriate encoding to be used for a given image and is beneficial for some tasks.

We further suggest replacing the convolutional layers with an element-wise optimization with a coordinate-MLP, which is equivalent to performing $1\times 1$ convolution on the whole input.
The intuition is that the PE image prior is sufficient to make the convolution redundant.
Moreover, in the following proposition, we show that \textit{for a linear model}, DIP can be formulated as learning of an implicit function with Fourier features as its input. Full details are in Appendix~\ref{sec:prop1_proof}.

\begin{proposition}
\label{prop:PIP_FF_relation}
For $E$ being the $\ell_2$ loss and $f(z) = h*z$, where $h$ is a convolution kernel, \eqref{eq:reparametrization} is equivalent to an element-wise optimization with Fourier features. 
\end{proposition}

DIP uses an encoder-decoder hourglass architecture. 
It consists of convolutional layers with strides for down-sampling in the encoder; convolutional layers with simple bi-linear or nearest-neighbors up-sampling in the decoder; and possibly skip-connection between corresponding scales in the encoder and decoder.
For PIP we adopt the same multi-scale hourglass model, but replace each convolution with $k \times k$ kernel and $s \times s$ strides with a pixel-level MLP followed by nearest-neighbor down-sampling (simply implemented as a convolution with $1 \times 1$ kernel and $s \times s$ strides).
Fig.~\ref{fig:teaser} illustrates the modifications.
The suggested architecture is more efficient in terms of compute and memory as it uses $1 \times 1$ convolution instead of $3 \times 3$.
Also, we observe that while DIP is somewhat sensitive to the architecture used (changing $3 \times 3$ kernels to $1 \times 1$ makes it utterly fail), this is not the case for PIP.
It is much more robust to changes and we use the same model for all tasks.

Finally, unlike DIP, PIP can be easily adapted to other domains.
For example, when DIP is adapted to video (using 3D convolutions), additional losses have to be used to ensure temporal consistency \cite{yoo2021time}.
In PIP, the model is simply adjusted by defining Fourier features for all dimensions.

\noindent {\bf Implementation details.} Unless otherwise specified, we follow the original architecture choice of DIP for denoising and SR: 5 levels U-net with skip connections. Each level has 2 convolutional blocks with 128 channels, and the skip connections are with convolutional blocks with 4 channels (see \ref{fig:pip_arch} for a detailed diagram).
We also follow the original DIP hyper-parameters, Adam optimizer with a learning rate of 0.01 and the same number of iterations for early-stopping.
For all demonstrated applications, we used the original configurations and only replaced the 3x3 convolutional filters with 1x1 and plugged in Fourier features as input instead of noise.
\begin{figure}[]
    \centering
    \resizebox{\columnwidth}{!}{
    \includegraphics[]{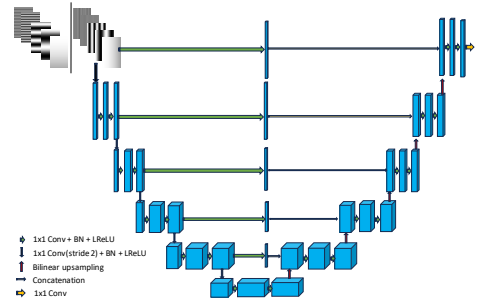}  
    }
    \vspace{-0.6in}
    \caption{\textbf{PIP architecture.} Our architecture is based on DIP common architecture of a Unet with skip-connections.}
    \label{fig:pip_arch}    
\end{figure}

\section{Evaluations}
In this section, we first perform a quantitative evaluation of denoising and super-resolution (SR).
We also study the effect of different design choices and understand what makes PIP work from a spectral-bias perspective.
We then qualitatively demonstrate the applicability of PIP to more tasks -- inpainting, blind dehazing, and CLIP inversion.
For more visual results, please refer to the project \href{https://nimrodshabtay.github.io/PIP/}{website}.
We follow the denoising and SR experiments from DIP \cite{ulyanov2018deep} using the same set of images.
For denoising: 9 colored images\footnote{\href{https://webpages.tuni.fi/foi/GCF-BM3D/index.html\#ref\_results}{https://webpages.tuni.fi/foi/GCF-BM3D/index.html\#ref\_results}} with additive Gaussian noise ($\sigma=25$ and $\sigma=10$) and Poisson noise.
For SR: the union of Set14~\cite{set14} and Set5~\cite{set5}, resulting in 19 images, with $\times4$ and $\times8$ downscaling factors.

To test PIP robustness, for each task we checked 4 PIP variants: with CNN ($3 \times 3$ kernels) vs. MLP ($1 \times 1$ kernels) architecture and fixed vs. learned frequencies.
Table \ref{tab:denoise_SR} summarizes the results.
Observe that fixed frequencies perform slightly better for denoising while learned frequencies perform slightly better for SR.
PIP exhibits robustness to the architecture used as CNN and MLP have comparable performance, while DIP relies on the CNN structure and fails with MLP.
% Figure \ref{fig:dip_pip_corr} demonstrates a high correlation in performance between DIP (CNN) and PIP (MLP) suggesting they enjoy similar prior effects.

Figures \ref{fig:denoise} and \ref{fig:sr_x4_x8} show examples of the denoising and SR results, respectively, for both PIP and DIP.
% We observe quite similar generated images, with similar artifacts, suggesting, again, CNN and PE have similar `prior' effect.
We observe a similar quality level overall, with DIP tending a little bit more to over-smoothing artifacts, and PIP to localized artifacts. 
We also observed that the performance on different images by DIP and PIP is highly correlated (Coefficient of determination $R^2=0.99$ ), suggesting CNN and PE have a similar `prior' effect. (see Figure \ref{fig:dip_pip_corr} in the appendix for more details).

To further demonstrate PIP's superiority over recent proposed implicit models for images, we performed a quantitative evaluation as described above. Table \ref{tab:denoise_sr_inr} summarized our results and show that PIP is better or comparbale to other methods in all reconstruction tasks.

\noindent {\bf U-Net vs MLP.} DIP employs a U-Net architecture. It has an encoder-decoder structure, where each encoder block uses strides to downsample the image and each decoder block upsamples the image by nearest or bilinear interpolation.
Parts of the original resolution signal are passed through skip connections. These down- and up-sampling operations have the effect of a low-pass filter on the parts of the signal passed through them. This can be treated as part of the image prior of the model. To test this effect we compare $1\times 1$ U-Net vs. a vanilla simple MLP applied per pixel in Table \ref{tab:unet_effect}. It shows that for denoising, a simple MLP, commonly used in implicit models, is enough and performs just as well as the U-Net that operates in a multi-scale scheme. 
For SR, however, U-Net is stronger than the simple MLP. This is probably due to the fact that for this task non-local information is important for retrieving lost high frequencies. 
Overall, PIP, unlike DIP which requires convolutions and the U-Net structure, exhibits strong robustness to the model used.
Throughout this paper, for the sake of simplicity, we use the $1\times 1$ U-Net architecture for all experiments.

\noindent {\bf Activations} DIP was originally designed to work with leaky-ReLU activations. Current improvements in the field of implicit neural representations suggest replacing the vanilla relu with activations that are always differentiable such as Sine activations \cite{sitzmann2019siren} or Gaussian activations \cite{ramasinghe2022beyond}.
In table \ref{tab:activations}, we incorporate those recent proposed activations into PIP U-net architecture to measure the effect on reconstruction performance and show that in PIP for reconstruction tasks the original activation (LeakyReLU) works best compared to newer proposed activation functions.

\noindent {\bf Effect of the Fourier Features (FF) frequency range.}
The range of input frequencies allows controlling the generated image. When $f_{max}$ in FF is too low, we generate a blurry image. When $f_{max}$ is high we can fit high-frequency noise.
However, thanks to the spectral-bias effect, early-stopping prevents fitting the noise even at high $f_{max}$. 

Figure \ref{fig:max_freq_effect_on_results} demonstrates the effect of $f_{max}$. 
For low $f_{max}=2$ we get blurry images. We find that
$2^8\leq f_{max} \leq 2^{10}$ provides a good balance and generates good images.
But even if we use the very high $f_{max}=2^{16}$, thanks to the early stopping we still do not fit the noise despite the model capability.
We also show the PSNR performance as a function of $f_{max}$. The performance improves till $f_{max}=2^8$. Then at higher frequencies the performance varies a little bit but with no clear trend of performance degradation.

We additionally compared the frequency hyper-parameter effect vs. SIREN which uses a meshgrid and sine activation. SIREN uses a frequency hyper-parameter $\omega_0$ to control the bandwidth of the output. In Figure \ref{fig:bw_pip_siren} we show that SIREN is highly sensitive to the choice of $\omega_0$ while PIP is less sensitive.

A special case of limited frequency range is using a ``meshgrid" input, i.e. two input channels containing the horizontal and vertical coordinates. In FF, the encoding with the lowest frequency is $\sin(x)$, which simply applies a monotonic mapping to the coordinates and hence produces a very similar representation. 
Note that DIP used meshgrid as input for the inpainting task and in \cite{ShiIJCV22} it was tested with a pure MLP. We repeat the experiment here for that task of denoising with the purpose of showing the advantage and the key role of the FF. 
Figure \ref{fig:meshgrid} compares the training of a denoising MLP model with three different options for the input: FF, meshgrid, and noise.
Clearly, using noise as input with MLP layers does not work (also noted in Table \ref{tab:denoise_SR}).
Observe that with meshgrid the model does not overfit and hence does not need early stopping. Yet, it is unable to reach FF performance even after long training. This is because it encodes only the lowest frequency, and is thus incapable of reconstructing the higher frequencies in the image. Figure~\ref{fig:meshgrid_sup_mat} provide a visual example of DIP with meshgrid showing that it fails to recover the texture details and outputs over-smoothed results. 

\begin{figure}[t]   
\centering
    \includegraphics[width=0.8\linewidth]{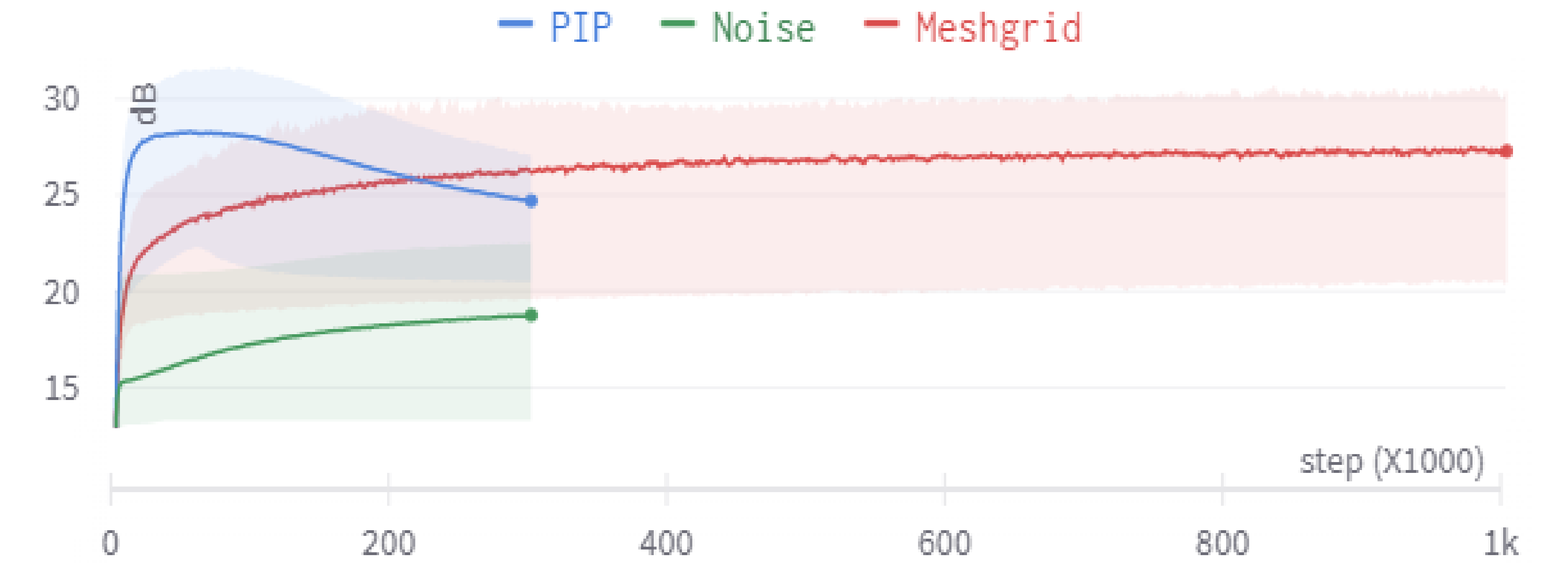}    
    \caption{\textbf{MLP-Denoising over the "standard dataset" with FF (PIP), Meshgrid and Noise encoding}. Noise encoding fails to generate a valid mapping between the input code and the target image. The meshgrid provides a much better representation but the lack of high frequencies generates an over-smoothed image. PIP lets the model produce a suitable range of frequencies such that it denoises the image well.}
    \label{fig:meshgrid}
\end{figure}

\begin{figure*}[t]
    \centering
    \vspace{-9pt}
    \hspace{-0.15in}
    \begin{subfigure}[t]{0.74\linewidth}
    \resizebox{\linewidth}{!}{
    \begin{tabular}{ccccc}
    GT & $f_{max}=2^4$ & $f_{max}=2^6$ & $f_{max}=2^8$  & $f_{max}=2^{16}$\\
      \includegraphics[width=0.25\linewidth]{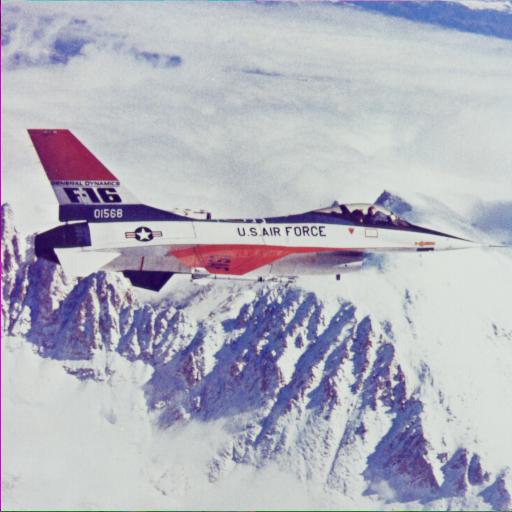}   &  
      \includegraphics[width=0.25\linewidth]{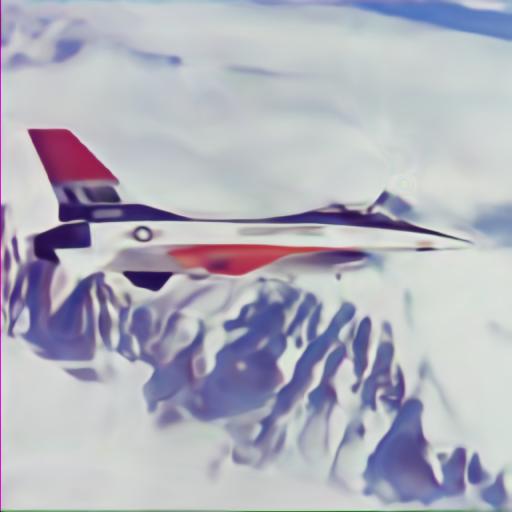}   &  
      \includegraphics[width=0.25\linewidth]{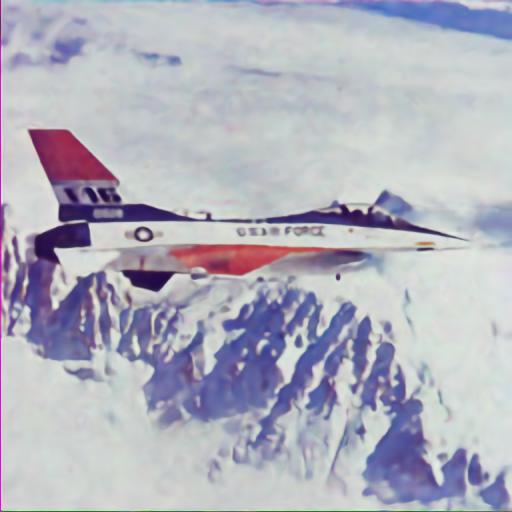}   &  
      \includegraphics[width=0.25\linewidth]{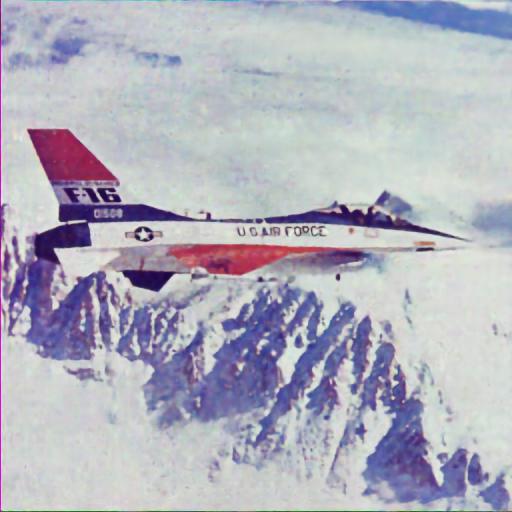}   &  
      \includegraphics[width=0.25\linewidth]{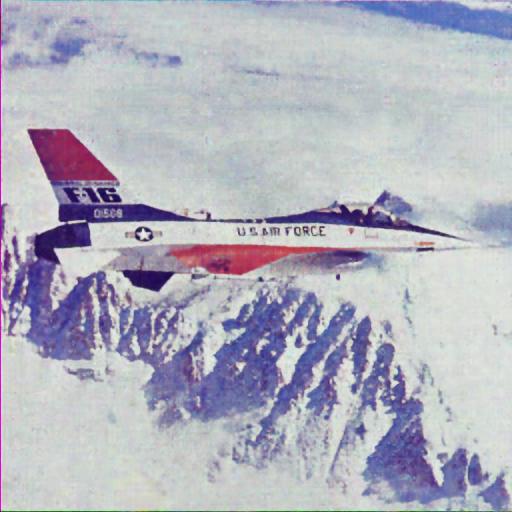}  
      \\
      \includegraphics[width=0.25\linewidth, trim={150px 250px 150px  200px},clip]{figures/frequencies/f16_gt.jpg} & 
      \includegraphics[width=0.25\linewidth, trim={150px 250px 150px  200px},clip]{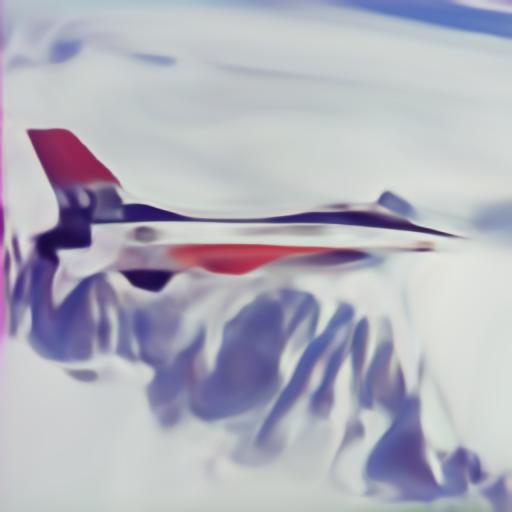} & 
      \includegraphics[width=0.25\linewidth, trim={150px 250px 150px  200px},clip]{figures/frequencies/F16_max_f_16.jpg} & 
      \includegraphics[width=0.25\linewidth, trim={150px 250px 150px  200px},clip]{figures/frequencies/F16_max_f_256.jpg} & 
      \includegraphics[width=0.25\linewidth, trim={150px 250px 150px  200px},clip]{figures/frequencies/F16_max_f_65k.jpg}\\
    \end{tabular}
    }
    \end{subfigure}%
    \hspace{-0.2in}
    \begin{subfigure}[t]{0.24\linewidth}
    \begin{tabular}{c}
    \includegraphics[width=1\linewidth]{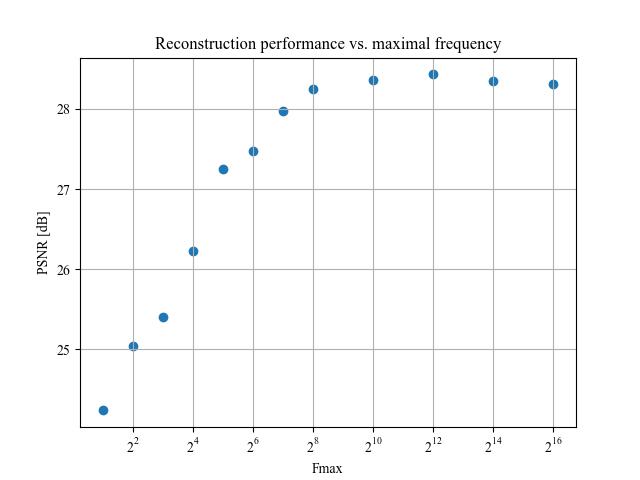}
    \end{tabular}
    \end{subfigure}
    \vspace{-5pt}
    \caption{\textbf{FF bandwidth.} We compare image denoising results over the "standard dataset`` for various $f_{max}$. As we increase $f_{max}$ the reconstruction is sharper until it reaches a plateau around $f_{max}=2^{8}$. Using a very large frequency range, $f_{max}=2^{16}$, does not drastically affect results due to the role of early-stopping in preventing noise overfitting. On the right: PSNR as a function of $f_{max}$.}
    \label{fig:max_freq_effect_on_results}
\end{figure*}

\noindent {\bf Early stopping.} A well known issues for DIP (and for PIP as well) is the problem of determining the number of iterations before hand. In our experiments, we follow the original DIP method and use a fixed number of iterations for early stopping (ES). However, there are methods for automatically choosing the stopping point. We tested EMV and WMV from \cite{wang2023early} and found they work better for PIP compared to DIP. See quantitative results in Tab. \ref{tab:early_stop}.

\begin{table}[t]
\footnotesize
\begin{minipage}[t]{0.4\linewidth}
\caption{\textbf{Early stopping} PSNR difference in dB ($\downarrow$}) from the optimal stopping point.
\label{tab:early_stop}
% \begin{table}
\centering
    \begin{tabular}{lcc}
    \toprule
         & DIP & PIP \\
         \midrule
     EMV & 1.22 & \textbf{0.82} \\
     WMV & 0.78 & \textbf{0.75}  \\
     \bottomrule
    \end{tabular}
\end{minipage}
\hspace{10pt}
\begin{minipage}[t]{0.5\linewidth}
\caption{\textbf{Computation efficiency. ($\downarrow$})} PIP archives similar performance to DIP with fewer params and FLOPs.
\label{tab:mem_flops}     
    \centering
    \renewcommand{\tabcolsep}{2pt}
    \begin{tabular}{@{}l|ccc@{}}
    \toprule
    Arch. & Params[M] & FLOPs[G] & Layers\\
    \midrule
    DIP & 2.22 & 65.84 & 26\\
    PIP & \textbf{0.33} & \textbf{12.15} & 26\\
    \bottomrule
    \end{tabular}

\end{minipage}
\end{table}

\noindent {\bf Computation efficiency.} We compared the original DIP's CNN vs. our MLP in terms of parameter count, FLOPs and layer count. Despite mostly having a similar performance, thanks to the removal of all spatial filters, the number of parameters and FLOPs decreased by a factor of $\sim 7$ and $\sim 5$ with the same amount of layers; see Tab. \ref{tab:mem_flops}.

\subsection{Spectral Bias}
Following \cite{ShiIJCV22}, which analyzed DIP and showed that the model learns first the low frequencies and then the higher frequencies, we perform a similar evaluation for PIP.
To test this, we fit a PIP model to a synthetic image comprised of 3 sinusoidal signals with frequencies $2\pi \cdot 4, 2\pi \cdot 8, 2\pi \cdot 16$, while the input frequency range is limited to $f_{max}=2\pi \cdot 8$.
Figure \ref{fig:spectral_bias} presents the synthetic image as well as the images generated during the optimization process.
We also calculate the Fourier transform of the generated images and observe the rate at which the GT delta functions corresponding to the original sinusoidal signals are recovered.
Clearly, the generated image fits the sinusoidals according to their frequency order, from small to large.

We further quantify in Figure \ref{fig:absdif_freqs} the convergence during the optimization by measuring the absolute difference between the GT and generated image Fourier transform amplitude for the 3 frequencies.
Clearly, lower frequencies converge faster.

\begin{figure*}[t]
    \centering
    \resizebox{1\linewidth}{!}{
    \begin{tabular}{cccc|c}
    \includegraphics[width=0.125\linewidth]{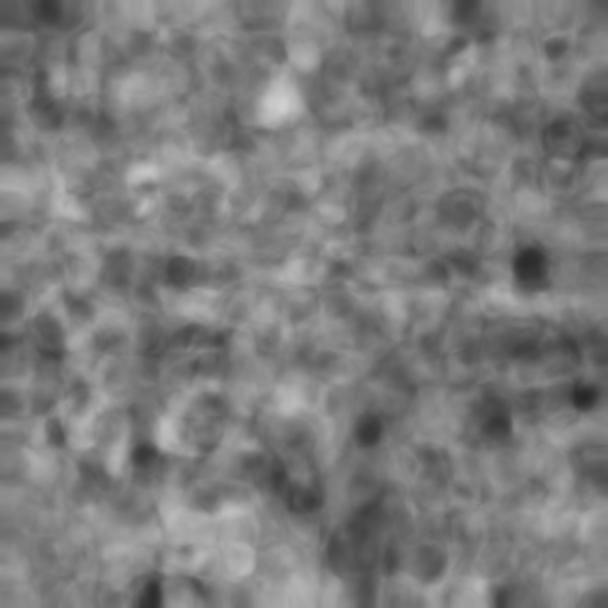} &
    \includegraphics[width=0.125\linewidth]{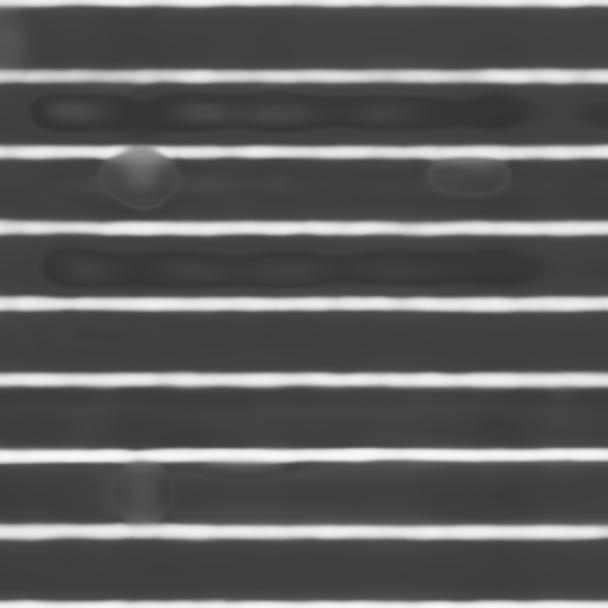} &
    \includegraphics[width=0.125\linewidth]{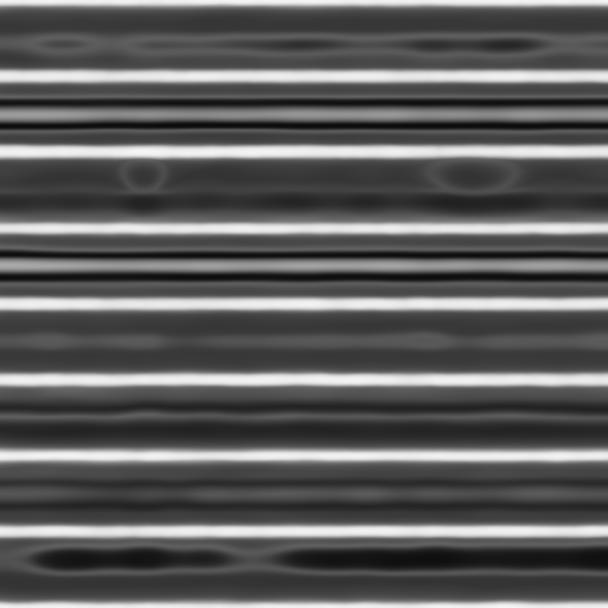} &
    \includegraphics[width=0.125\linewidth]{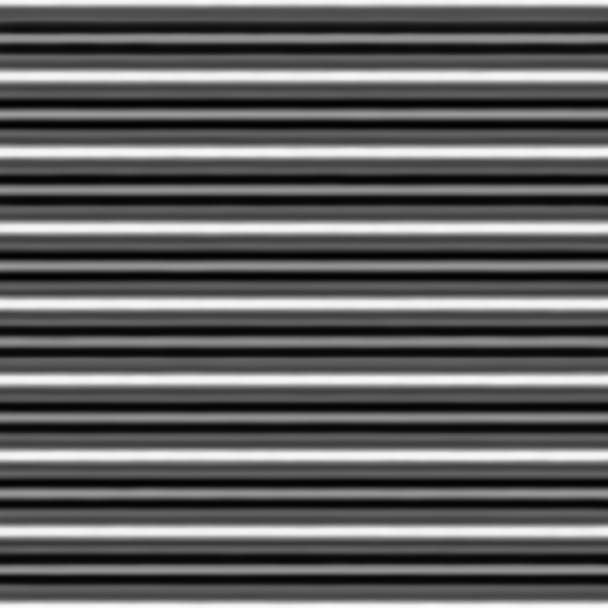} &
    \includegraphics[width=0.125\linewidth]{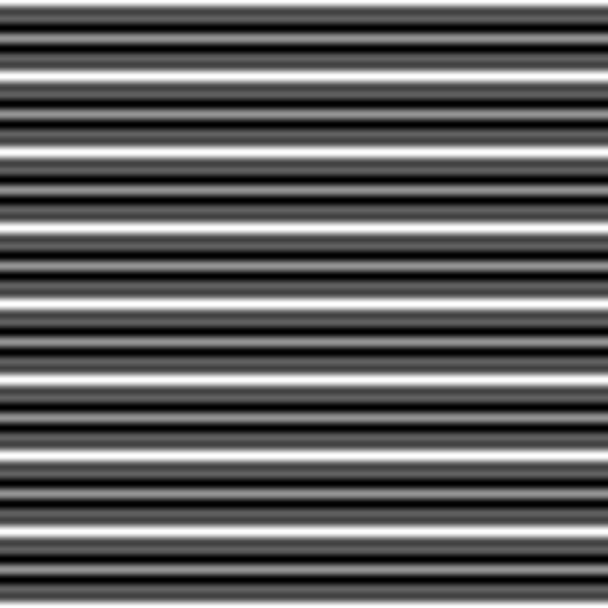} 
    \\
    \includegraphics[width=0.125\linewidth]{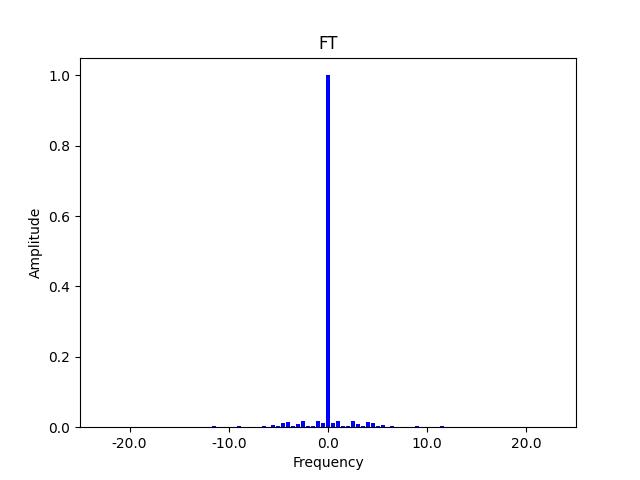} &
    \includegraphics[width=0.125\linewidth]{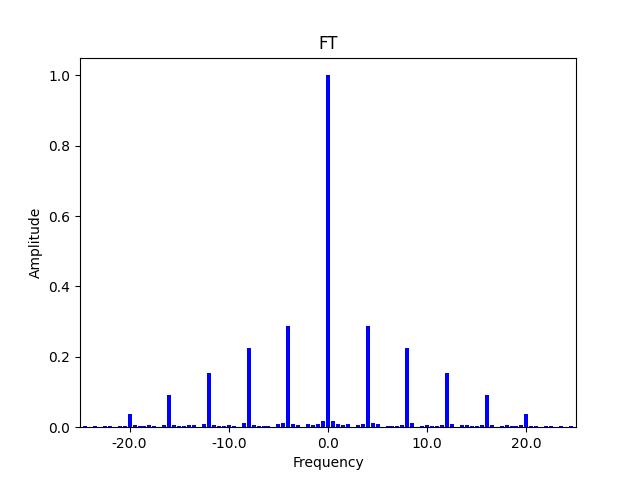} &
    \includegraphics[width=0.125\linewidth]{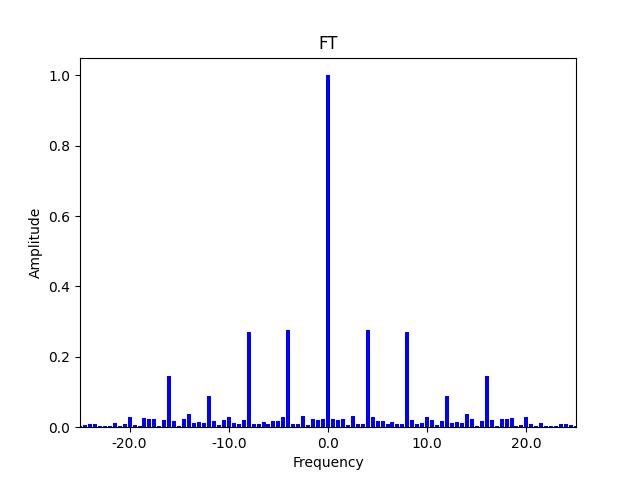} &
    \includegraphics[width=0.125\linewidth]{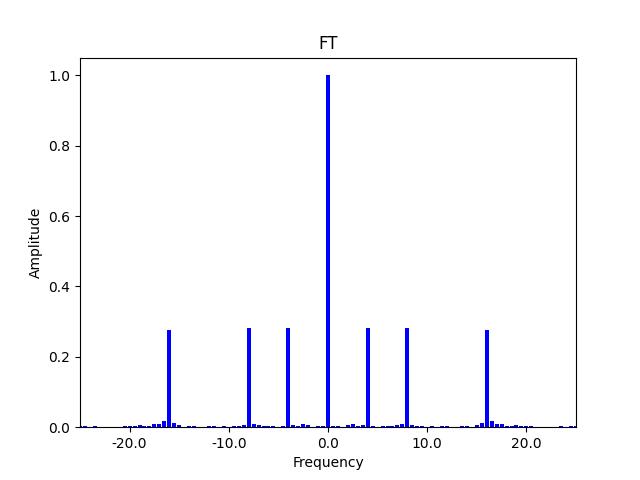} &
    \includegraphics[width=0.125\linewidth]{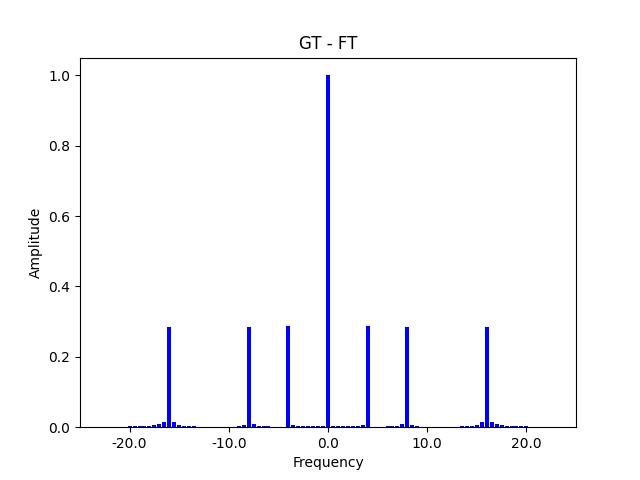} \\
    \scriptsize Iteration 0 & \scriptsize Iteration 40 & \scriptsize Iteration 80 & \scriptsize Iteration 200  & \scriptsize GT 
    \end{tabular}
    }
    \vspace{-5pt}
    \caption{\textbf{Spectral Bias.} We fit a PIP model to a synthetic image with 3 vertical frequencies (right column). We show the generated images as the training progresses (first row) and the Fourier transforms over a vertical line in the image (second row). We find that the model is biased towards first fitting the lower frequencies (lowest frequency at $\sim40$ iterations; middle at $\sim80$; highest at $\sim200$). This explains why when the model is early-stopped the image's lower frequencies are fitted while the high-frequency noise is removed. The highest frequency in the synthetic image is larger than $f_{max}$ in the model inputs, yet, the model is able to fit it too.}
    \vspace{-5pt}
    \label{fig:spectral_bias}
\end{figure*}

\begin{figure}[t]
    % \hspace{-20pt}
    \centering    
    \includegraphics[width=0.7\linewidth]{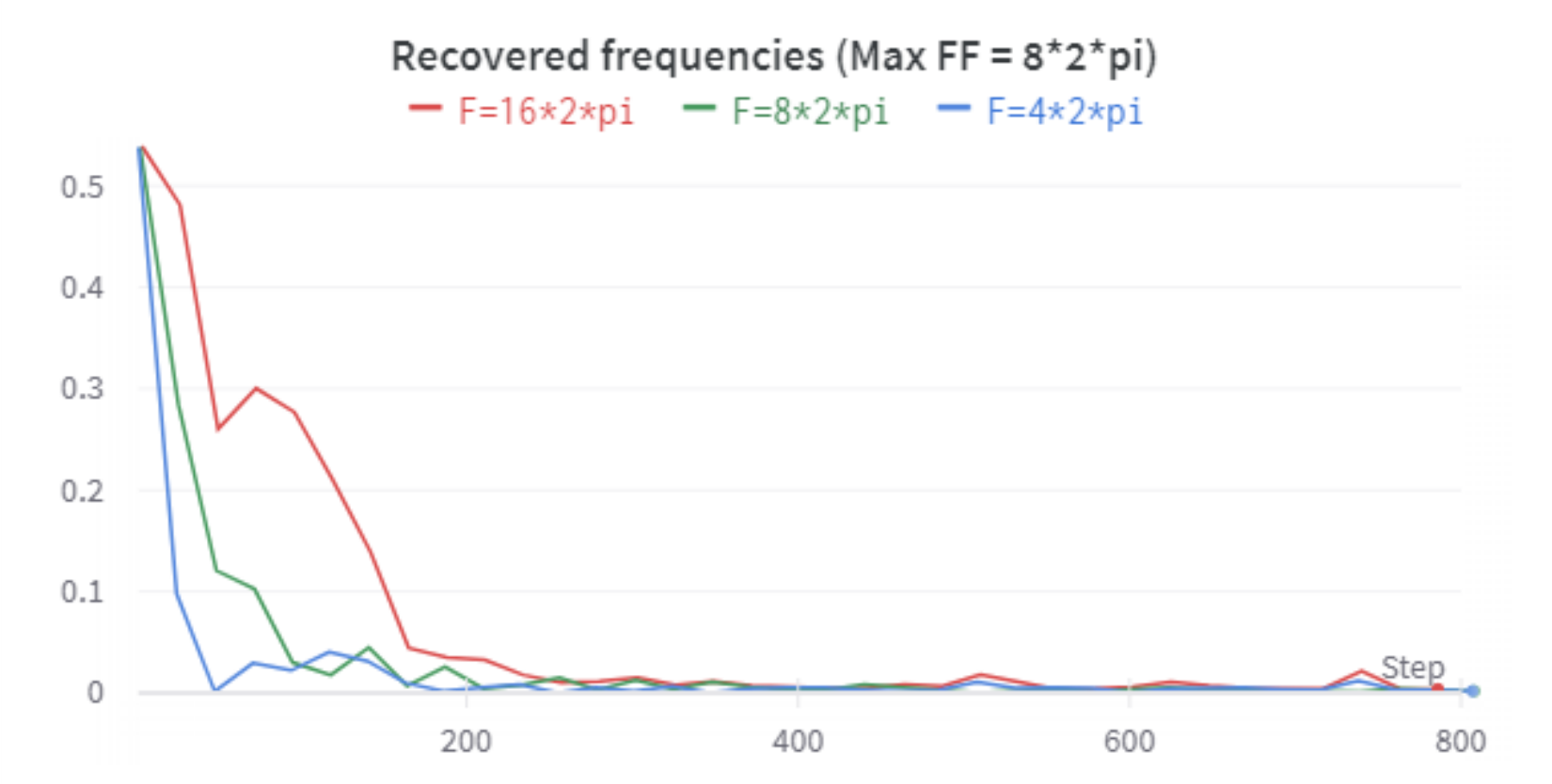}
    \vspace{-0.1in}
    \caption{\textbf{Spectral Bias.} We measure the absolute difference between the predicted frequencies' amplitude vs. the GT amplitudes. Lower frequencies are fitted first during the optimization.}
    \label{fig:absdif_freqs}
\end{figure}

\subsection{Additional Applications}
We further demonstrate that PIP can serve as a drop-in replacement in other applications DIP is used for.
% abilities on some follow up works done on DIP to show to improvement in computation time and memory while still able to gain comparable results. Further results can be found in the supplementary material.

\noindent {\bf Inpainting.} We follow the examples demonstrated for DIP. In both cases, we used the learned-FF configuration and increased the number of iterations to $8k$ instead of $6k$. Examples of our inpainting results are presented in Figs.~\ref{fig:inpainting} and \ref{fig:inpainting_sup}. 
The architecture used in DIP for hole inpainting is different than the architecture used in denoising and SR in its depth (6 levels instead of 5), in its building blocks (no skip connections and bigger spatial kernels), in its embedding depth at each layer, and in the input depth (changed from 32 to 1).
To make minimal adaptations to our PIP network we simply added another layer to match the same number of levels as DIP, we keep the rest of our setup exactly as described in the paper. 
Since PIP is less sensitive to the architecture choice, it achieves similar or better results compared to DIP with no need for re-designing the architecture for this task.

\begin{figure*}[t]
\vspace{-5pt}
    \centering
    \resizebox{0.85\linewidth}{!}{
    \begin{tabular}{ccccc}
    Original & Masked Image & DIP & PIP-CNN & PIP-MLP  \\
    \includegraphics[width=0.25\textwidth]{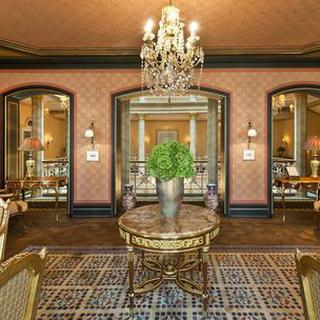} &
    \includegraphics[width=0.25\textwidth]{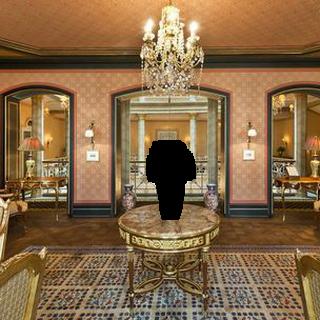} &
    \includegraphics[width=0.25\textwidth]{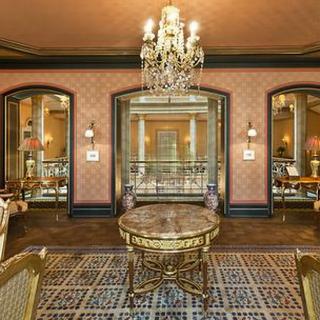} &
        \includegraphics[width=0.25\textwidth]{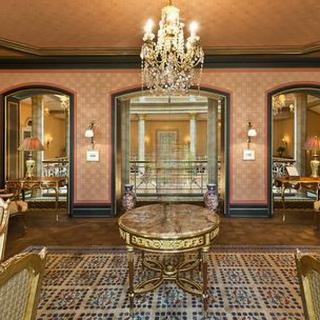} &
    \includegraphics[width=0.25\textwidth]{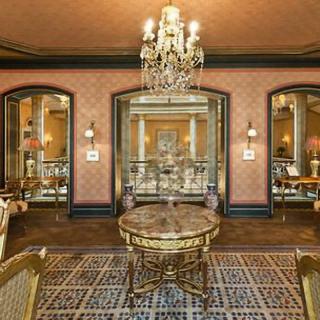}\\
    \includegraphics[width=0.25\textwidth]{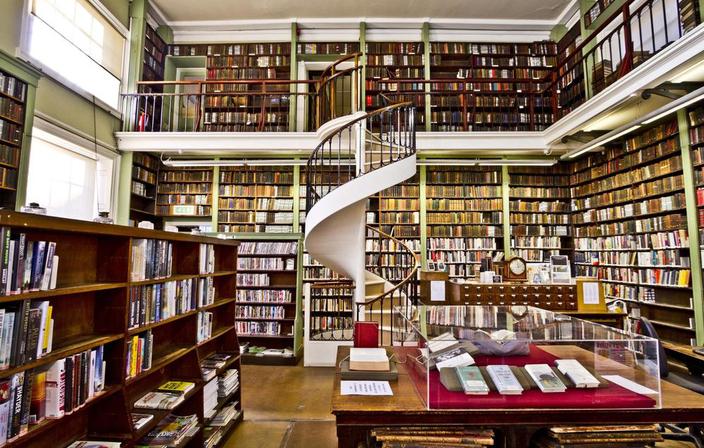} &
    \includegraphics[width=0.25\textwidth]{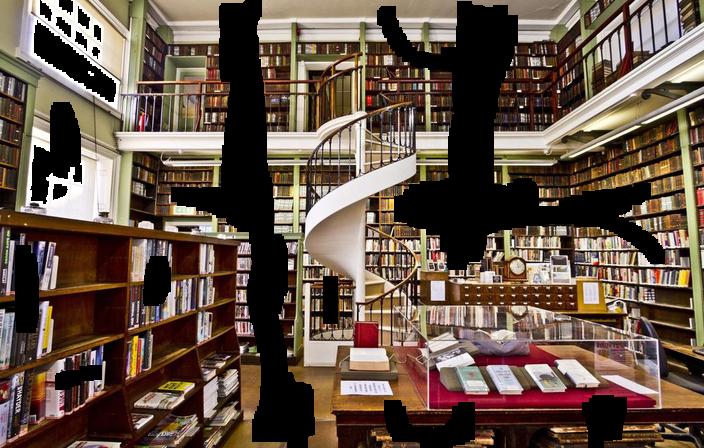} &
    \includegraphics[width=0.25\textwidth]{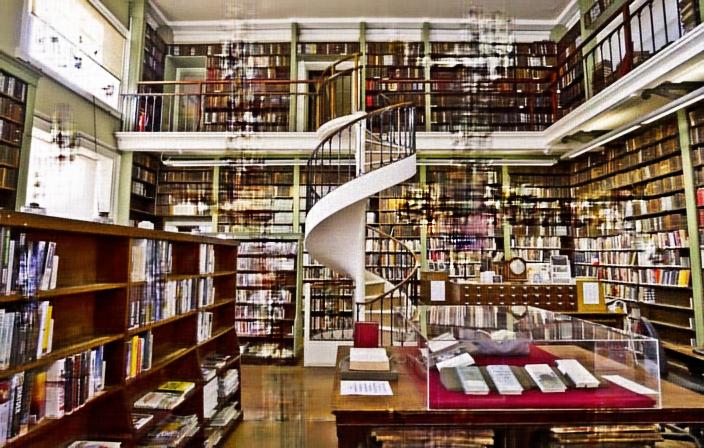} &
        \includegraphics[width=0.25\textwidth]{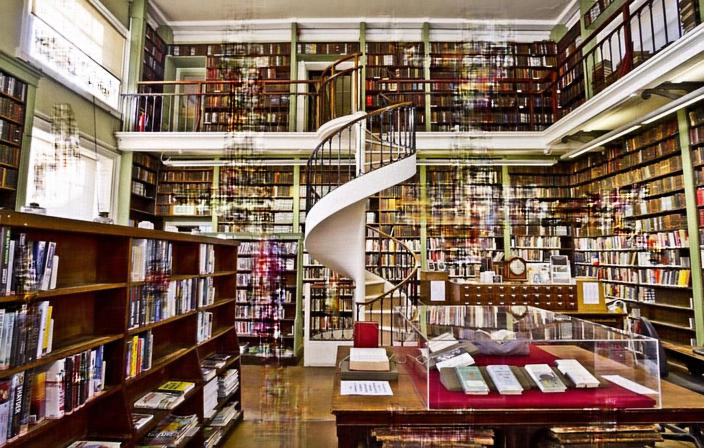} &
    \includegraphics[width=0.25\textwidth]{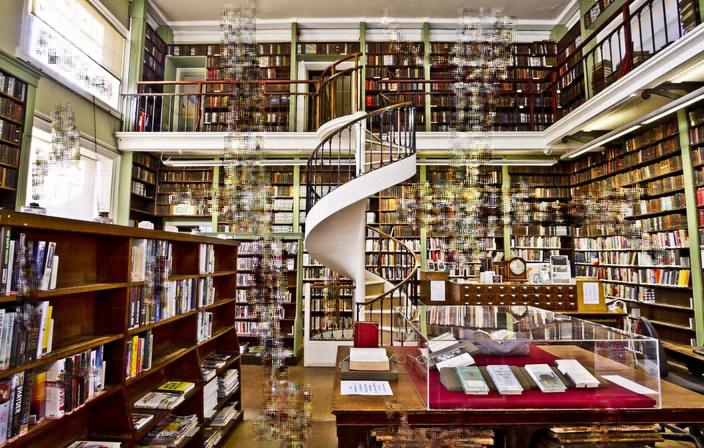}\\
    \end{tabular}
    }
    \vspace{-5pt}
    \caption{\textbf{Image Inpainting.} A comparison for large hole inpainting. PIP-CNN outperforms visually in both images, where PIP-MLP shows comparable results to DIP without any spatial kernels.}
    \label{fig:inpainting}
    \vspace{-5pt}
\end{figure*}

\noindent {\bf Double-PIP - Dehazing.}
Image dehazing is the problem of extracting a haze-free image from a hazy image. Formally, the image acquisition is modeled as
$I(x)=J(x)t(x)+A(1-t)(x)$ where I is a hazy image, J is a haze-free image, A is the airlight model and $t(x)=e^{-\beta d(x)}$. In Double-DIP \cite{DoubleDIP}, three DIP models are trained together to generate $J(x)$ and $t(x)$ ($A$ can be derived from other algorithms or learned) with respect to the above model.
Again we evaluate PIP as a drop-in replacement.
We switched all noise inputs to FF and the original CNNs to our MLPs.
Figs. \ref{fig:double_pip} and \ref{fig:dehazing2} show that PIP may replace DIP also here.

\begin{figure}[t]
\vspace{-5pt}
    \centering
    \resizebox{1\linewidth}{!}{
    \begin{tabular}{ccc}
    Input & DIP & PIP (Ours)  \\
    \includegraphics[width=0.25\textwidth]{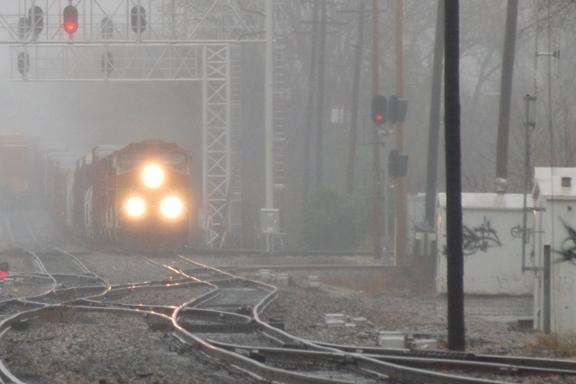} &
    \includegraphics[width=0.25\textwidth]{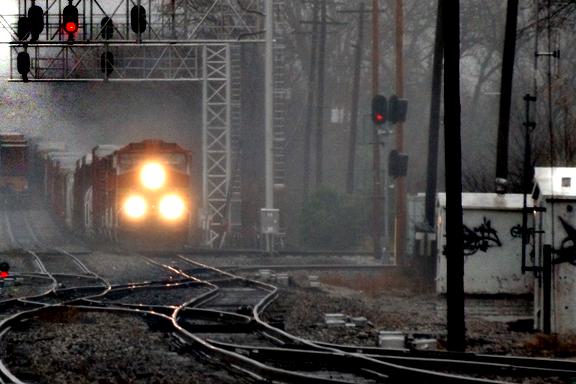} &
    \includegraphics[width=0.25\textwidth]{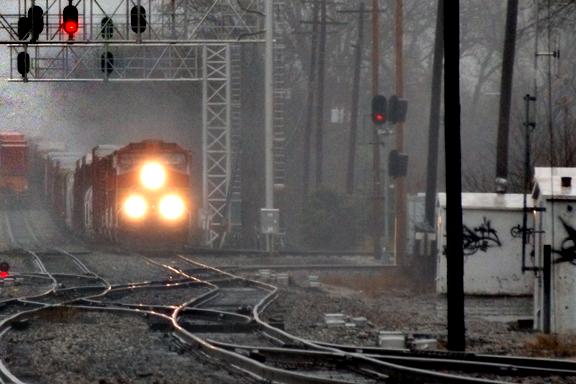}\\
    \includegraphics[width=0.25\textwidth]{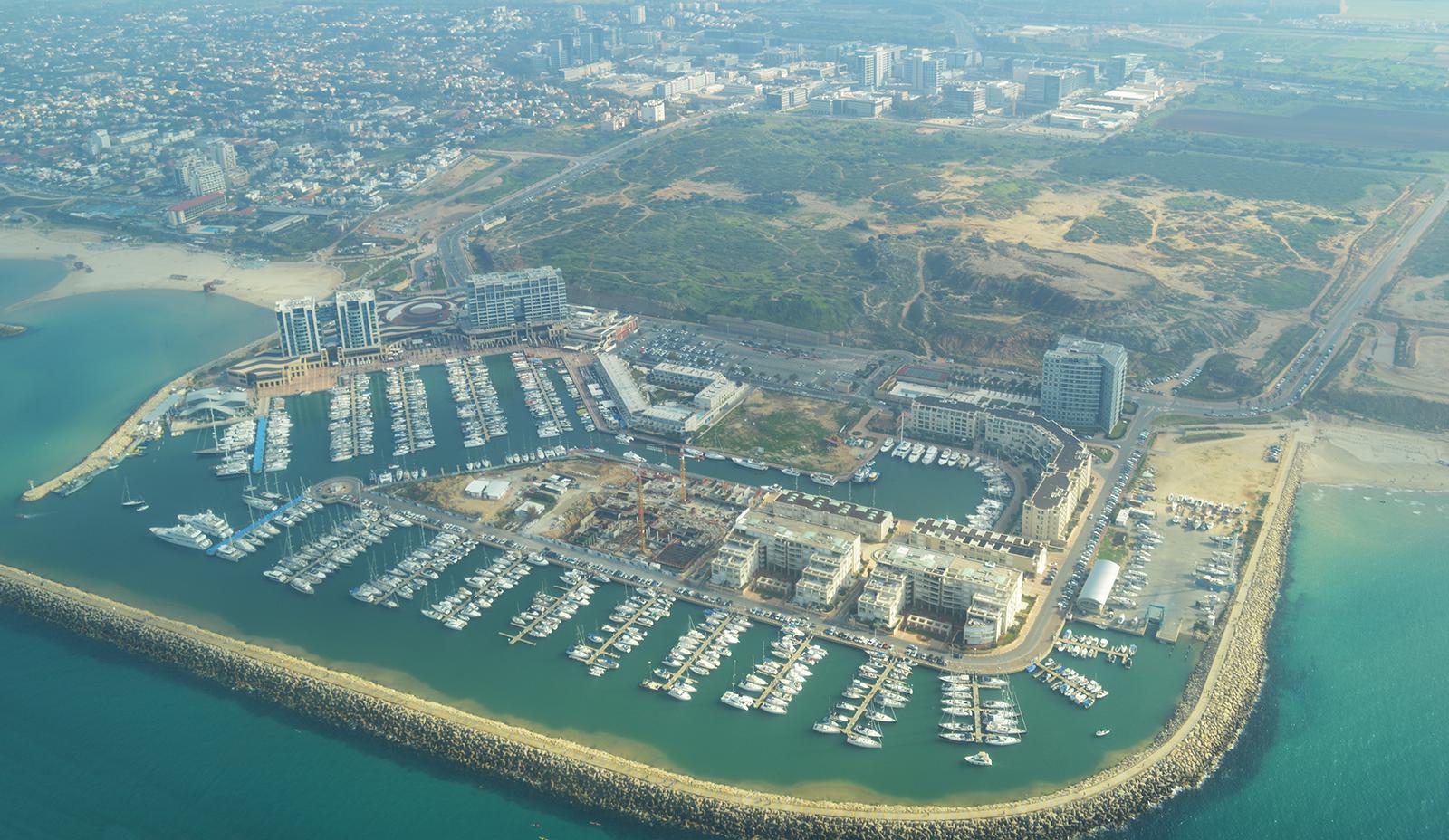} &
    \includegraphics[width=0.25\textwidth]{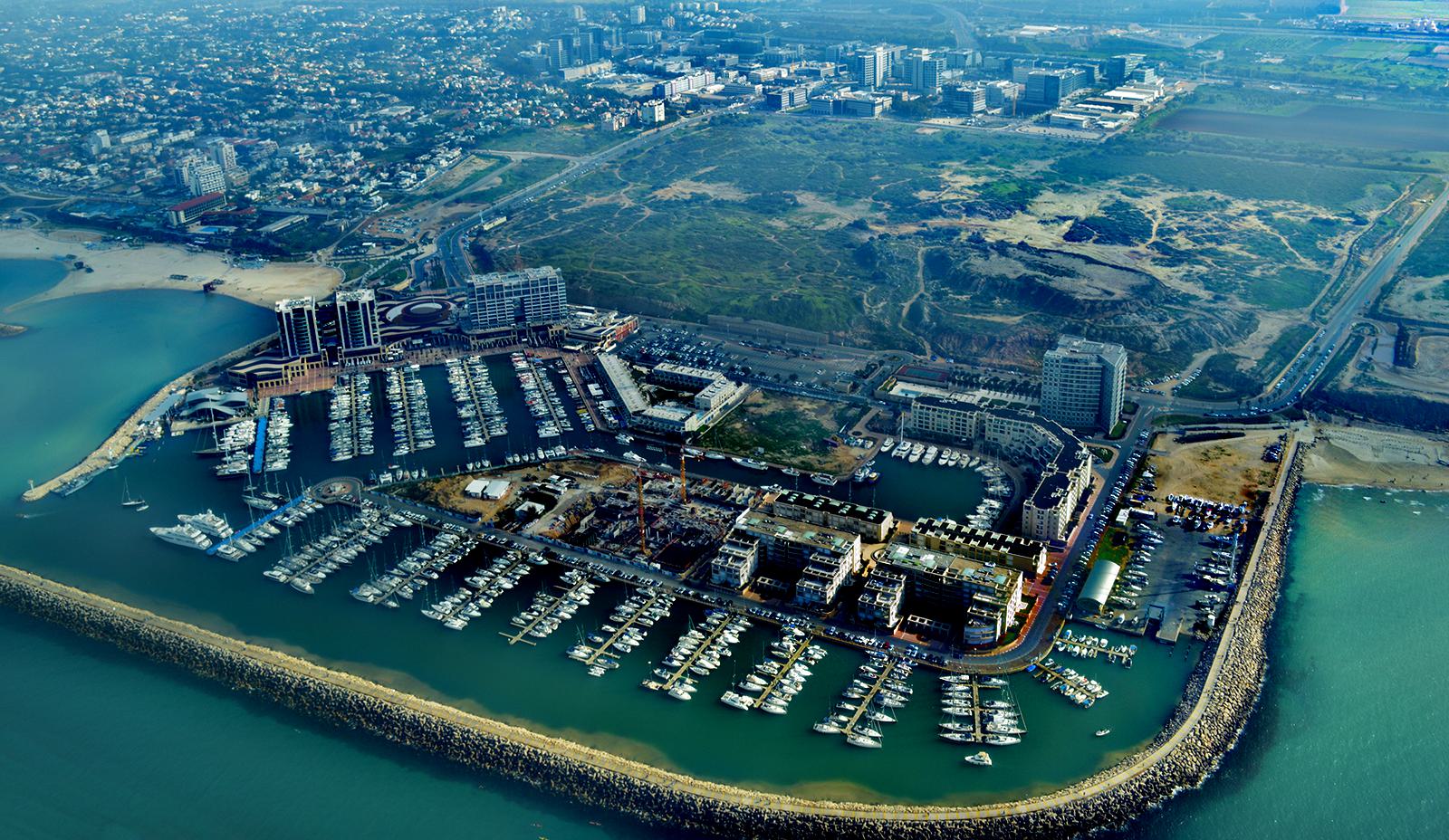} &
    \includegraphics[width=0.25\textwidth]{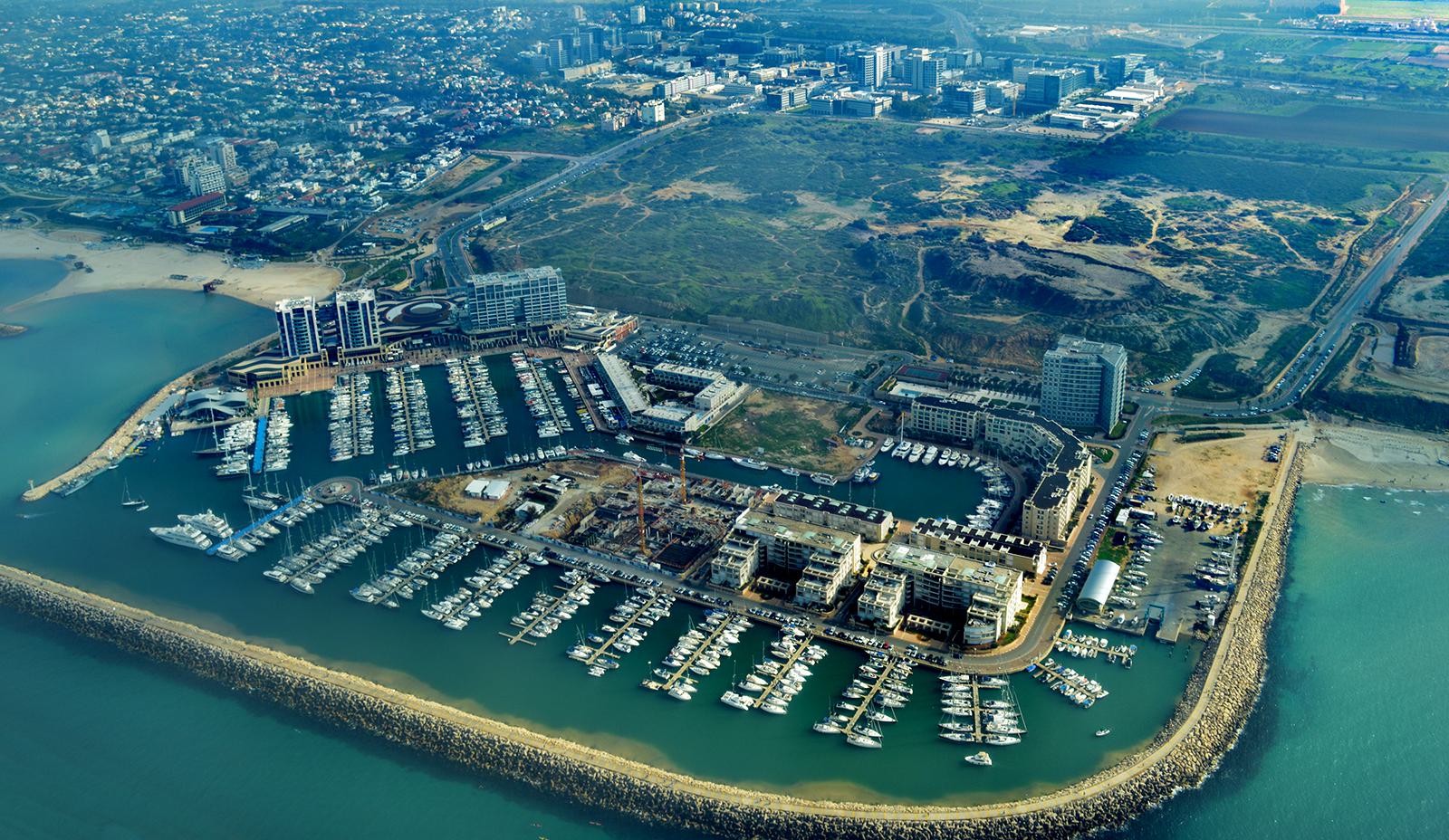}\\
    \includegraphics[width=0.25\textwidth]{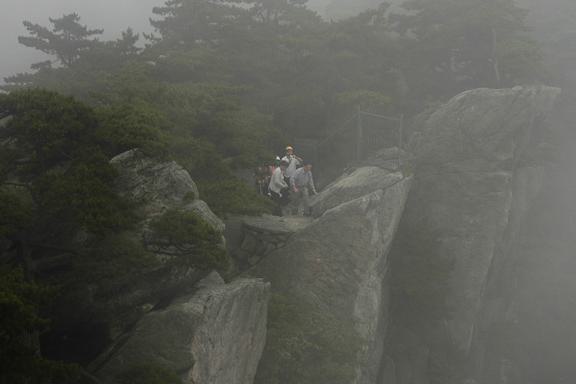} &
    \includegraphics[width=0.25\textwidth]{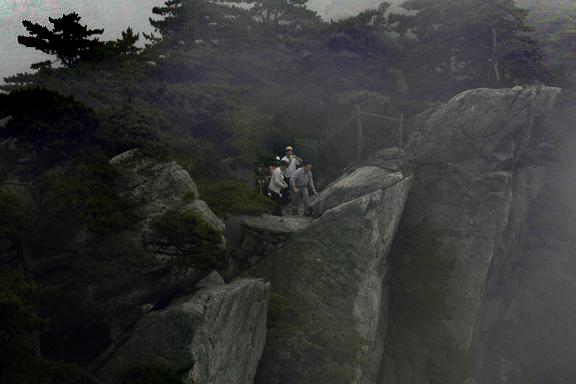} &
    \includegraphics[width=0.25\textwidth]{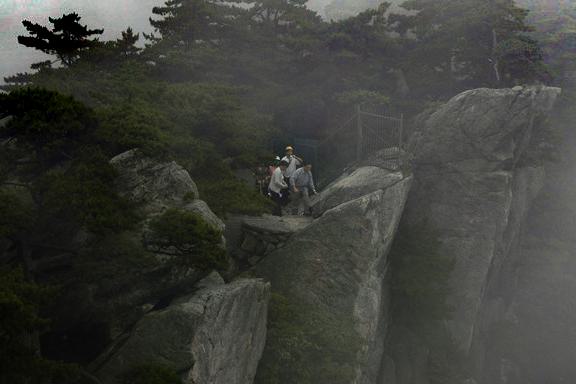} \\
    \includegraphics[width=0.25\textwidth]{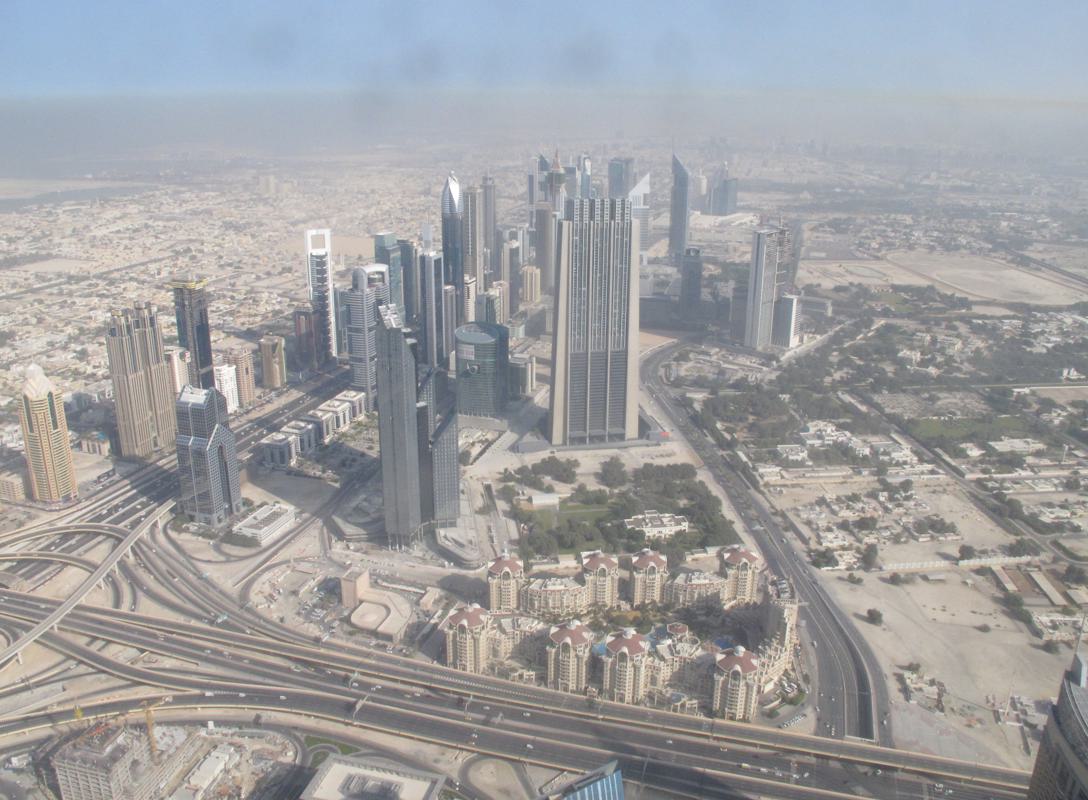} &
    \includegraphics[width=0.25\textwidth]{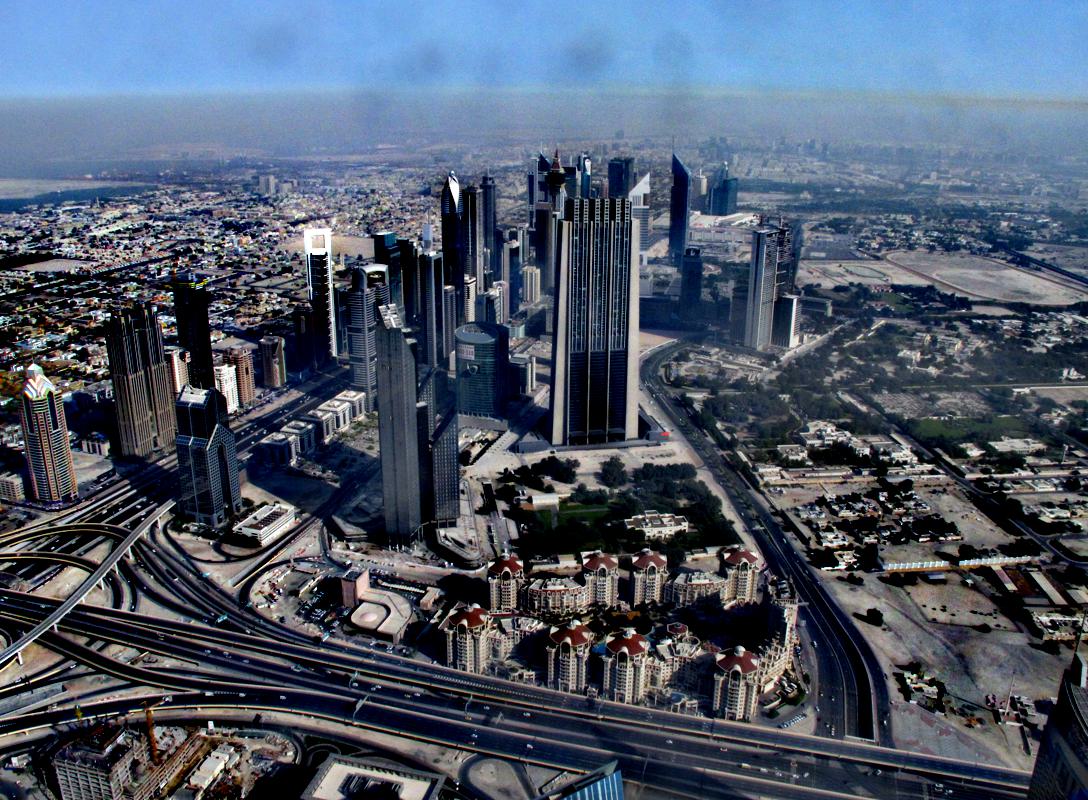} &
    \includegraphics[width=0.25\textwidth]{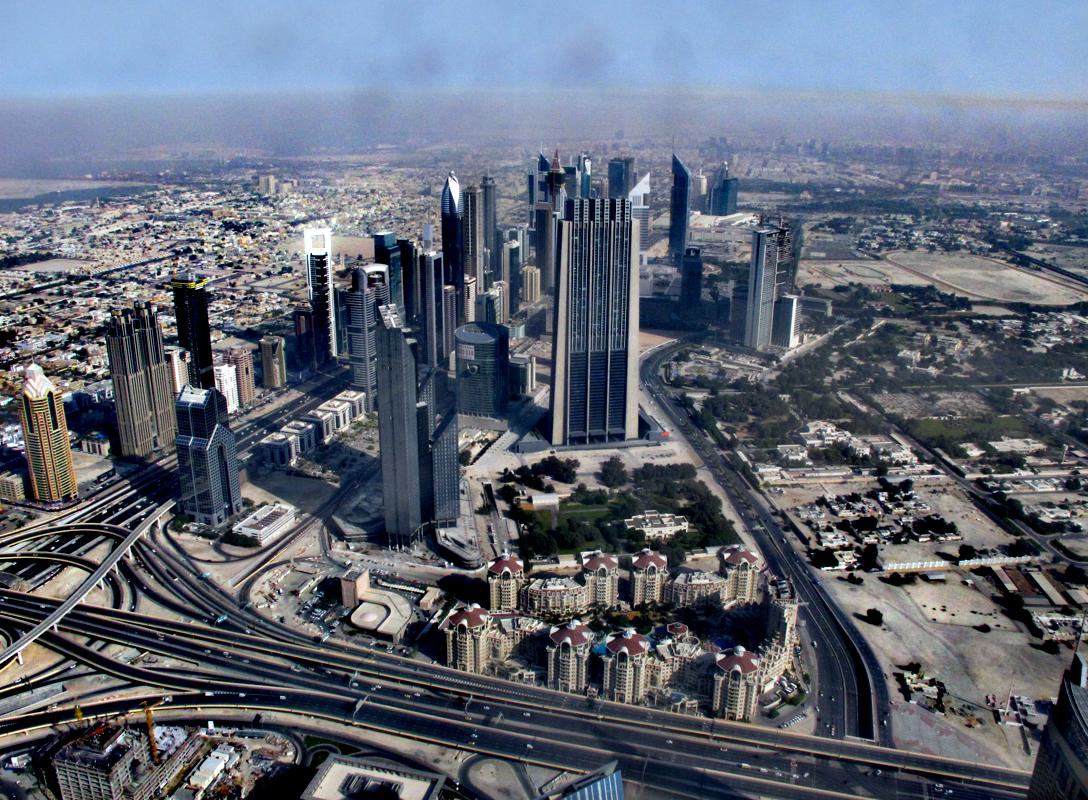} \\
    \includegraphics[width=0.25\textwidth]{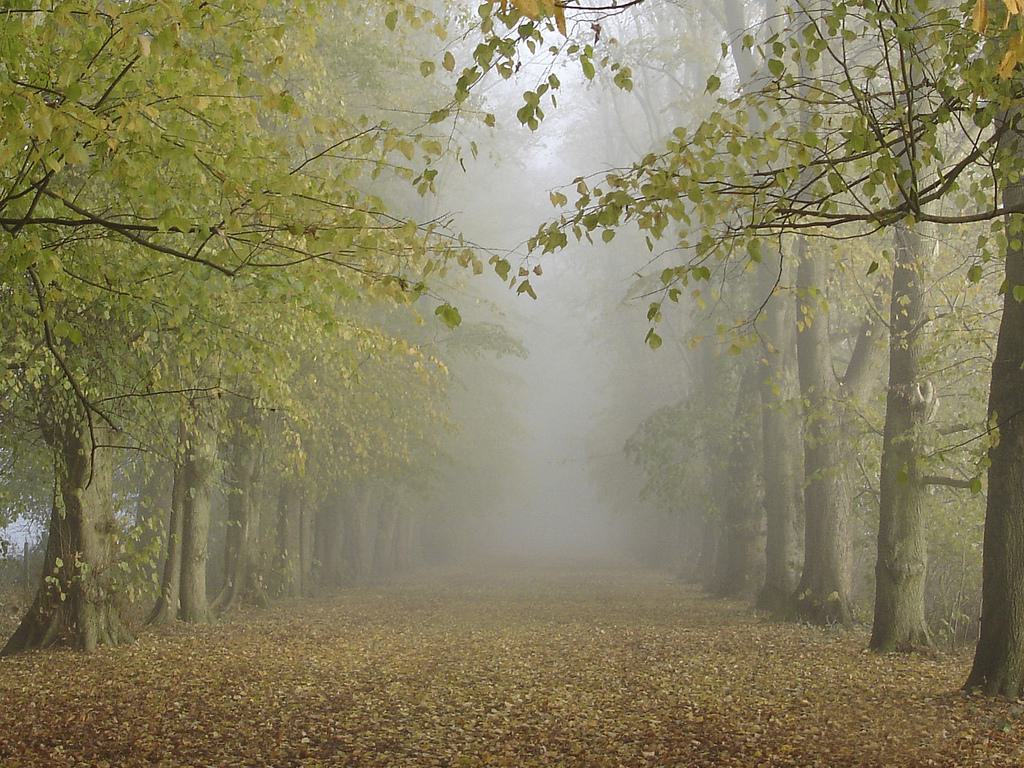} &
    \includegraphics[width=0.25\textwidth]{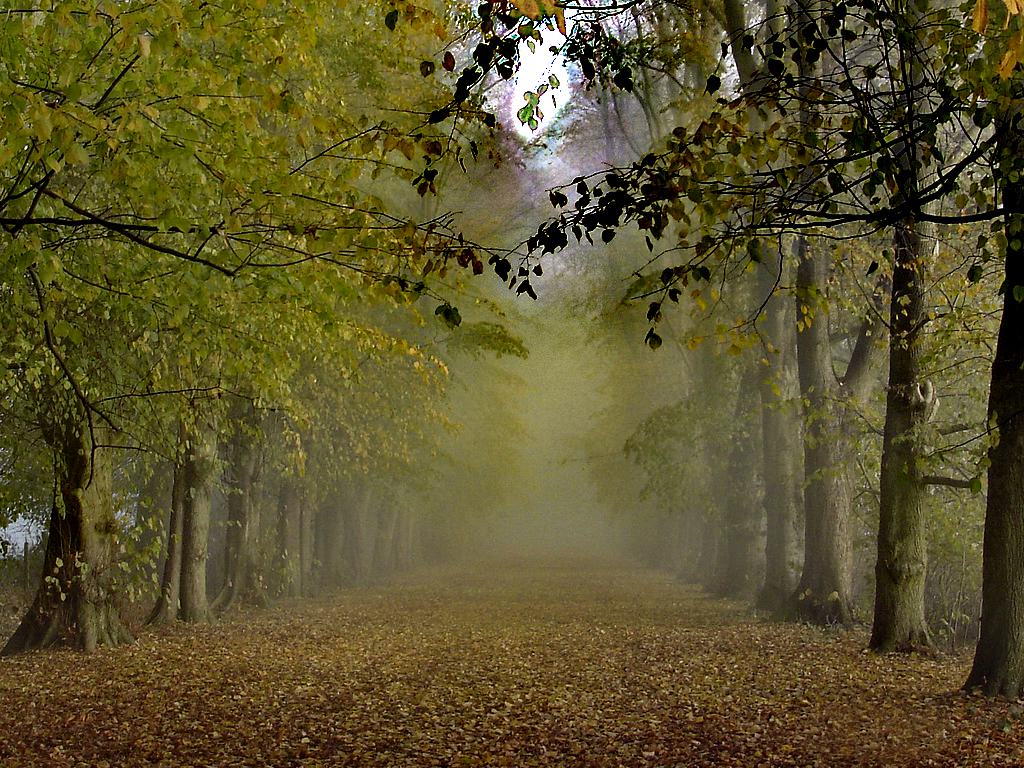} &
    \includegraphics[width=0.25\textwidth]{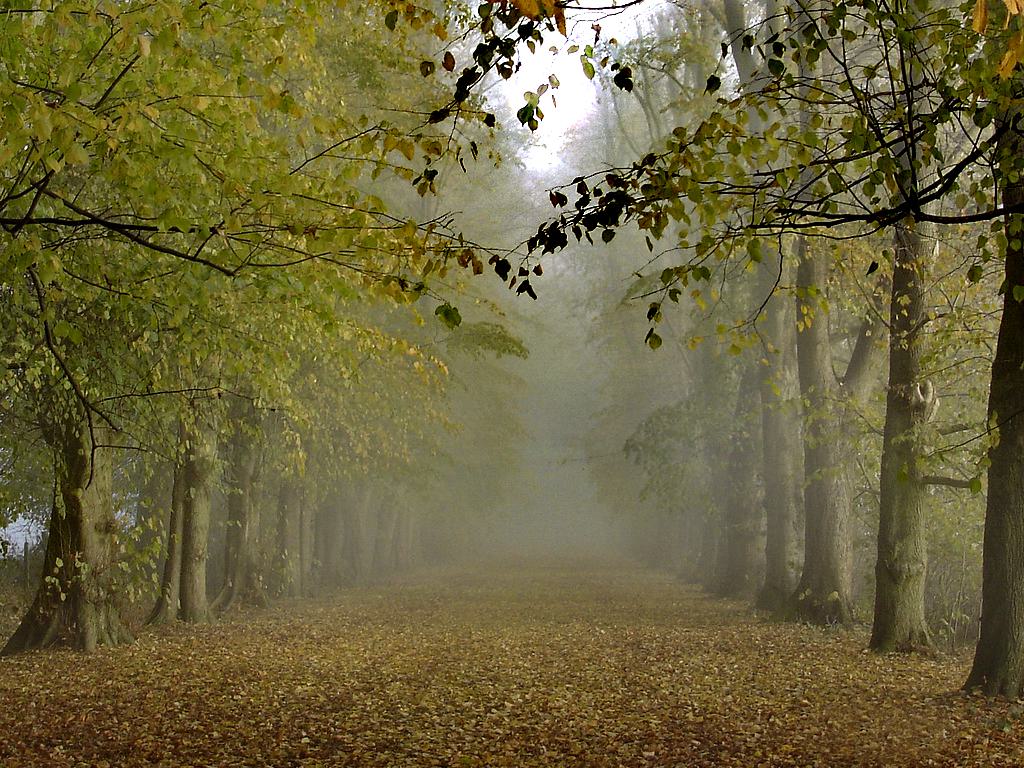} \\
    \includegraphics[width=0.25\textwidth]{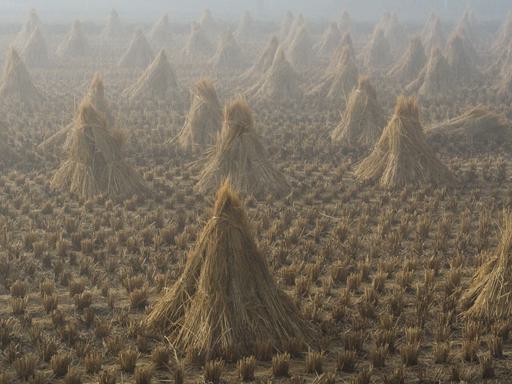} &
    \includegraphics[width=0.25\textwidth]{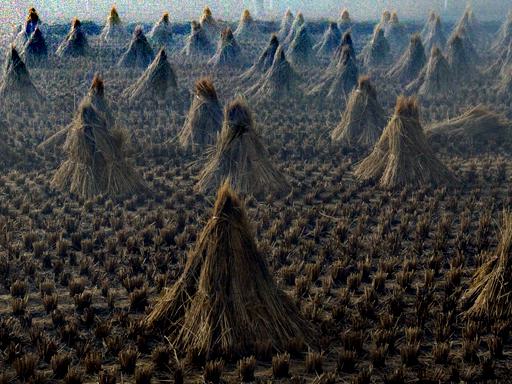} &
    \includegraphics[width=0.25\textwidth]{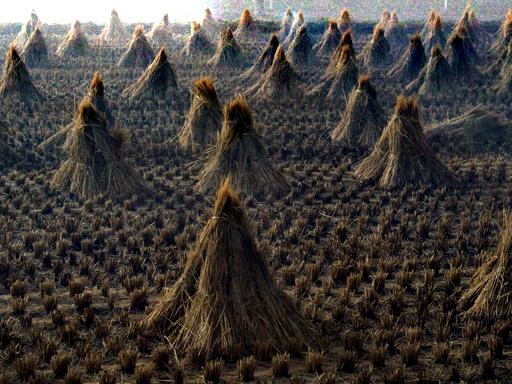} \\  
    \end{tabular}
    }
    \vspace{-5pt}
    \caption{\textbf{Blind image dehazing.} We compare ``Double-PIP" to ``Double-DIP", where both methods produce comparable visual results.}
    \label{fig:double_pip}
    \vspace{-5pt}
\end{figure}

\noindent {\bf CLIP inversion.}
Similarly to the original DIP inversion of AlexNet/VGG (generating images from class labels), recently DIP has been used for generating images from textual descriptions using the CLIP model \cite{CLIP_inversion}. DIP serves as an image generator conditioned on input text, such that the generated image embedding is close to the text embedding. Figs.~\ref{fig:clip}, \ref{fig:clip_0} and \ref{fig:clip_2} show examples of image generation by inverting CLIP using both DIP and PIP.

\begin{figure}[t]
% \vspace{-10pt}
    \centering
    \resizebox{\linewidth}{!}{
    \begin{tabular}{lccc}
         \multirow{5}{*}{DIP} &
         \multirow{5}{*}{\includegraphics[width=0.3\linewidth]{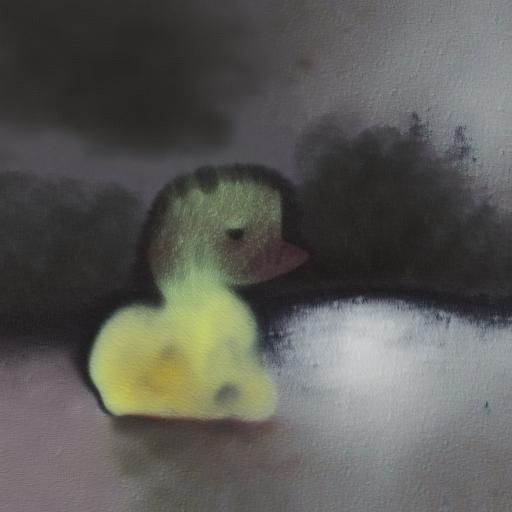}} & 
         \multirow{5}{*}{\includegraphics[width=0.3\linewidth]{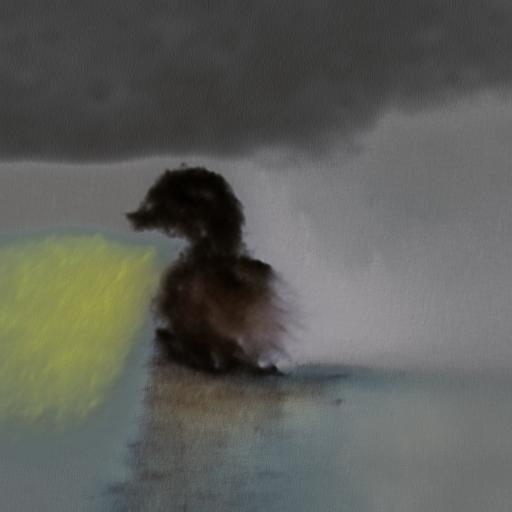}} & 
         \multirow{5}{*}{\includegraphics[width=0.3\linewidth]{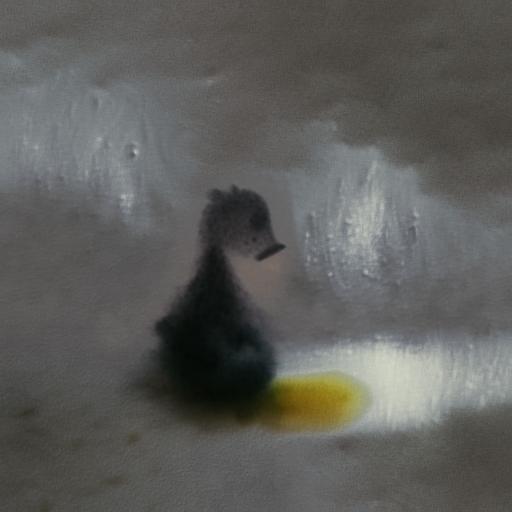}} \\ [1in]
         \multirow{5}{*}{PIP} &
         \multirow{5}{*}{\includegraphics[width=0.3\linewidth]{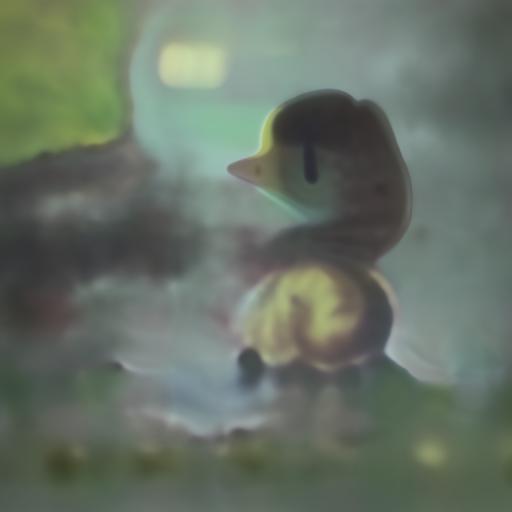}} & 
         \multirow{5}{*}{\includegraphics[width=0.3\linewidth]{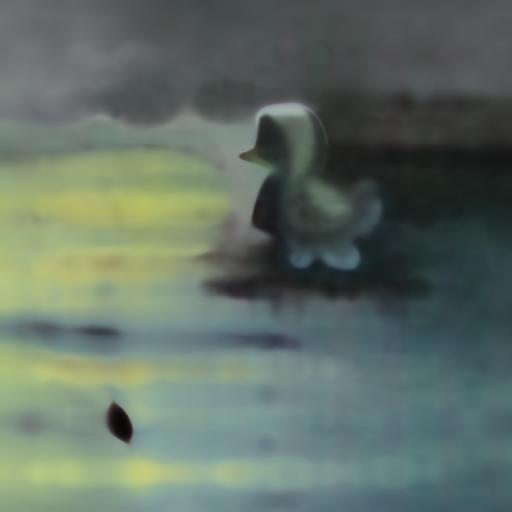}} & 
         \multirow{5}{*}{\includegraphics[width=0.3\linewidth]{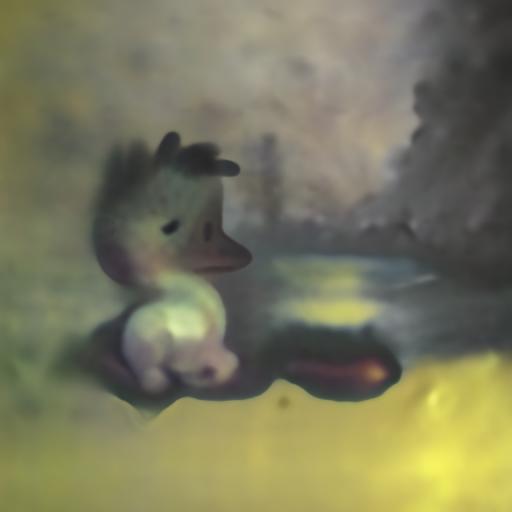}} \\[1in]
         
    \end{tabular}
    }
    \caption{\textbf{CLIP inversion.} Generated images for the description: `A moody painting of a lonely duckling'.}
    \label{fig:clip}    
\end{figure}

\begin{figure}[]
% \vspace{-10pt}
    \resizebox{\columnwidth}{!}{
    \begin{tabular}{lccc}
         \multirow{5}{*}{DIP} &
         \multirow{5}{*}{\includegraphics[width=0.25\linewidth]{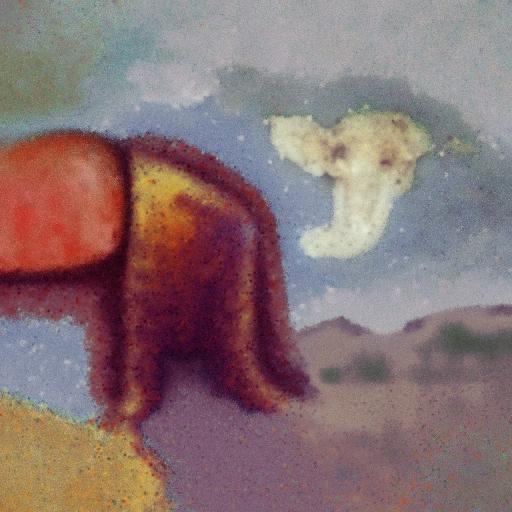}} & 
         \multirow{5}{*}{\includegraphics[width=0.25\linewidth]{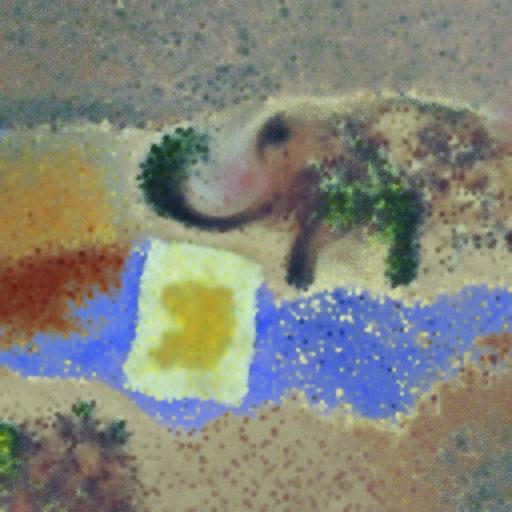}} & 
         \multirow{5}{*}{\includegraphics[width=0.25\linewidth]{sup_mat_figures/Image/CLIP/DIP/an_oil_painting_of_an_elephant_sleeping_in_the_desert_1.jpg}} \\[0.8in]
         \multirow{5}{*}{PIP} &
         \multirow{5}{*}{\includegraphics[width=0.25\linewidth]{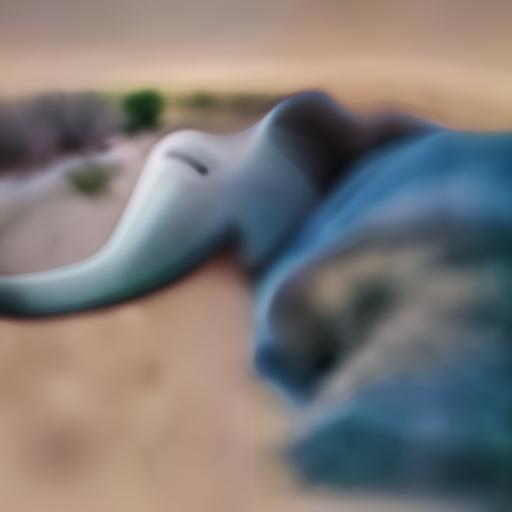}} & 
         \multirow{5}{*}{\includegraphics[width=0.25\linewidth]{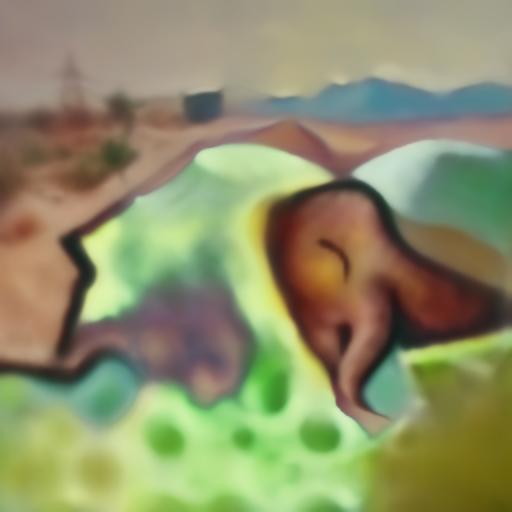}} & 
         \multirow{5}{*}{\includegraphics[width=0.25\linewidth]{sup_mat_figures/Image/CLIP/PIP/an_oil_painting_of_an_elephant_sleeping_in_the_desert_1.jpg}} \\[0.8in]
    \end{tabular}
    }    
    \caption{\textbf{CLIP inversion.} Generated images for the description: `An oil painting of an elephant sleeping in the desert'.}
    \label{fig:clip_0}    
\end{figure}

\subsection{Extension to Video}

\begin{figure}[t]
    \centering    
    \resizebox{1\linewidth}{!}{
    \begin{tabular}{p{1cm} cc}
     & frame 7 & frame 15 \\
    \multirow{5}{*}{GT} &
    \multirow{5}{*}{\includegraphics[width=0.25\textwidth]{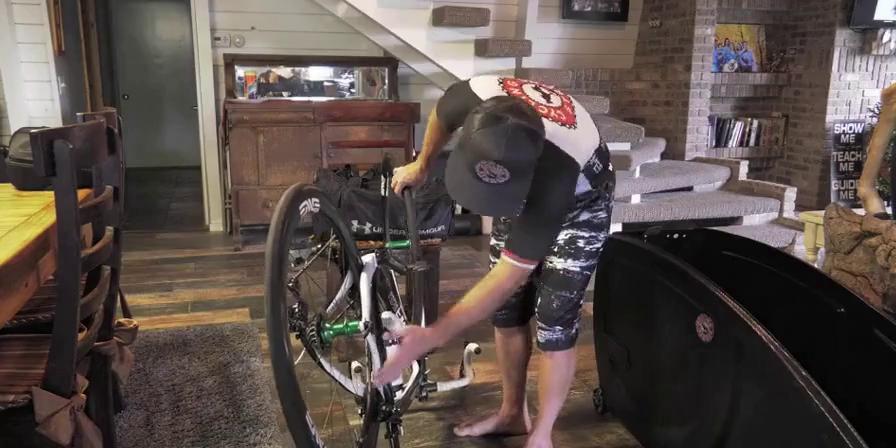}} &
    \multirow{5}{*}{\includegraphics[width=0.25\textwidth]{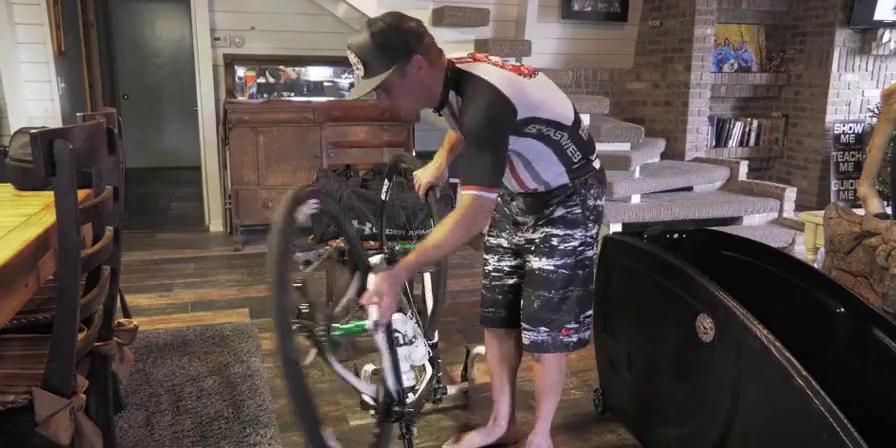}} \\[0.7in]
    \multirow{5}{*}{Noisy} &
    \multirow{5}{*}{\includegraphics[width=0.25\textwidth]{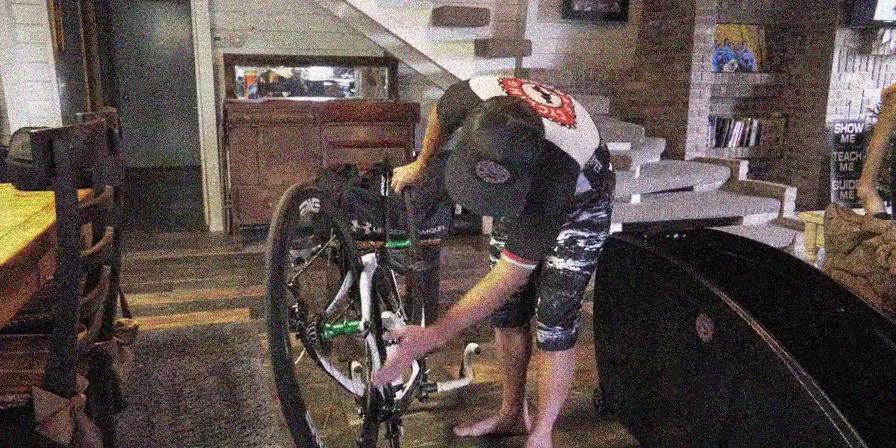}} &
    \multirow{5}{*}{\includegraphics[width=0.25\textwidth]{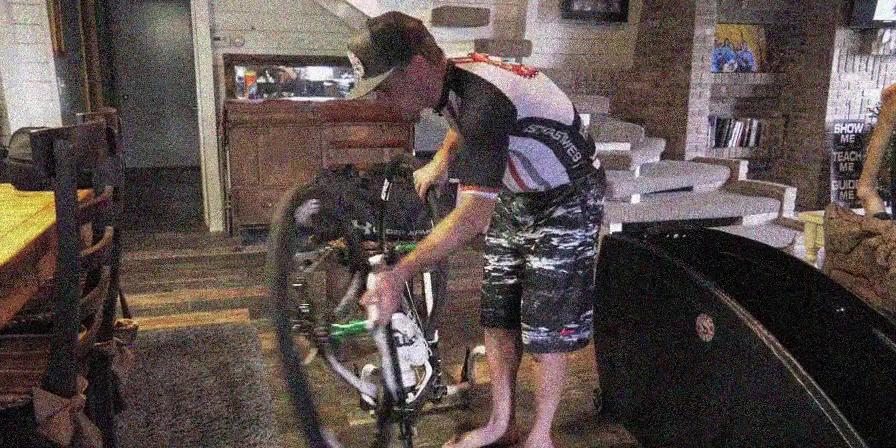}} \\[0.7in]    
    \multirow{5}{*}{3D-DIP} &
    \multirow{5}{*}{\includegraphics[width=0.25\textwidth]{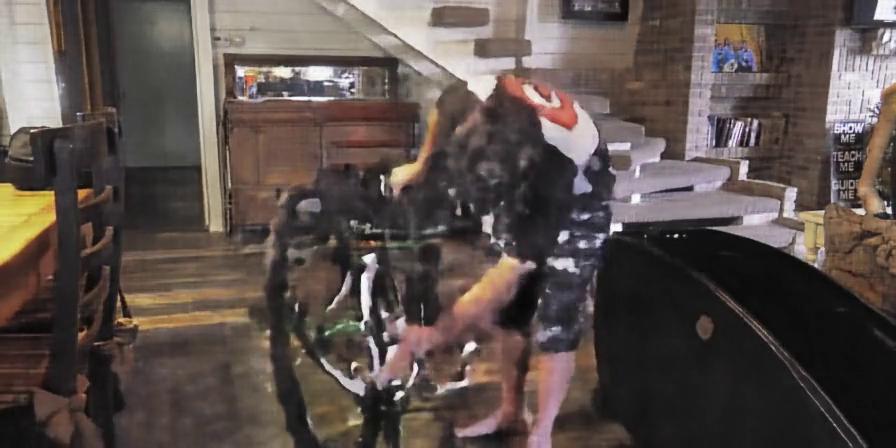}} &
    \multirow{5}{*}{\includegraphics[width=0.25\textwidth]{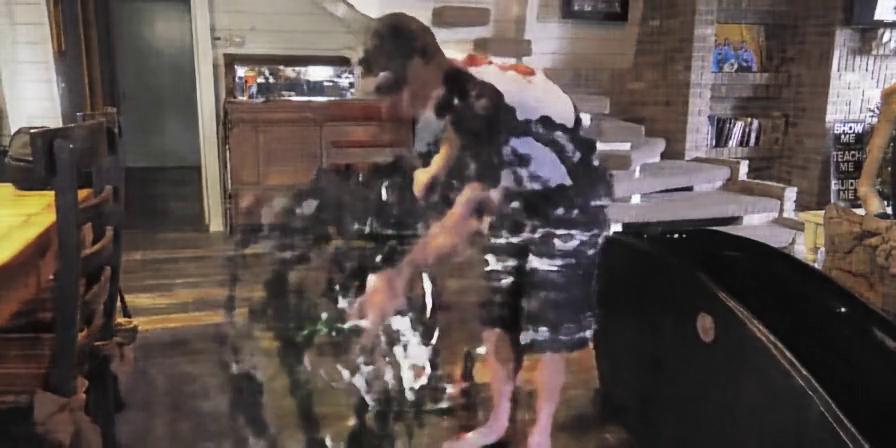}} \\[0.7in]
    \multirow{5}{*}{3D-PIP} &
    \multirow{5}{*}{\includegraphics[width=0.25\textwidth]{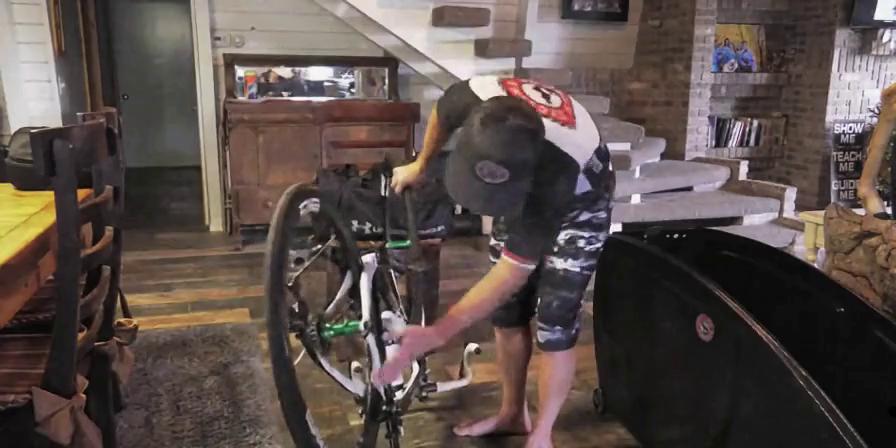}} &
    \multirow{5}{*}{\includegraphics[width=0.25\textwidth]{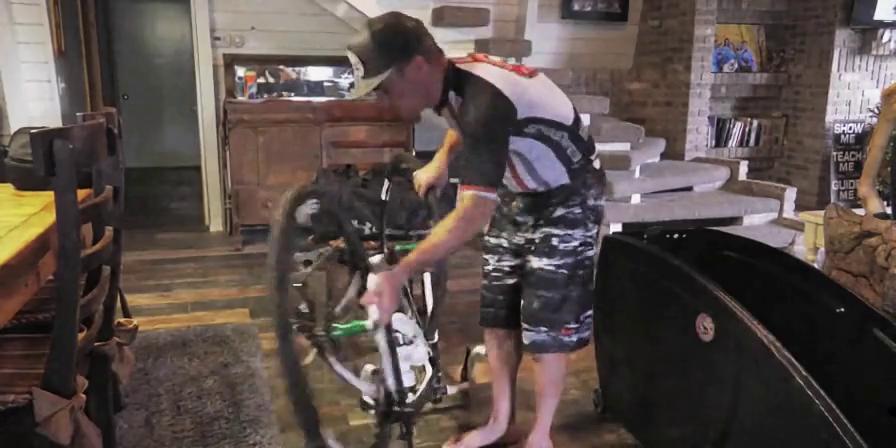}} \\[0.7in]
    \end{tabular}
    }    
    \caption{\textbf{Video denoising.} We compare 3D-PIP to 
    % 2D-DIP (frame-by-frame) and 
    3D-DIP (3D conv.). 3D-PIP generates better images (e.g. the grass area) and achieves better temporal consistency (full videos in the project page). }
    \label{fig:video_denoising}    
\end{figure}

\begin{table}[t]
\caption{Image Prior Video-Denoising and SR: PSNR/3D-SSIM. 3D-PIP outperforms 2D-DIP (frame-by-frame) and 3D-DIP both in terms of pixel reconstruction (PSNR) and temporal consistency (3D-SSIM).}
    \label{tab:video_denoising_image_prior}
    \centering
    \resizebox{\linewidth}{!}{
    \begin{tabular}{l|ccc|c}
    \toprule
    & \multicolumn{3}{c|}{Denoising} &  SR \\
    Method & $\sigma=25$ & $\sigma=10$ & Poisson & x4 \\
     \midrule
    2D-DIP \cite{ulyanov2018deep}&26.35/0.80 & 26.63/0.81 & 26.72/0.81 & 25.54/0.81\\
    3D-DIP \cite{Zhang_2019_ICCV}&24.92/0.82 & 25.15/0.83 & 25.27/0.83 & 22.91/0.72\\
    DVP    \cite{lei2020blind}   &25.88/0.83 & 26.01/0.83 & 25.75/0.83 & 24.29/0.79\\
    2D-PIP &26.17/0.80 & 26.61/0.81 & 26.67/0.81 & 25.45/0.78\\
    3D-PIP &\textbf{29.29/0.9}&\textbf{30.06/0.91} & \textbf{29.97/0.91} & \textbf{26.04/0.83}\\
    \bottomrule
    \end{tabular}
    }   
\end{table}

\begin{table}[]
\caption{Implicit models Video-Denoising and SR: PSNR/3D-SSIM. 3D-PIP outperforms SIREN and Gaussian Activations both in terms of pixel reconstruction (PSNR $\uparrow$) and temporal consistency (3D-SSIM $\uparrow$).}
    \label{tab:video_denoising_implicit}
    \centering
    \resizebox{\linewidth}{!}{
    \begin{tabular}{l|ccc|c}
    \toprule
    & \multicolumn{3}{c|}{Denoising} &  SR \\
    Method & $\sigma=25$ & $\sigma=10$ & Poisson & x4 \\
     \midrule
    SIREN \cite{sitzmann2019siren}            & 26.55/0.72 & 28.43/0.79 & 28.35/0.79 & 23.33/0.57\\
    Gauss. Act. \cite{ramasinghe2022beyond}   & 23.85/0.74 & 24.10/0.74 & 24.04/0.74 & 22.86/0.71\\    
    3D-PIP                                    & \textbf{29.29/0.9}&\textbf{30.06/0.91} & \textbf{29.97/0.91} & \textbf{26.04/0.83}\\
    \bottomrule
    \end{tabular}
    }   
\end{table}

We tested an extension of PIP to video (3D-PIP) denoising and SR.
To that end, we extend our positional encoding from 2D to 3D and encode the temporal domain as well as the spatial domain. Instead of stacking two spatial PEs (one for each coordinate),  we stack three PEs (adding temporal encoding). 
Since in the case of a video, the model needs to implicitly represent a full video and not a single frame, we increase the model's capacity. We change the model to have 6 levels instead of 5 and double the representation depth.
An important observation is that other than increasing the model's capacity, PIP does not use any 3D building blocks (such as 3D convolutional layers or tri-linear interpolation), or any additional loss and regularization as done in \cite{chen2021learning, lu2021}. 
% We compare our self to another blind denoising method \cite{claus2019videnn}. 

We compare 3D-PIP to 2D-PIP (frame-by-frame) to various image prior based methods such as 2D-DIP (frame-by-frame), 3D-DIP \cite{chen2021learning} and "Deep Video Prior" \cite{lei2020blind}, as well as implicit model based models such as SIREN \cite{sitzmann2019siren} and Gaussian activations \cite{ramasinghe2022beyond}.
We trained all video models for $5K$ iterations on all frames.
We tested them on 10 videos from the DAVIS dataset \cite{DAVIS}. Table~\ref{tab:video_denoising_image_prior} summarizes the results for all image prior based methods and Table~\ref{tab:video_denoising_implicit} summarizes the results for all implicit models.
In addition to PSNR, we also report SSIM-3D \cite{6466936} to measure temporal consistency. For visual results, please refer to the project \href{https://nimrodshabtay.github.io/PIP/}{website}.
2D-DIP and 2D-PIP have comparable results. 3D-DIP and DVP have better temporal consistency compared to 2D models but lack reconstruction quality. Our 3D-PIP outperforms all other models both in terms of temporal consistency and PSNR (+3dB for denoising, +0.5dB for SR).
3D-PIP achieves that with $\times6$ fewer parameters compared to 3D-DIP.
3D-PIP also outperforms all INR methods by at least 2-3dB in PSNR for all restoration tasks and achieves better temporal consistency.

PIP's high reconstruction performance can be attributed to the fact that it gets better temporal consistency from the FF temporal encoding in addition to the spatial encoding, which is represented in a frequency range that facilitates it to represent rapid movements as well as static backgrounds. The temporal consistency embedded in the input encoding allows PIP reconstructs much better than other image-prior methods which do not have any temporal consistency (frame-by-frame methods) or whose limited consistency stems from the 3D convolution filter size. In addition, the multi-scale architecture allows it to keep high details for each frame alongside the temporal consistency.

Moreover, when compared to other implicit models, PIP outperforms them as unlike PIP they rely on a single scale architecture, which degrades their performance. This is especially true in the case of spatial SR where all the implicit models fail since they are optimized on a single scale.

To conclude, the combination of adding FF temporal encoding and the multi-scale architecture allows PIP to output a high quality and consistent reconstruction output.

\vspace{-5pt}
\section{Conclusions}
\label{sec:conclusions}
\vspace{-5pt}
In this paper we have shed a new light on DIP, treating it as an implicit model mapping shifted noise-blocks to RGB values.
We have shown that Fourier features that are popular in implicit models can be used as a replacement for the noise inputs in DIP.
Furthermore, when using PE, the convolutional layers can be replaced with MLPs, as commonly done in implicit models.
We showed that the PE's image prior is very similar to the CNN's image prior, both in terms of quantitative performance (PSNR) and the visual appearance.
Like DIP, we showed that PIP first fits the low-frequencies, providing an explanation on how it operates.

One limitation of PIP is that it uses the same amount of frequencies for all images. An adaptation to the frequencies in a given target image may improve the reconstruction significantly in some cases \cite{hertz2021sape,lindell2021bacon}. 
In future work, we may explore selecting the input frequencies based on the input image content, e.g. based on BACON \cite{lindell2021bacon}.
The frequencies can also be spatially adapted, i.e., different sets of frequencies may be employed for different regions in the image as done in SAPE \cite{hertz2021sape}. For example, in the baby image in Fig.~\ref{fig:sr_x4_x8} it would be beneficial to use lower frequencies for the face and higher for the hat. Smartly choosing the frequencies may also help to mitigate another limitation, e.g., the need for early stopping, which is also required for DIP.

Finally, some insights from DIP and PIP might help improve implicit models, e.g., using multi-scale processing in a U-Net architecture.
The PE image prior that we observed might explain the ability of NeRF to interpolate between sparse sets of input views.
It can also explain NeRF's ability to learn from noisy inputs, e.g. as exhibited by RAW-NeRF \cite{mildenhall2022nerf}. Our results hint that this ability of NeRF is not just due to averaging samples from different views but also due to the impact of PE that generates well-behaved images.
Finally, we demonstrated that PIP can be used for temporal signals (video denoising), where DIP based methods and INR models struggles.

\section*{Acknowledgements}
This work was supported by the European research council under Grant ERC-StG 757497. This work was supported by Tel Aviv University Center for AI and Data Science (TAD).

{\small
\bibliographystyle{ieee_fullname}
\bibliography{main}
}

\clearpage

\appendix
\part*{Appendix} 

\section{Proof of Proposition~\ref{prop:PIP_FF_relation}}
\label{sec:prop1_proof}

\setcounter{proposition}{0}
\begin{proposition}
For $E$ being the $\ell_2$ loss and $f(z) = h*z$, where $h$ is a convolution kernel of size $r$, problem (2) is equivalent to an element-wise optimization with Fourier features. 
\end{proposition}

\begin{proof}
We are examining the case of 1D denoising using a single linear layer without any activation.
Input is as described in DIP, A uniform noise from $n \sim U[0, \frac{1}{10}]$. 

The following optimization is performed: 
\begin{equation}
\label{eq:conv_eq_optimization}
\| h * n - y\|_{2}^{2}, \end{equation}
where h is our spatial kernel, n is the input noise and y is the denoised image. Following \cite{Mathieu2013Fast}, we can formulate the spatial convolutions into an element-wise product in the frequency domain:
\begin{equation}
\|\mathcal{F}^{-1} \{\mathcal{F}(h) \cdot \mathcal{F}(n)\} - y\|_{2}^{2}.
\end{equation}
We will turn to focus on the left term. Denote by $F_N$ the $N\times N$ Discrete Fourier Transform (DFT) matrix:
\begin{equation}   \label{eq:FN-def}
  F_N
  = \left( \begin{array}{ccccc}
        1 &  1 & 1 & \ldots & 1 \\
        1 & \omega_N^{1} & \omega_N^{2} & \ldots & \omega_N^{N-1} \\
        1 & \omega_N^{2} & \omega_N^{4} & \ldots & \omega_N^{2(N-1)} \\
        \vdots&\vdots&\vdots& &\vdots\\
        1 & \omega_N^{N-1} & \omega_N^{2(N-1)} &\ldots&\omega_N^{(N-1)^2} 
    \end{array} \right),
\end{equation}
where, 
\[
    \omega_N^{-l} = e^{-i l 2\pi/N} 
                  = \overline{ e^{i l 2\pi/N} }
                  = \overline{\omega_N}^l.
\]
Thus, we may write the Fourier transformation in a matrix form as follows, $ \widetilde{n} = \mathcal{F}(n) = $
\begin{equation}
   \left( \begin{array}{ccccc}
        1 &  1 & 1 & \ldots & 1 \\
        1 & \omega_N^{1} & \omega_N^{2} & \ldots & \omega_N^{N-1} \\
        1 & \omega_N^{2} & \omega_N^{4} & \ldots & \omega_N^{2(N-1)} \\
        \vdots&\vdots&\vdots& &\vdots\\
        1 & \omega_N^{N-1} & \omega_N^{2(N-1)} &\ldots&\omega_N^{(N-1)^2} 
    \end{array} \right)     
    \left( \begin{array}{c}
        n_0 \\ n_1 \\ n_2 \\ \vdots \\ n_{N-1}
    \end{array} \right).
\end{equation}
and similarly, $\widetilde{h}  = \mathcal{F}(h) =$
\begin{equation}
     \left( \begin{array}{ccccc}
        1 &  1 & 1 & \ldots & 1 \\
        1 & \omega_N^{1} & \omega_N^{2} & \ldots & \omega_N^{N-1} \\
        1 & \omega_N^{2} & \omega_N^{4} & \ldots & \omega_N^{2(N-1)} \\
        \vdots&\vdots&\vdots& &\vdots\\
        1 & \omega_N^{N-1} & \omega_N^{2(N-1)} &\ldots&\omega_N^{(N-1)^2} 
    \end{array} \right)     
    \left( \begin{array}{c}
        h_0 \\ \vdots \\ h_{r-1} \\ 0 \\ \vdots \\ 0
    \end{array} \right).
\end{equation}
Notice that since $F_N$ is unitary and $n$ is initialized as i.i.d. random Gaussian with zero mean, then $\tilde{n}$ is also i.i.d. random Gaussian with zero mean. Note also that $\tilde{h}$ is basically a (random) linear combination of Fourier features. Now, note that $ \hat{y}  \triangleq \mathcal{F}^{-1}(\widetilde{n} \cdot \widetilde{h}) =$
\begin{equation} \tiny
    \frac{1}{N} \small \left( \begin{array}{ccccc}
        1 &  1 & 1 & \ldots & 1 \\
        1 & \overline{\omega_N^{1}} & \overline{\omega_N^{2}} & \ldots & \overline{\omega_N^{N-1}} \\
        1 & \overline{\omega_N^{2}} & \overline{\omega_N^{4}} & \ldots & \overline{\omega_N^{2(N-1)}} \\
        \vdots&\vdots&\vdots& &\vdots\\
        1 & \overline{\omega_N^{N-1}} & \overline{\omega_N^{2(N-1)}} &\ldots&\overline{\omega_N^{(N-1)^2}}
    \end{array} \right)     
    \left( \begin{array}{c}
        \widetilde{h_0} \cdot \widetilde{n_0} \\ \widetilde{h_1} \cdot \widetilde{n_1} \\ \widetilde{h_2} \cdot \widetilde{n_2} \\ \vdots \\ \tiny \widetilde{h_{N-1}} \cdot \widetilde{n_{N-1}} \small
    \end{array} \right)    
\end{equation}

Focusing at the $k^{th}$ row, we have that 
\begin{equation}
    \hat{y}(k) = \frac{1}{N}\sum_{l=0}^{N-1} \overline{\omega_N^{kl}} (\widetilde{h_l} \cdot \widetilde{n_l}).
\end{equation}

Since n and h are real signals, both are symmetric in the frequency domain. (namely $\widetilde{n}[k]=\widetilde{n}[N-k]$). and $\omega_N^{kl}$ is periodic N, we can write: (*up to N odd/even)
\begin{equation}
    \hat{y}(k) = \frac{2}{N}(\widetilde{h_0} \cdot \widetilde{n_0}) + \frac{1}{N}\sum_{l=1}^{N/2-1} (\overline{\omega_N^{kl}} + \overline{\omega_N^{-kl}}) (\widetilde{h_l} \cdot \widetilde{n_l})
\end{equation}
Which is equivalent to taking only the real part in the spatial domain (of a signal with a length of $\frac{N}{2}$

Using the relations between complex exponents and cosine we get the following:
$\frac{2\pi}{N} \cdot m = f_m$ we get that we can re-write the $k^{th}$ row as follows:
\begin{equation}
    \hat{y}(k) = \frac{2}{N}\sum_{l=1}^{N/2-1} cos(f_l \cdot k) \cdot (\widetilde{h_l} \cdot \widetilde{n_l})
\end{equation}

From the equation, one can see that the left term in Equation \eqref{eq:conv_eq_optimization} is equivalent to a sum of cosines that element-wise multiplied by $\hat{h_l} = \frac{1}{N} \widetilde{h_l} \cdot \widetilde{n_l}$. In DIP, $\widetilde{h_l}$ is the term that is being optimized (through the entries of $h$). In PIP, we directly optimize $\hat{h_l}$ and treats its entries independently. 

\end{proof}

\begin{figure}[]
% \vspace{-10pt}
    \centering
    \resizebox{\columnwidth}{!}{
    \includegraphics[]{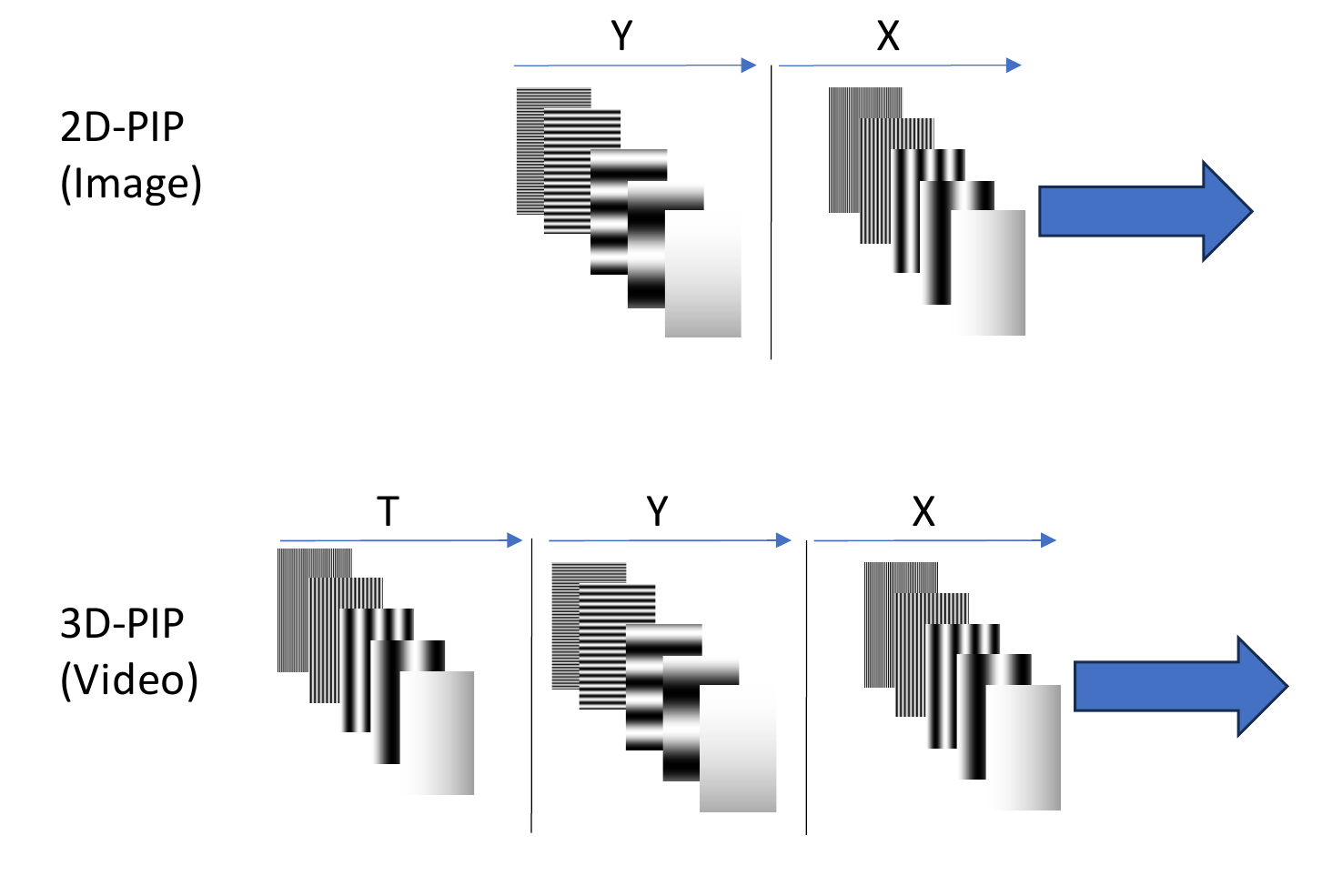}    
    }
    \caption{\textbf{PIP input encoding.} For image restoration we used a 2D encoding, concatenating vertical and horizontal FF-PE. when PIP is extended to video, we concatinate a temporal FF encoding based on all the frames (normalized from 0 (first frame) to 1 (last frame).}
    \label{fig:pip_encoding}    
\end{figure}

\begin{figure}
    \centering
    \includegraphics[width=0.6\columnwidth]{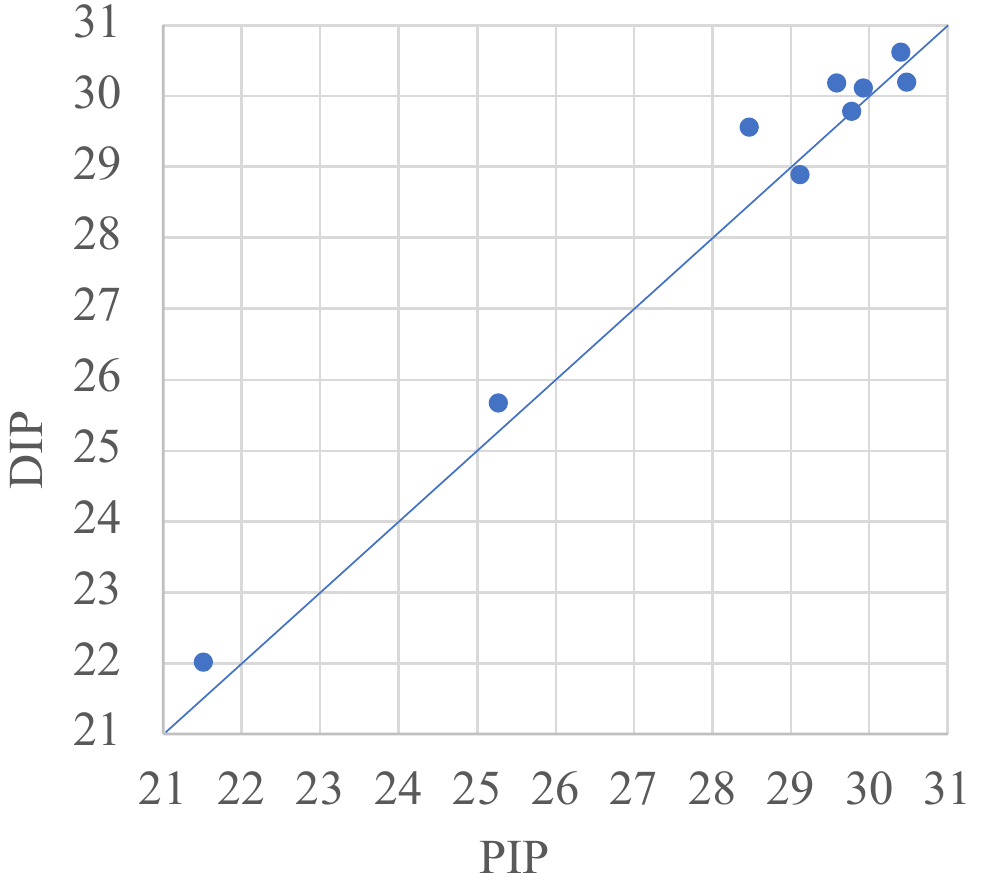}
    % \vspace{-8pt}
    \caption{DIP and PIP correlation on Image denoising ($\sigma=25$), the measured "the coefficient of determination" ($R^2$) is 0.99 suggesting PIP and DIP has similar image prior.}
    \label{fig:dip_pip_corr}
    % \vspace{-10pt}
\end{figure}

\begin{figure}[]
% \vspace{-10pt}
    \vspace{-0.3in}
    \resizebox{\columnwidth}{!}{
    \begin{tabular}{lccc}
         \multirow{5}{*}{DIP} &
         \multirow{5}{*}{\includegraphics[width=0.25\linewidth]{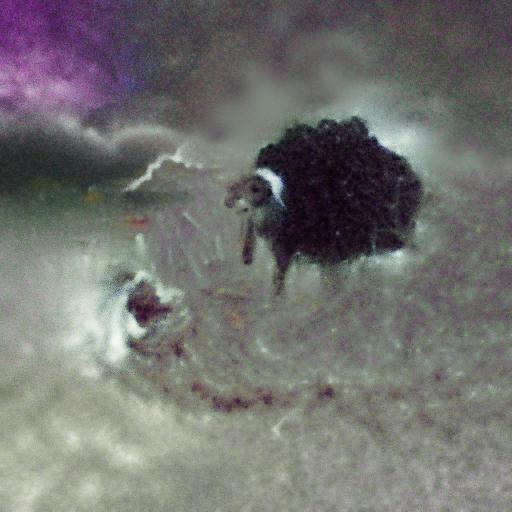}} & 
         \multirow{5}{*}{\includegraphics[width=0.25\linewidth]{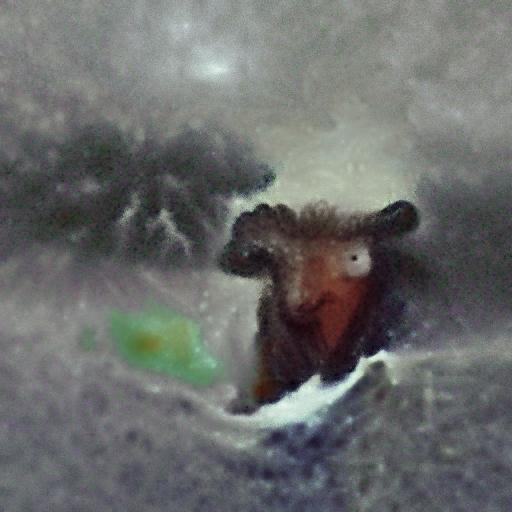}} & 
         \multirow{5}{*}{\includegraphics[width=0.25\linewidth]{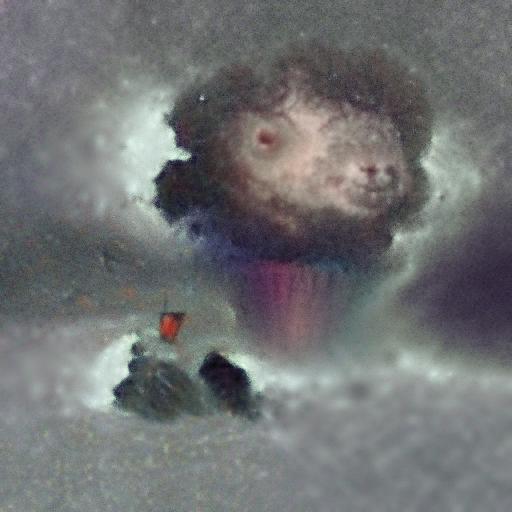}} \\[0.8in]
         \multirow{5}{*}{PIP} &
         \multirow{5}{*}{\includegraphics[width=0.25\linewidth]{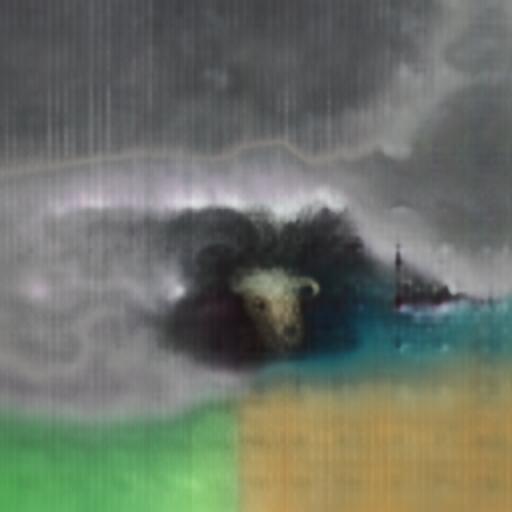}} & 
         \multirow{5}{*}{\includegraphics[width=0.25\linewidth]{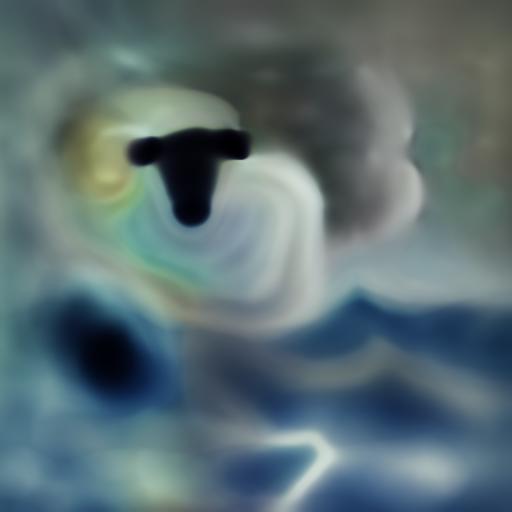}} & 
         \multirow{5}{*}{\includegraphics[width=0.25\linewidth]{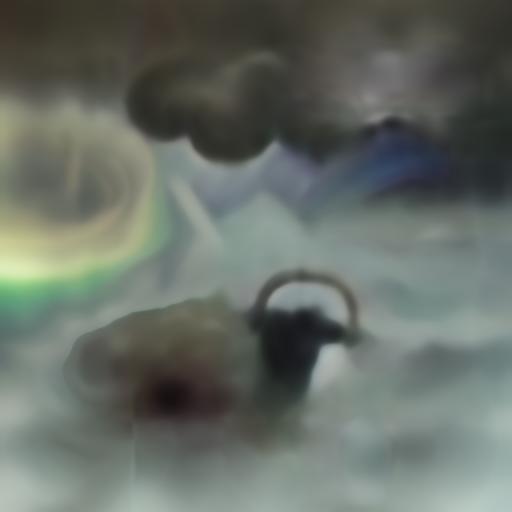}} \\[0.8in]
    \end{tabular}
    }
    \caption{\textbf{CLIP inversion.} Generated images for the description: `A dark oil painting of a sheep in the middle of the sea'.}
    \label{fig:clip_2}    
\end{figure}

\begin{figure}[]
    \resizebox{\columnwidth}{!}{
    \begin{tabular}{cccc}
    Org & noisy & DIP\cite{ulyanov2018deep} (meshgrid) & PIP\\
    \includegraphics[width=0.25\textwidth]{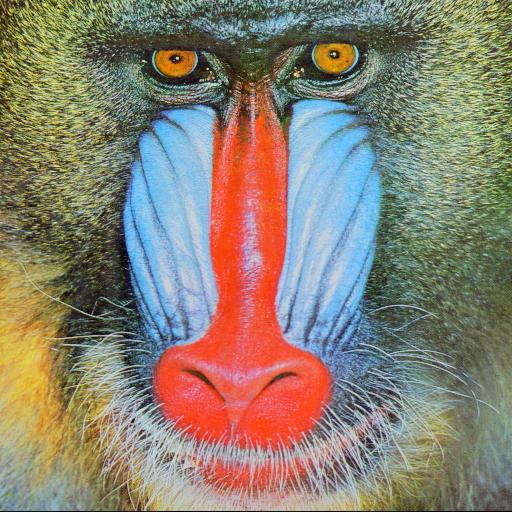} &
    \includegraphics[width=0.25\textwidth]{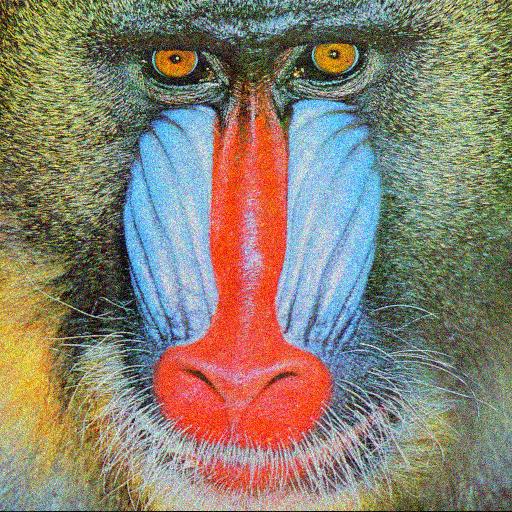} &
    \includegraphics[width=0.25\textwidth]{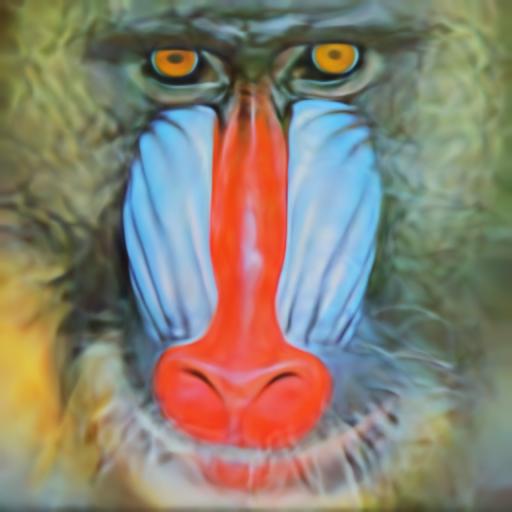} &
    \includegraphics[width=0.25\textwidth]{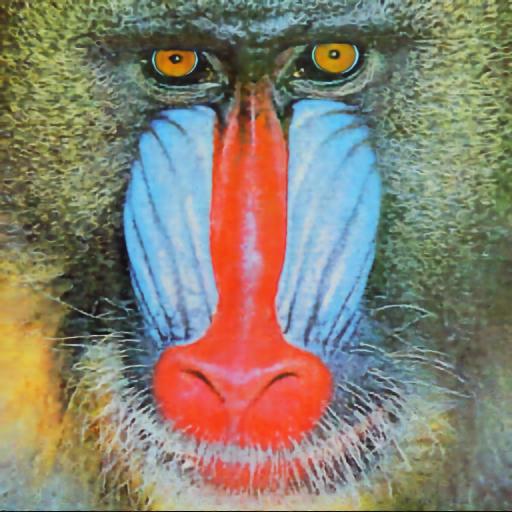}\\
    \includegraphics[width=0.25\textwidth]{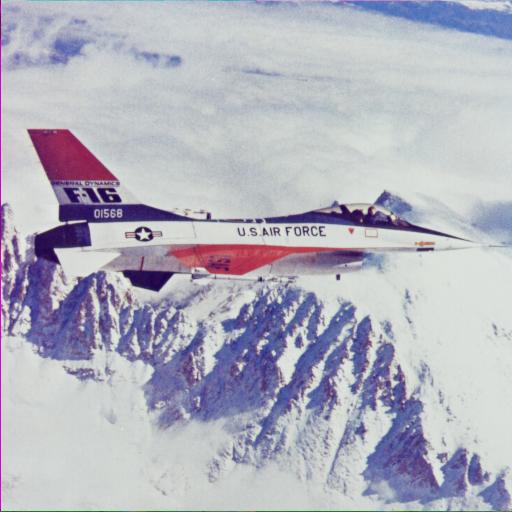} &
    \includegraphics[width=0.25\textwidth]{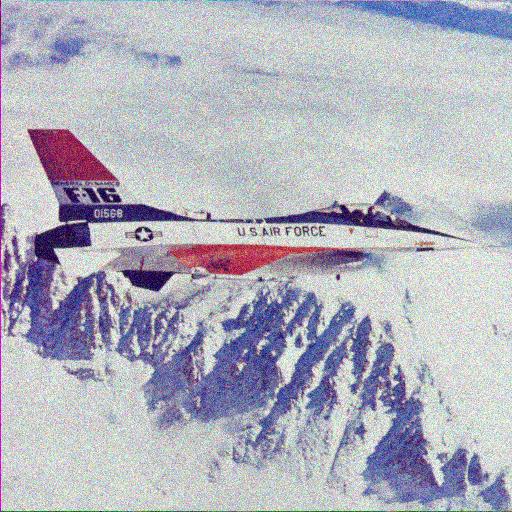} &
    \includegraphics[width=0.25\textwidth]{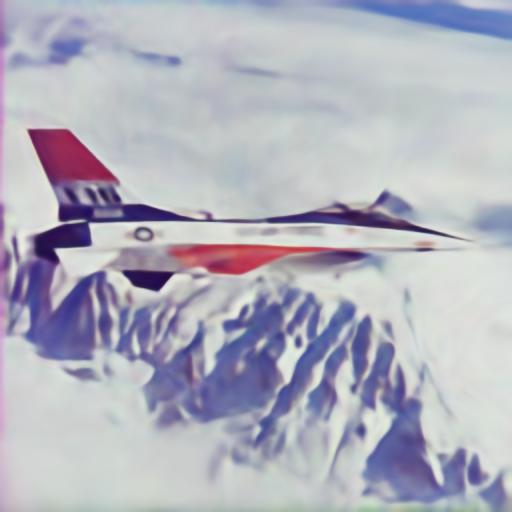} &
    \includegraphics[width=0.25\textwidth]{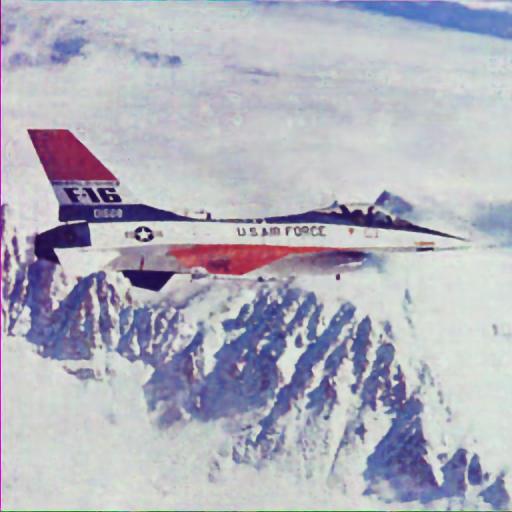}\\
    \includegraphics[width=0.25\textwidth]{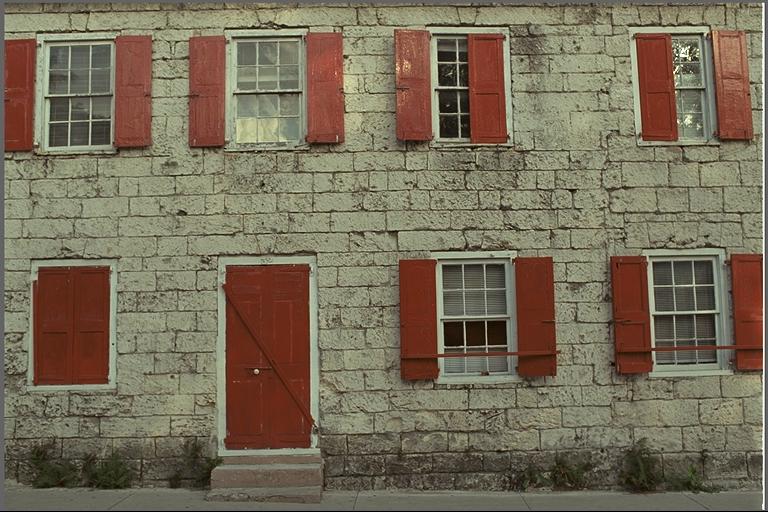} &
    \includegraphics[width=0.25\textwidth]{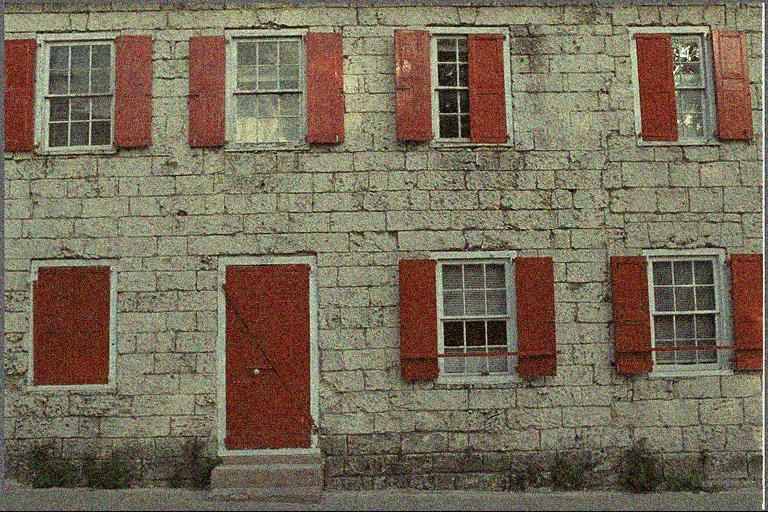} &
    \includegraphics[width=0.25\textwidth]{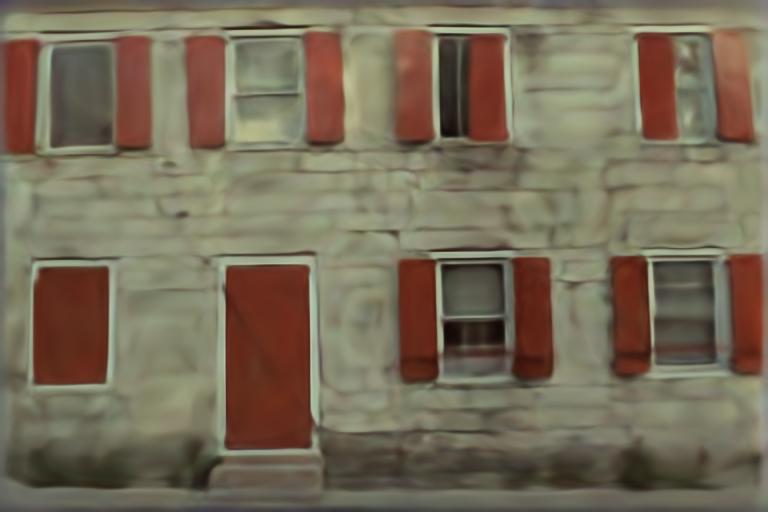} &
    \includegraphics[width=0.25\textwidth]{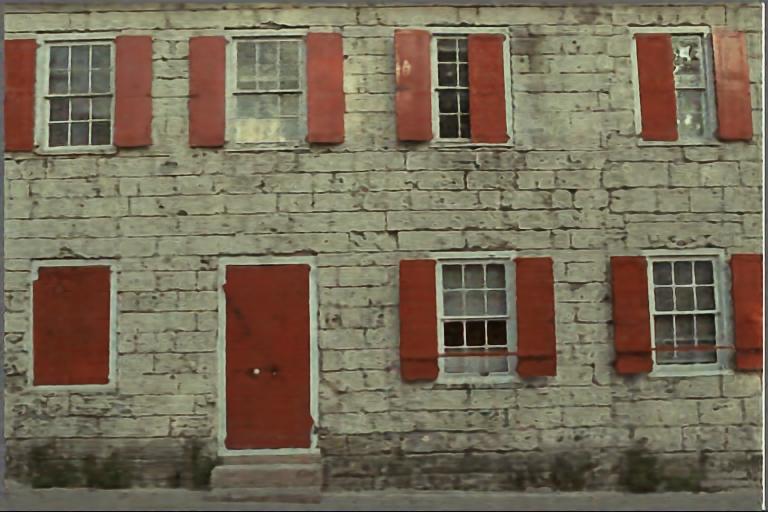}\\
    \includegraphics[width=0.25\textwidth]{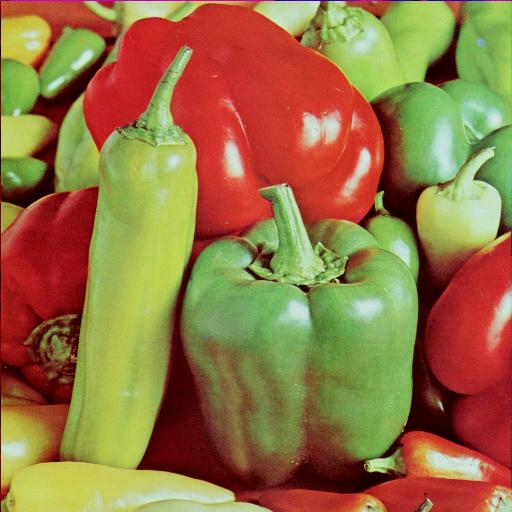} &
    \includegraphics[width=0.25\textwidth]{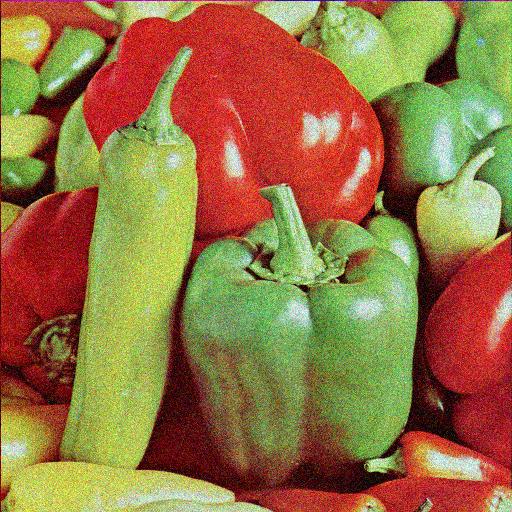} &
    \includegraphics[width=0.25\textwidth]{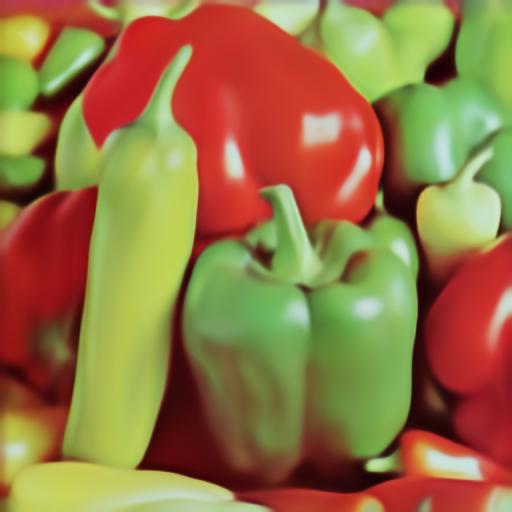} &
    \includegraphics[width=0.25\textwidth]{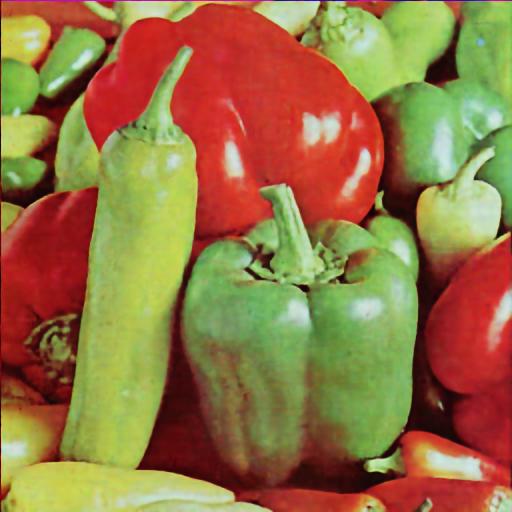}\\
    \end{tabular}
    }
    % \vspace{-8pt}
    \caption{\textbf{Image denoising ($\sigma=25$) with meshgrid.} Comparison of PIP to DIP with meshgrid encoding. Note that for DIP with meshgrid as input the results are over-smoothed and the model is unable to recover details.}
    % \label{fig:meshgrid_1}
    \label{fig:meshgrid_sup_mat}
    % \vspace{-10pt}
\end{figure}

\begin{figure}[ht]
    \resizebox{\columnwidth}{!}{
    \begin{tabular}{cccc}
    GT & Masked image & DIP & PIP (ours)\\
    \includegraphics[width=0.25\textwidth]{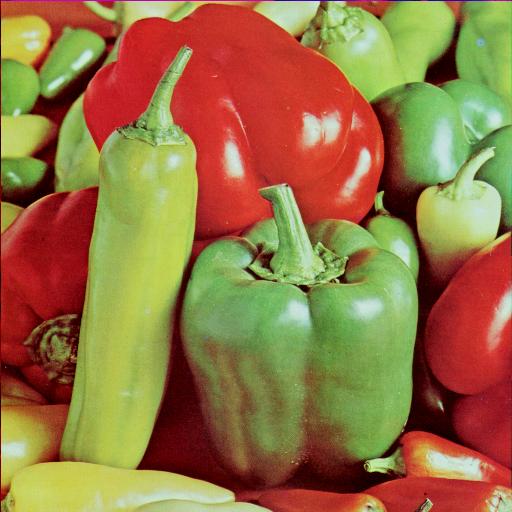} &
    \includegraphics[width=0.25\textwidth]{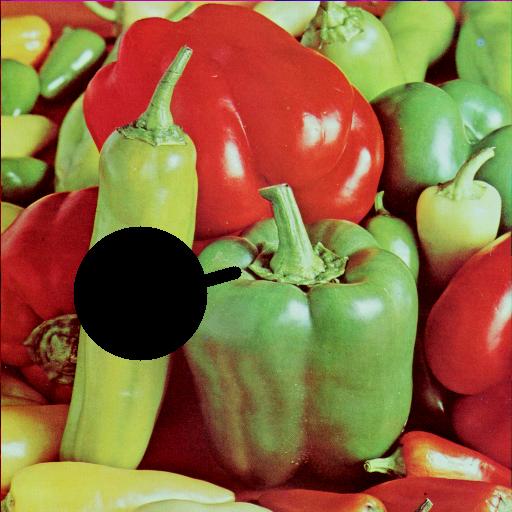} &
    \includegraphics[width=0.25\textwidth]{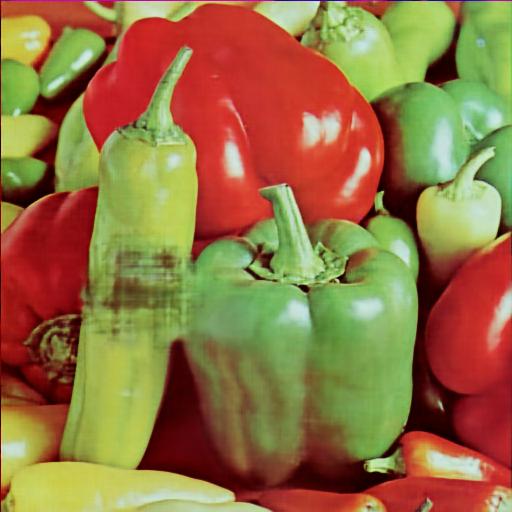} &
    \includegraphics[width=0.25\textwidth]{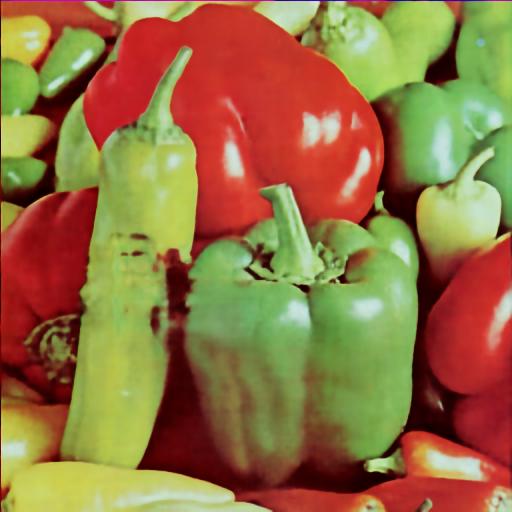} \\
    \includegraphics[width=0.25\textwidth]{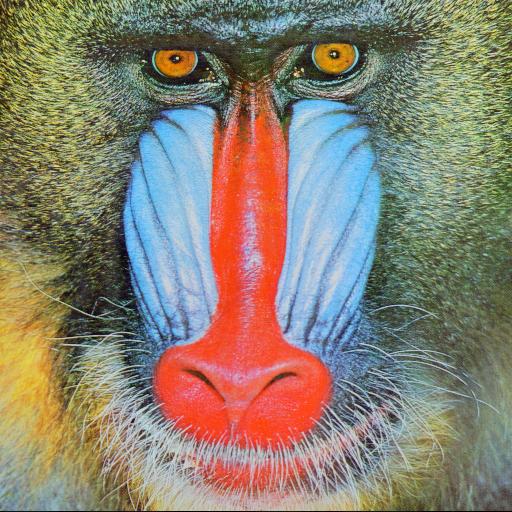}&
    \includegraphics[width=0.25\textwidth]{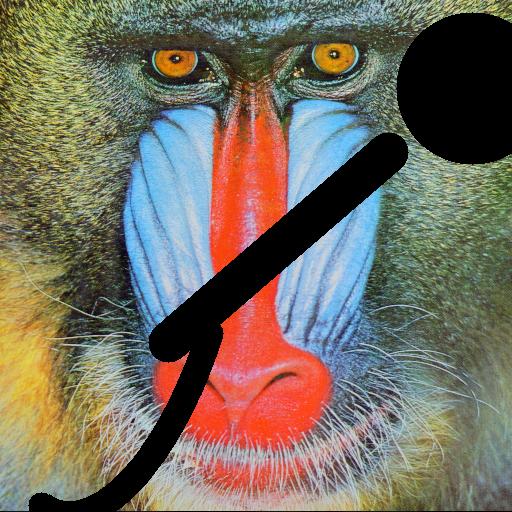} &
    \includegraphics[width=0.25\textwidth]{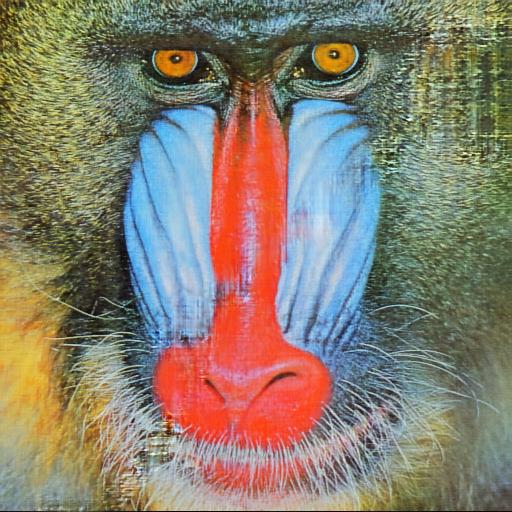} &
    \includegraphics[width=0.25\textwidth]{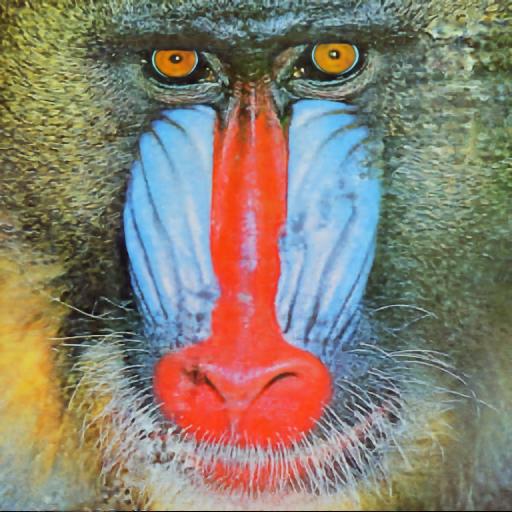} \\
    \includegraphics[width=0.25\textwidth]{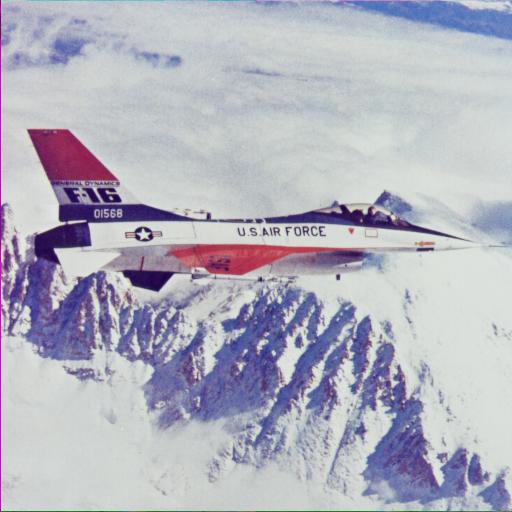} &
    \includegraphics[width=0.25\textwidth]{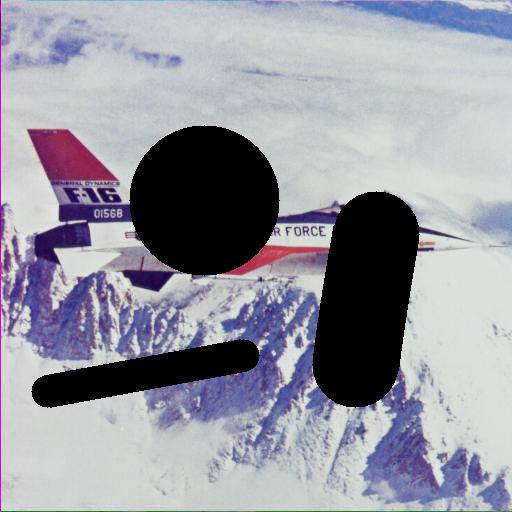} &
    \includegraphics[width=0.25\textwidth]{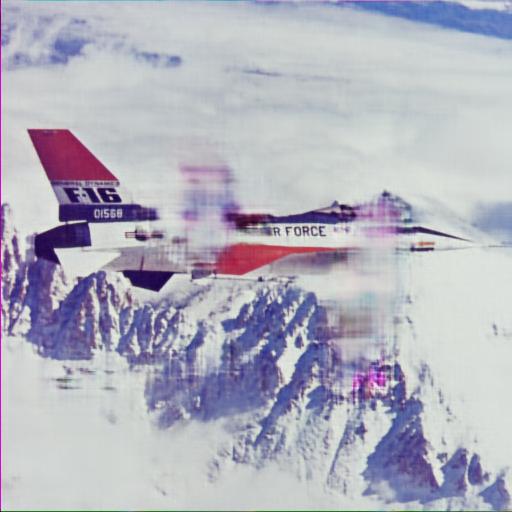} &
    \includegraphics[width=0.25\textwidth]{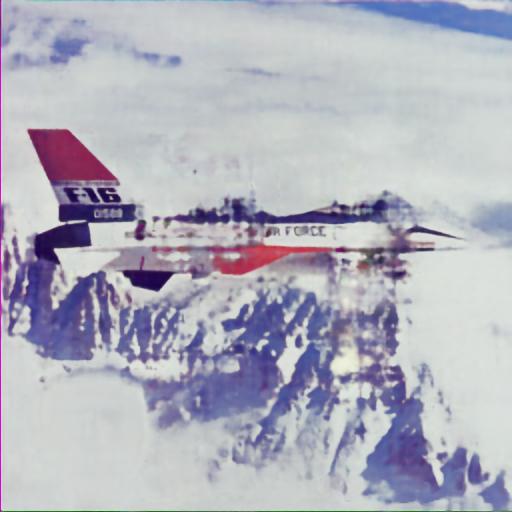} \\
    
    \end{tabular}
    }
    % \vspace{-8pt}
    \caption{\textbf{Image inpainting.} DIP vs. PIP comparison }
    \label{fig:inpainting_sup}
    \vspace{-10pt}
\end{figure}

\begin{figure}[]
    \resizebox{\columnwidth}{!}{
    \begin{tabular}{ccc}
    Hazy & Double-DIP & Double-PIP (ours)\\
    \includegraphics[width=0.25\textwidth]{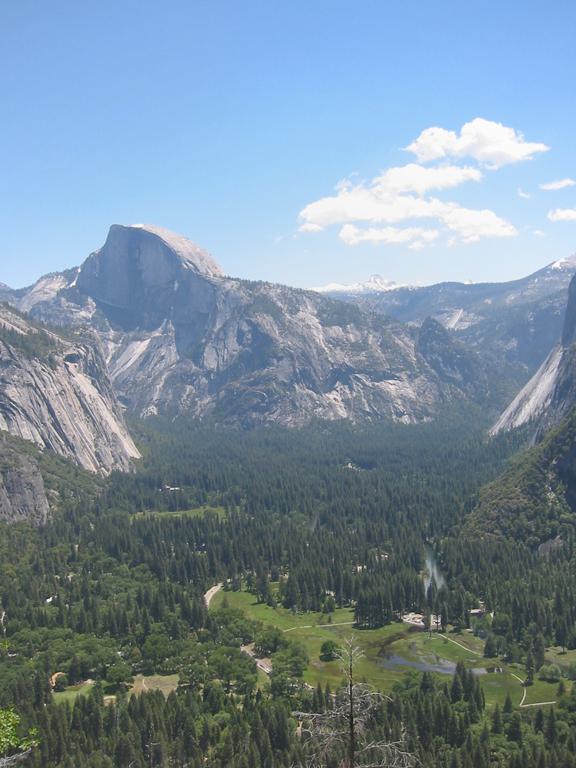} &
    \includegraphics[width=0.25\textwidth]{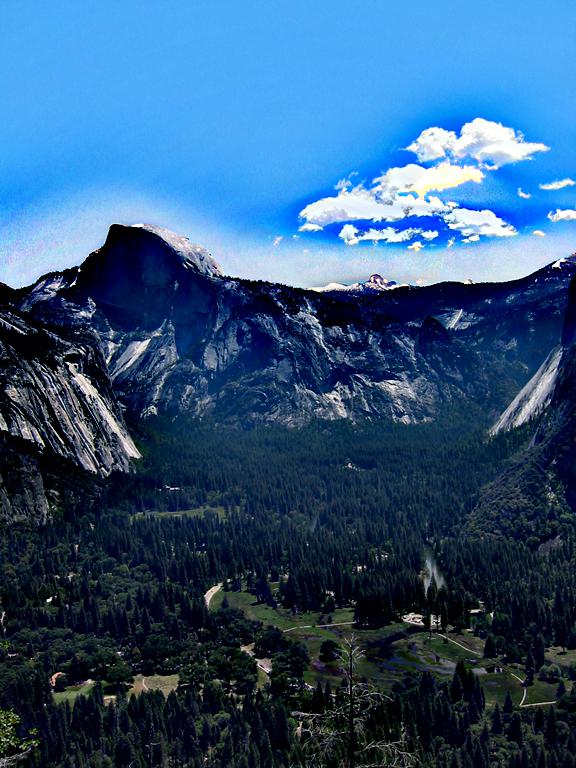} &
    \includegraphics[width=0.25\textwidth]{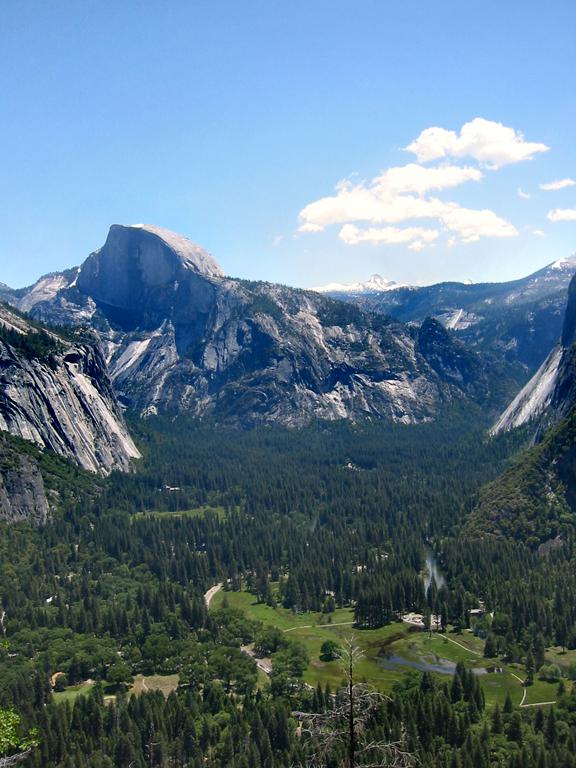} \\
    \includegraphics[width=0.25\textwidth]{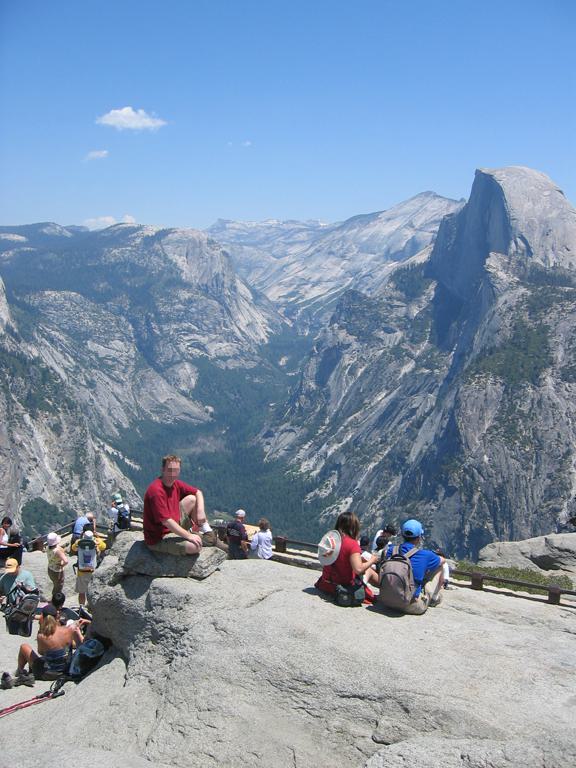} &
    \includegraphics[width=0.25\textwidth]{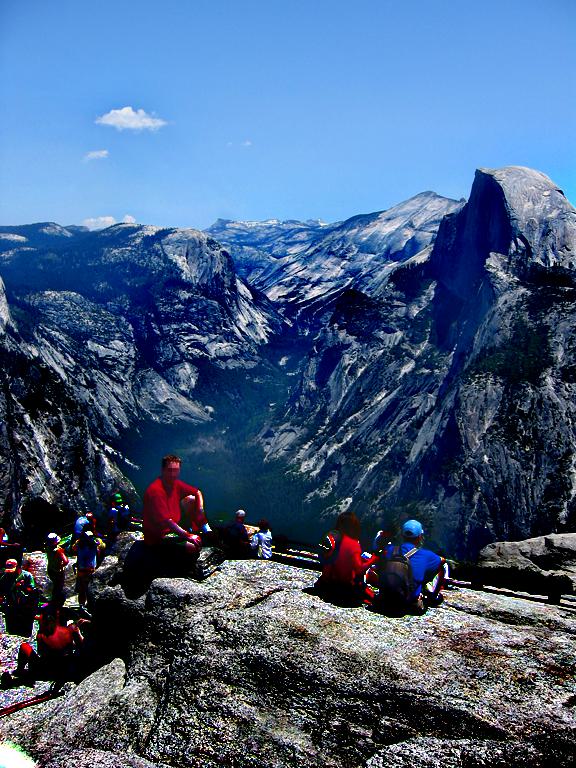} &
    \includegraphics[width=0.25\textwidth]{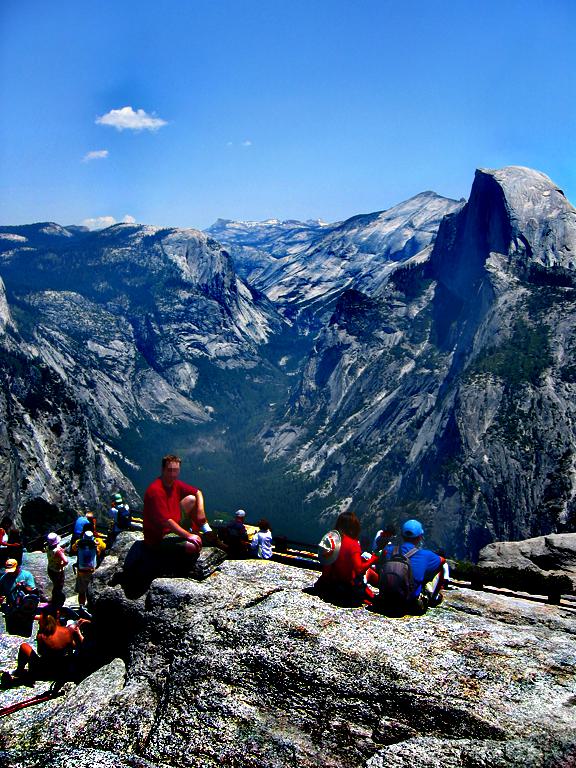} \\
    \includegraphics[width=0.25\textwidth]{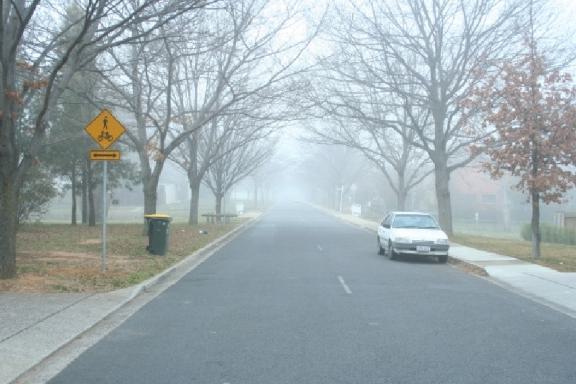} &
    \includegraphics[width=0.25\textwidth]{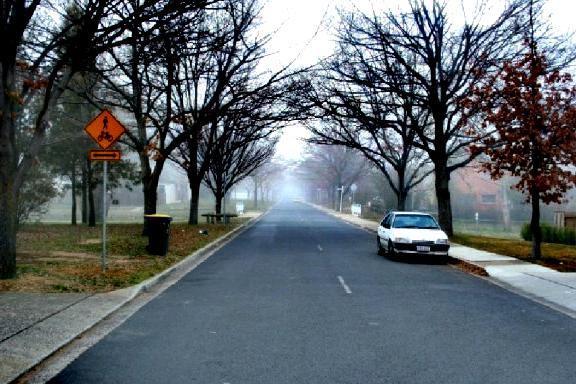} &
    \includegraphics[width=0.25\textwidth]{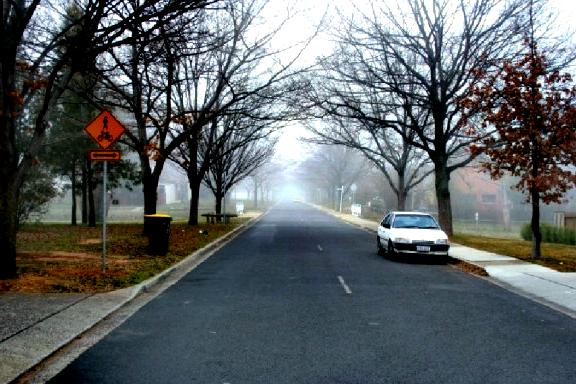} \\
    \includegraphics[width=0.25\textwidth]{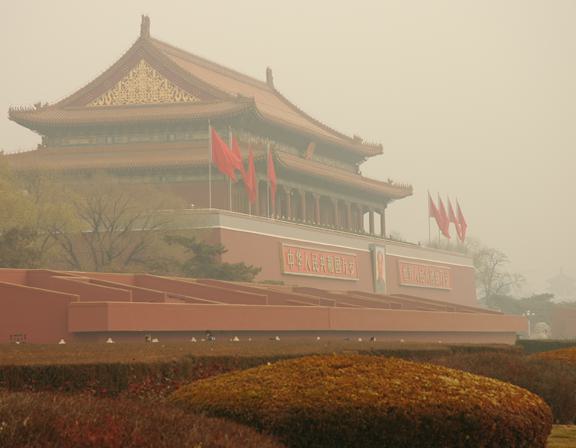} &
    \includegraphics[width=0.25\textwidth]{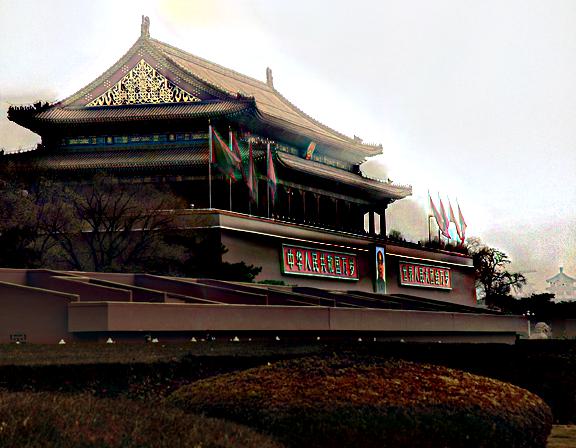} &
    \includegraphics[width=0.25\textwidth]{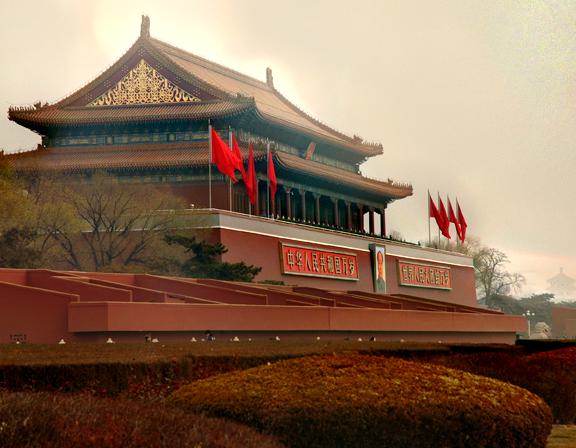} \\
     \end{tabular}
    }
        % \vspace{-8pt}
    \caption{\textbf{Blind image dehazing.} Double-DIP and Double-PIP qualitative comparison}
    \label{fig:dehazing2}
    % \vspace{-10pt}
\end{figure}

\end{document}